\documentclass{article}

\PassOptionsToPackage{numbers, compress}{natbib}

\usepackage[final]{neurips_2024}

\usepackage[utf8]{inputenc} 
\usepackage[T1]{fontenc}    
\usepackage{hyperref}       
\usepackage{url}            
\usepackage{booktabs}       
\usepackage{amsfonts}       
\usepackage{nicefrac}       
\usepackage{microtype}      
\usepackage{xcolor}         
\usepackage{arydshln}
\usepackage{amsmath}
\usepackage{amssymb}
\usepackage{mathtools}
\usepackage{amsthm}
\usepackage{algorithm}
\usepackage{algorithmic}
\usepackage{bbm}
\usepackage{pifont}
\usepackage{multirow}
\usepackage{enumitem}
\usepackage{MnSymbol}

\definecolor{orange}{HTML}{E38256}
\definecolor{blue}{HTML}{4575AD}
\definecolor{green}{HTML}{A4D666}
\definecolor{pink}{HTML}{EC8EC1}

\def\ie{{\em i.e.},\ }
\def\eg{{\em e.g.},\ }

\def\Surv#1#2{S(#1 \mid #2)}
\def\SurvPred#1#2{\hat{S}(#1 \mid #2)}
\def\Bfx#1{\boldsymbol{x}_{#1}}
\def\BfX{\mathbf{X}}
\def\E{\mathbb{E}}
\def\prob{\mathbb{P}}

\def\Data{\mathcal{D}}
\def\iset{\mathcal{I}}
\def\Model{\mathcal{M}}

\def\Pct{\text{Percentile}}
\def\xmark{\text{\sffamily X}}%
\DeclareMathOperator*{\argmax}{arg\,max}

\newcommand*{\Scale}[2][4]{\scalebox{#1}{$#2$}}%

\theoremstyle{plain}
\newtheorem{theorem}{Theorem}[section]

\newtheorem{lemma}[theorem]{Lemma}

\theoremstyle{definition}
\newtheorem{definition}[theorem]{Definition}

\theoremstyle{remark}

\def\CSDiPOT{\texttt{CiPOT} }
\usepackage[finalnew]{trackchanges}
\addeditor{RG} 
\addeditor{SQ}
\addeditor{YY} 

\title{
Toward Conditional Distribution Calibration in Survival Prediction}

\author{%
  Shi-ang Qi $^1$, Yakun Yu $^2$, Russell Greiner $^{1 \ 3}$\\
  $^1$Computing Science, University of Alberta, Edmonton, Canada \\
  $^2$Electrical Computer Engineering, University of Alberta, Edmonton, Canada \\
  $^3$Alberta Machine Intelligence Institute, Edmonton, Canada \\
  \texttt{\{shiang, yakun2, rgreiner\}@ualberta.ca} \\
}

\begin{document}

\maketitle

\begin{abstract}
Survival prediction often involves estimating the time-to-event distribution from censored datasets. 
Previous approaches have focused on enhancing discrimination and marginal calibration. 
In this paper, we highlight the significance of \emph{conditional calibration} for real-world applications
-- especially 
its role in individual decision-making. 
We propose a method based on conformal prediction that uses the model's predicted individual survival probability at {that instance's} observed time. 
This method effectively improves the model's marginal and conditional calibration, without compromising discrimination.
We provide asymptotic theoretical guarantees for both marginal and conditional calibration and 
test it extensively
across 15 diverse real-world datasets,
demonstrating the method's practical effectiveness and versatility in various settings.
\end{abstract}

\section{Introduction}
\label{sec:intro}

Individual survival distribution (ISD), or time-to-event distribution, is a probability distribution that describes the times until the occurrence of a specific event of interest for an instance, based on information about that individual.
Accurately estimating ISD is essential for effective decision-making and clinical resource allocation. 
However, a challenge in learning such survival prediction models is training 
{on datasets that include \emph{censored}\ instances}, where we only know a lower bound of their time-to-event.

Survival models typically focus on two important
but 
distinct properties during optimization and evaluation:
(i)~\emph{discrimination} measures how well a model's relative predictions between individuals align with the observed order~\cite{harrell1984regression, harrell1996multivariable}, which is useful for pairwise decisions such as prioritizing treatments;
(ii)~\emph{calibration} assesses how well the predicted survival probabilities match the actual distribution of observations~\cite{haider2020effective, chapfuwa2020calibration}, supporting both individual-level (\eg determining high-risk treatments based on the probability) and group-level (\eg allocating clinical resources) decisions.
{Some} prior research has sought to improve calibration by integrating a calibration-specific loss during optimization~\cite{avati2020countdown, goldstein2020x, chapfuwa2020calibration}. However, these often produce models with poor discrimination~\cite{kamran2021estimating,qi2024conformalized}, limiting their utility in scenarios where precise pairwise decisions are critical.

Furthermore, previous studies have typically addressed calibration in a marginal sense
-- \ie assessing whether probabilities align with the actual distribution \emph{across the entire population}.
However, for many applications, marginal calibration may be inadequate
-- we often require that predictions are correctly calibrated, 
\emph{conditional on any 
{combination of} features}.
This can be helpful for making more precise clinical decisions for individuals and groups.
For example, when treating an overweight male, a doctor might decide on cardiovascular surgery using a model calibrated for both overweight and male.
{Note this might lead to a different decision that one based on a model that was calibrated for all patients.}
Similarly, a hospice institution may want to allocate nursing care based on a model that generates calibrated predictions \emph{for elderly individuals}. 
This also aligns with the fairness perspective~\cite{verma2018fairness}, where clinical decision systems should guarantee equalized calibration performance across any protected groups.

\paragraph{Contributions} To overcome these challenges, we introduce the \CSDiPOT framework, a post-processing approach built upon conformal prediction~\cite{vovk2005algorithmic, romano2019conformalized, candes2023conformalized, qi2024conformalized} that uses the Individual survival Probability at Observed Time (iPOT) as conformity scores and generates conformalized survival distributions. 
The method has 3 important properties: 
(i)~this conformity score naturally conforms to the definition in distribution calibration in survival analysis~\cite{haider2020effective}; 
(ii)~it also captures the distribution variance of the ISD, therefore is adaptive to the features; 
and (iii)~the method is computationally friendly for survival analysis models. 
Our key contributions are:
\begin{itemize}
    \item Motivating the use of conditional distribution calibration in survival analysis, and proposing a metric ($\text{Cal}_{\text{ws}}$, defined in Section~\ref{sec:metrics}) to evaluate this property.
    \item Developing the \CSDiPOT framework, to accommodate censorship. The method effectively solves some issues of previous conformal methods wrt inaccurate Kaplan-Meier estimation;
    \item Theoretically proving that \CSDiPOT asymptotically guarantees marginal and conditional distribution calibration under some specified  
    assumptions;
    \item Conducting extensive experiments across 15 datasets, showing that \CSDiPOT improves both marginal and conditional distribution calibration without sacrificing discriminative ability;
    \item Demonstrating that \CSDiPOT is computationally more efficient than prior conformal method on survival analysis.
\end{itemize}

\section{Problem statement and Related Work} 
\label{sec:problems}

\subsection{Notation}
\label{sec:notation}
A survival dataset $\mathcal{D}=\{(\Bfx{i}, t_i, \delta_i)\}_{i=1}^n$ contains $n$ tuples, 
each containing covariates $\Bfx{i} \in \mathbb{R}^d$, an observed time $t_i \in \mathbb{R}_+$, 
and an event indicator $\delta_i \in \{0, 1\}$. 
For each subject, there are two potential times of interest: the event time $e_i$ and the censoring time $c_i$. 
However, only the earlier of the two is observable.
We assign $t_i \triangleq \min\{e_i, c_i\}$ and $\delta_{i} \triangleq \mathbbm{1}[e_i \leq c_i]$, so
$\delta_i = 0$ means the event has not happened by $t_i$ (right-censored) and $\delta_i = 1$ indicates the event occurred at $t_i$ (uncensored). 
Let $\iset$ denote the set of indices in dataset $\Data$, then we can use $i\in \iset$ to represent $(\Bfx{i}, t_i, \delta_i) \in \Data$.

Our objective is to estimate the Individualized Survival Distribution (ISD), $S (t \mid \Bfx{i}) = \prob(e_i > t \mid \BfX=\Bfx{i})$, which represents the survival probabilities of the $i$-th subject for any time
{$t\geq 0$}.

\subsection{Notions of calibration in survival analysis} 
\label{sec:calibration}


Calibration measures the alignment between the predictions against observations.
Consider distribution calibration at the individual level:
if an oracle knows the true ISD $S (t \mid \Bfx{i})$, 
and draws realizations of $e_i \mid \Bfx{i}$ (call them $e_i^{(1)}, e_i^{(2)}, \ldots$),
then the survival probability at observed time $\{S (e_i^{(m)} \mid \Bfx{i})\}_{m}$ 
should be distributed across a standard uniform distribution $\mathcal{U}_{[0,1]}$ (probability integral theorem~\citep{angus1994probability}). 
However, in practice, for each unique $\Bfx{i}$, there is only one realization of $e_i \mid \Bfx{i}$, meaning we cannot check the calibration in this individual manner.

To solve this, \citet{haider2020effective} proposed \emph{marginal calibration}, which holds if the predicted survival probabilities at event times $e_i$ over the $\Bfx{i}$ in the dataset, $\{\hat{S} (e_i \mid \Bfx{i})\}_{i \in \iset}$, matches $\mathcal{U}_{[0, 1]}$. 
\begin{definition}
\label{def:margin_cal_uncensored}
For uncensored dataset, a model has perfect marginal calibration iff $\forall \ [\rho_1, \rho_2] \subset [0, 1]$,
\begin{equation}
\label{eq:cal_uncensored}
\Scale[0.93]{
    \prob \left( \, \hat{S} (e_i \mid \Bfx{i} ) \in [\rho_1, \rho_2]\, , \, i \in \iset \,\middle\vert\,  \delta_i  = 1  \, \right)  \ = \ \E_{i \in \iset} \ \mathbbm{1} \left[ \, \hat{S} (e_i \mid \Bfx{i}) \in [\rho_1, \rho_2]  \,\middle\vert\, \delta_i=1 \, \right] \ = \ \rho_2 - \rho_1.
    }
\end{equation}
\end{definition}
We can ``blur'' each censored subject uniformly over the probability intervals after the survival probability at censored time $\hat{S} (c_i \mid \Bfx{i})$~\cite{haider2020effective} (see the derivation in Appendix~\ref{appendix:calibration_intro}):
\begin{equation}
\label{eq:cal_censored}
\Scale[1.05]{
    \prob \left( \hat{S} (e_i \mid \Bfx{i}) \in [\rho_1, \rho_2] \,\middle\vert\, \delta_i = 0 \right) = \frac{\left(\SurvPred{t_i}{\Bfx{i}} - \rho_1\right)\mathbbm{1}\left[\SurvPred{t_i}{\Bfx{i}} \in [\rho_1, \rho_2]\right] + (\rho_2 - \rho_1)\mathbbm{1}\left[\SurvPred{t_i}{\Bfx{i}} \geq \rho_2\right]}{\SurvPred{t_i}{\Bfx{i}}}.
    }
\end{equation}

\begin{figure}
    \centering
\includegraphics[width=0.85\textwidth]{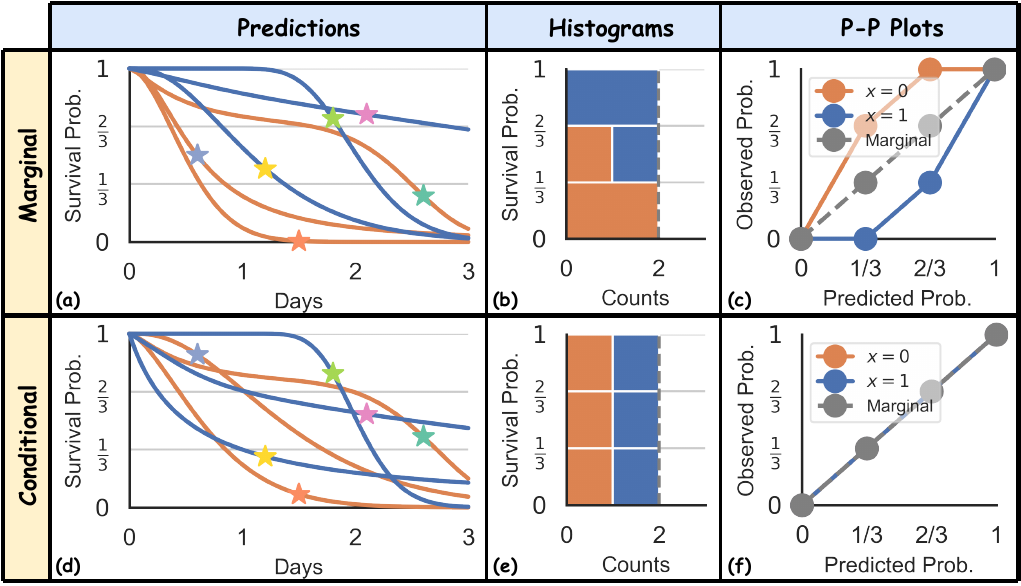}
    \caption{Two notions of distribution calibration: marginal and conditional, illustrated using 3 bins separated at $\frac{1}{3}$ and $\frac{2}{3}$. 
    The curves in~(a, d) represent the predicted ISDs.
    The colors of the stars distinguish the six subjects, with horizontal coordinates indicating the true event time (consistent across all panels) and vertical coordinates representing predicted survival probability at event time.
    Note the two groups (\textcolor{orange}{\textbf{orange}} for $x=0$ and \textcolor{blue}{\textbf{blue}} for $x=1$) correspond to the colors of the curves and histograms in (a, b, d, e). Note that all three P-P lines in the conditional case~(f) coincide.
    }
    \label{fig:calibration_example}
\end{figure}

Figure~\ref{fig:calibration_example}(a) illustrates how marginal calibration is assessed using 6 uncensored subjects. 
Figure~\ref{fig:calibration_example}(b, c) show the histograms and P-P plots, showing the predictions are marginally calibrated, over the predefined 3 bins%
, as we can see there are 6/3 = 2 instances in each of the 3 bins. 
However, if we divide the datasets into two groups (orange vs. blue -- think men vs. women), 
we can see that 
this is not the case, as there is no orange instance in the $[\frac{2}{3}, 1]$ bin, and 2 orange instances in the $[\frac{1}{3}, 1]$ bin.

In summary, individual calibration is ideal but impractical. Conversely, marginal calibration is more feasible but fails to assess calibration relative to certain subsets of the population by features. 
This discrepancy motivates us to explore a middle ground -- \emph{conditional calibration}.
A conditionally calibrated prediction, which ensures that the predicted survival probabilities are uniformly distributed in each of these groups, 
as shown in Figure~\ref{fig:calibration_example}(d, e, f), is more effective in real-world scenarios.
Consider predicting employee attrition within a company: 
while a marginal calibration using a Kaplan-Meier (KM)~\cite{kaplan1958nonparametric} curve might reflect overall population trends, it fails to account for variations such as the tendency of lower-salaried employees to leave earlier. 
A model that is calibrated for both high and low salary levels would be more helpful for 
predicting the actual quitting times and 
{facilitate planning}. 
Similarly, when predicting the timing of death from cardiovascular diseases, 
models calibrated for older populations, who exhibit more predictable and less varied outcomes~\cite{visseren20212021}, may not apply to younger individuals with higher outcome variability. Using age-inappropriate models could lead to inaccurate predictions, adversely affecting treatment plans.




\subsection{Maintaining discriminative performance while ensuring good calibration}
\label{sec:disc_cali_trade_off}

Methods based on the objective function~\cite{avati2020countdown, goldstein2020x, chapfuwa2020calibration} have been developed to enhance the marginal calibration of ISDs, involving the addition of a calibration loss to the model’s original objective function (\eg likelihood loss).  
However, while those methods are effective in improving the marginal calibration performance of the model, their model often significantly harms the discrimination performance~\cite{goldstein2020x, chapfuwa2020calibration, kamran2021estimating},
a phenomenon known as the \emph{discrimination-calibration trade-off}~\cite{kamran2021estimating}.

Post-processed methods~\cite{candes2023conformalized, qi2024conformalized} have been proposed to solve this trade-off by disentangling calibration from discrimination in the optimization process. 
\citet{candes2023conformalized} uses the individual censoring probability as the weighting in addition to the regular Conformalized Quantile Regression (CQR)~\cite{romano2019conformalized} method. 
However, their weighting method is only applicable to Type-I censoring settings where each subject must have a known censored time \cite{klein2006survival} -- which is not applicable to most of the right-censoring datasets.

\citet{qi2024conformalized} developed Conformalized Survival Distribution (\texttt{CSD}) by first discretizing the ISD curves into percentile times (via predefined percentile levels), and then applying CQR~\cite{romano2019conformalized} for each percentile level (see the visual illustration of \texttt{CSD} in Figure~\ref{fig:csd_compare} in Appendix~\ref{appendix:alg}).
Their method handles right-censoring using KM-sampling, which employs a conditional KM curve to simulate multiple event times for a censored subject, 
offering a calibrated approximation for the ISD based on the subject's censored time.
However, their method struggles with some inherent problems of KM~\cite{kaplan1958nonparametric} --
\eg KM can be inaccurate when the dataset contains a high proportion of censoring~\cite{liu2021ipdfromkm}. 
Furthermore, we also observed that the KM estimation often concludes at high probabilities (as seen in datasets like \texttt{HFCR}, \texttt{FLCHAIN}, and \texttt{Employee} in Figure~\ref{fig:datasets}).
This poses a challenge in extrapolating beyond the last KM time point, which hinders the accuracy of KM-sampling, thereby constraining the efficacy of \texttt{CSD} (see our results in Figure~\ref{fig:cindex_dcal_main}).

Our work is inspired by \texttt{CSD}~\citep{qi2024conformalized}, and can be seen as a percentile-based refinement of their regression-based approach.
Specifically, our \CSDiPOT effectively addresses and resolves issues prevalent in the KM-sampling, significantly outperforming existing methods in terms of improving the marginal distribution calibration performance. 
Furthermore, to our best knowledge, this is the first approach that optimizes conditional calibration within the survival analysis that can deal with censorship.

However, achieving conditional calibration (also known as conditional coverage in some literature~\cite{vovk2005algorithmic}) is challenging because it cannot be attained in a distribution-free manner for non-trivial predictions. In fact, guarantees of finite sample for conditional calibration are impossible to achieve even for standard regression datasets without censorship~\cite{vovk2005algorithmic, lei2014distribution, foygel2021limits}. This limitation is an important topic in the statistical learning and conformal prediction literature~\cite{vovk2005algorithmic, lei2014distribution, foygel2021limits}. Therefore, our paper does not attempt to provide finite sample guarantees. Instead, following the approach of many other researchers~\cite{romano2019conformalized, lei2018distribution, sesia2020comparison, izbicki2020flexible, chernozhukov2021distributional, izbicki2022cd}, we provide only asymptotic guarantees as the sample size approaches infinity. The key idea behind this asymptotic conditional guarantee is that the construction of post-processing predictions relies on the quality of the original predictions. Thus, we aim for conditional calibration only within the class of predictions that can be learned well -- that is, consistent estimators.

We acknowledge that this assumption may not hold in practice; 
however, (i)~reliance on consistent estimators is a standard (albeit strong) assumption in the field of conformal prediction~\cite{sesia2020comparison, izbicki2020flexible, chernozhukov2021distributional}, (ii)~to the best of our knowledge, no previous results have proven conditional calibration under more relaxed conditions, and (iii)~we provide empirical evidence of conditional calibration using extensive experiments (see Section~\ref{sec:exp}) .


\section{Methods}
\label{sec:methods}

This section describes our proposed method: Conformalized survival distribution using Individual survival Probability at Observed Time (\texttt{CiPOT}), which is motivated by the definition of distribution calibration~\cite{haider2020effective} and consists of three components: the estimation of continuous ISD prediction, the computation of suitable conformity scores (especially for censored subjects), and their conformal calibration.

\subsection{Estimating survival distributions}
\label{sec:method_step1}

For simplicity, our method is motivated by the split conformal prediction~\cite{papadopoulos2002inductive, romano2019conformalized}. 
We start the process by splitting the instances of the training data into a proper training set $\Data^{\text{train}}$ and a conformal set $\Data^{\text{con}}$. 
Then, we can use any survival algorithm or quantile regression algorithm (with the capability of handling censorship) to train a model $\Model$ using $\Data^{\text{train}}$ that can make ISD predictions for $\Data^{\text{con}}$ -- see Figure~\ref{fig:iPOT_illust}(a).

With little loss of generality, we assume that the ISD predicted by the model, $\hat{S}_{\Model}(t\mid \Bfx{i})$, are right-continuous and have unbounded range, \ie $\hat{S}_\Model (t\mid \Bfx{i}) > 0$ for all $t \geq 0$. 
For survival algorithms that can only generate piecewise constant survival probabilities (\eg Cox-based methods~\cite{cox1972regression, kvamme2019time}, discrete-time methods~\cite{yu2011learning, lee2018deephit}, etc.), the continuous issue can be fixed by applying some interpolation algorithms (\eg linear or spline). 

\begin{figure}
    \centering
    \includegraphics[width=\textwidth]{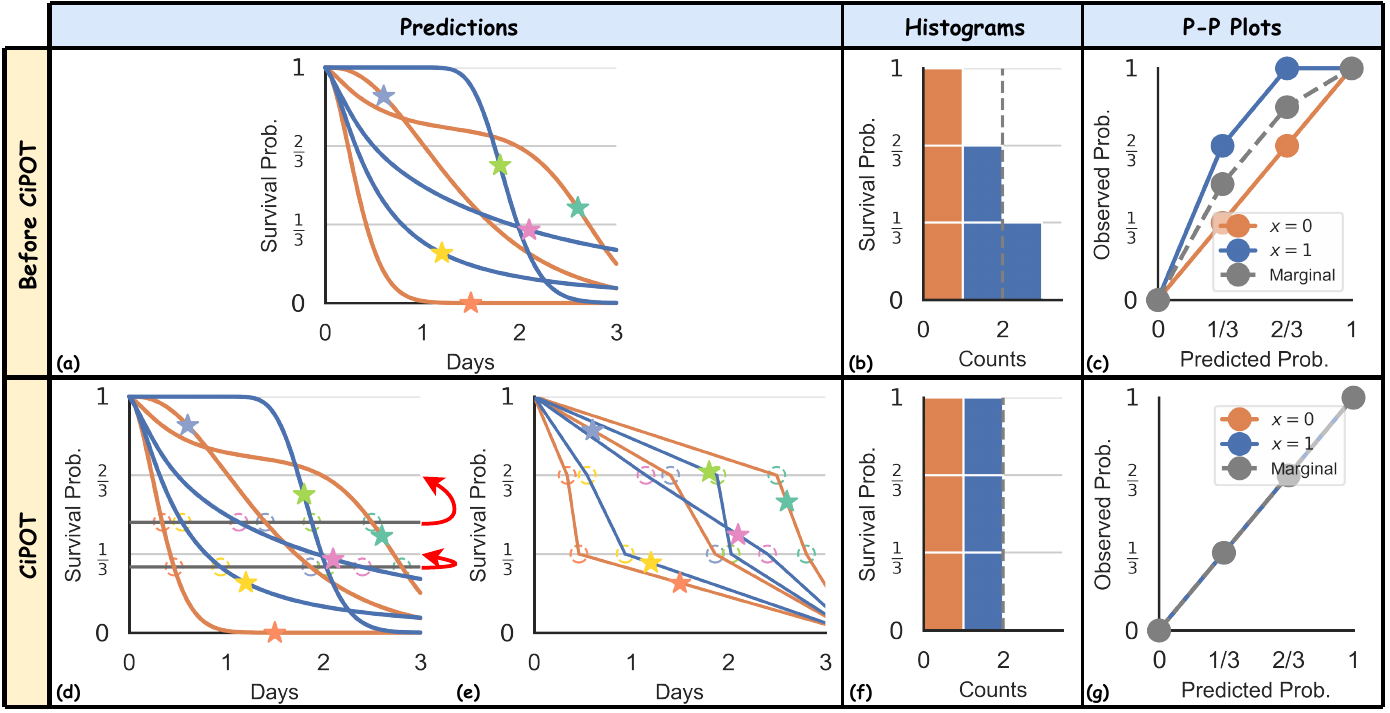}
    \vspace{-0.1in}
    \caption{A visual example of using \CSDiPOT to make the prediction (conditionally)-calibrated. 
    (a)~Initialize ISD predictions from an arbitrary survival algorithm with associated (b)~histograms and (c)~P-P plots.
    (d)~Calculate $\Pct (\rho; \, \Gamma_\Model)$ (grey lines) for all $\rho$s, and find the intersections (hollow points) of the ISD curves and the $\Pct (\rho; \, \Gamma_\Model)$ lines;
    (e)~Generate new ISD by vertically shifting the hollow points to the $\rho$'s level, with associated 
    (f)~histogram and (g)~P-P plots.
    Figure~\ref{fig:csd_compare} provide a side-by-side visual comparison between \texttt{CSD} and our method.}
    \label{fig:iPOT_illust}
\end{figure}

\subsection{Compute conformal scores and calibrate predicting distributions}
\label{sec:method_step2}

We start by sketching how \CSDiPOT deals with only uncensored subjects. 
Within the conformal set, for each subject $i \in \iset^{\text{con}}$, we define a distributional conformity score, wrt the model $\Model$, termed \textit{the predicted Individual survival Probability at Observed Time (iPOT)}:
\begin{align}
\label{eq:score_uncensored}
    \gamma_{i, \Model} \ := \ \hat{S}_\Model (e_i \mid \Bfx{i}).
\end{align}
Here, for uncensored subjects, the observed time corresponds to the event time, $t_i = e_i$.
Recall from Section~\ref{sec:calibration} that predictions from model $\Model$ are marginally calibrated if the iPOT values follow $\mathcal{U}_{[0, 1]}$ -- \ie if we collect the distributional conformity scores for every subject in the conformal set $\Gamma_\Model = \{\gamma_{i, \Model}\}_{i \in \iset^{\text{con}}}$, the $\rho$-th percentile value in this set should be equal to exactly $\rho$. If so, no post processing adjustments are necessary.

In general, of course, the estimated Individualized Survival Distributions (ISDs) $\hat{S}(t\mid \Bfx{i})$ may not perfectly align with the true distributions $S(t\mid \Bfx{i})$ from the oracle.
Therefore, for a testing subject with index $n+1$, we can simply apply the following adjustment to its estimated ISD: 
\begin{align}
\label{eq:adjustment}
    \Tilde{S}_\Model^{-1}(\rho \mid \Bfx{n+1}) \ := \ \hat{S}_\mathcal{M}^{-1} \left(\, \Pct(\rho; \, \Gamma_\Model )\mid \Bfx{n+1} \, \right), \quad \forall \rho \in (0, 1).
\end{align}
Here, $\Pct(\rho; \, \Gamma_\Model)$ calculates the $\frac{\left\lceil \rho (|\Data^{\text{con}}| + 1) \right\rceil}{|\Data^{\text{con}}|}$-th empirical percentile of $\Gamma_{\Model}$. 
This adjustment aims to re-calibrate the estimated ISD based on the empirical distribution of the conformity scores.

Visually, this adjustment involves three procedures: 
\begin{enumerate}[label=(\roman*)]
    \item It first identifies the empirical percentiles of the conformity scores -- $\Pct(\frac{1}{3}; \, \Gamma_\Model)$ and $\Pct(\frac{2}{3}; \, \Gamma_\Model)$, illustrated by the two grey lines at 0.28 and 0.47 in Figure~\ref{fig:iPOT_illust}(d), respectively -- which uniformly divide the stars according to their vertical locations; 
    \item It then determines the corresponding times on the predicted ISDs that match these empirical percentiles (the hollow circles, where each ISD crosses the horizontal line);
    \item Finally, the procedure shifts the empirical percentiles (grey lines) to the appropriate height of desired percentiles ($\frac{1}{3}$ and $\frac{2}{3}$), along with all the circles. This operation is indicated by the vertical shifts of the hollow points, depicted with curved red arrows in Figure~\ref{fig:iPOT_illust}(d).
\end{enumerate}

This adjustment results in the post-processed curves depicted in Figure~\ref{fig:iPOT_illust}(e). 
It shifts the vertical position of the \textcolor{green}{\textbf{green star} $\bigstar$} from the interval $[\frac{1}{3}, \frac{2}{3}]$ to $[\frac{2}{3}, 1]$, and the \textcolor{pink}{\textbf{pink star} $\bigstar$} from $[0, \frac{1}{3}]$ to $[\frac{1}{3}, \frac{2}{3}]$.
These shifts ensure that the calibration histograms and P-P plots achieve a uniform distribution (both marginally on the whole population and conditionally on each group) across the defined intervals.

After generating the post-processed curves, we can apply a final step, which involves transforming the inverse ISD function back into the ISD function for the testing subject:
\begin{align}
\label{eq:reverse_ISD}
     \Tilde{S}_\Model(t \mid \Bfx{n+1}) \ = \ \inf \{ \, \rho: \Tilde{S}_\Model^{-1}(\rho \mid \Bfx{n+1}) \leq t  \, \} .
\end{align}

The simple visual example in Figure~\ref{fig:iPOT_illust} shows only two percentiles created at $\frac{1}{3}$ and $\frac{2}{3}$.
In practical applications, the user provides a predefined set of percentiles, $\mathcal{P}$, to adjust the ISDs.
The choice of $\mathcal{P}$ can slightly affect the resulting survival distributions, each capable of achieving provable distribution calibration; see ablation study \#2 in Appendix~\ref{appendix:ablations} for how $\mathcal{P}$ affects the performance.

\subsection{Extension to censorship}
It is challenging to incorporate censored instances into the analysis as we do not observe their true event times, $e_i$, which means we cannot directly apply conformity score in~\eqref{eq:score_uncensored} and the subsequent conformal steps.
Instead, we only observe the censoring times, which serve as lower bounds of the event times.

Given the monotonic decreasing property of the ISD curves, the iPOT value for a censored subject, \ie $\hat{S}_\Model(t_i \mid \Bfx{i}) = \hat{S}_\Model(c_i \mid \Bfx{i})$, now serves as the upper bound of $\hat{S}_\Model(e_i \mid \Bfx{i})$. Therefore, given the prior knowledge that $\hat{S}_\Model (e_i \mid \Bfx{i}) \sim \mathcal{U}_{[0, 1]}$, the observation of the censoring time updates the possible range of this distribution. Given that $\hat{S}_\Model(e_i \mid \Bfx{i})$ must be less than or equal to $\hat{S}_\Model(c_i \mid \Bfx{i})$, the updated posterior distribution follows $\hat{S}_\Model (e_i \mid \Bfx{i}) \sim \mathcal{U}_{[0, \hat{S}_\Model (c_i \mid \Bfx{i})]}$.


Following the calibration calculation in~\cite{haider2020effective}, where censored patients are evenly "blurred" across subsequent bins of $\hat{S} (c_i \mid \Bfx{i})$, 
our approach uses the above posterior distribution to uniformly draw $R$ potential conformity scores for a censored subject, for some constant $R \in \mathbbm{Z}^+$. 
Specifically, for a censored subject, we calculate the conformity scores as:
\begin{align*}
    \boldsymbol{\gamma}_{i, \Model} \ = \ \hat{S}_\Model(c_i \mid x_i) \cdot \boldsymbol{u}_{R}, \quad \text{where} \quad  \boldsymbol{u}_{R} \ = \ \left[ \, 0 / R \, , \, 1/ R \, ,  \ldots, R/ R \, \right].
\end{align*}
Here, $\boldsymbol{u}_{R}$ is a pseudo-uniform vector to mimic the uniform sampling operation, 
significantly reducing computational overhead compared to actual uniform distribution sampling.
For uncensored subjects, we also need to apply a similar sampling strategy to maintain a balanced censoring rate within the conformal set. 
Because the exact iPOT value is known and deterministic for uncensored subjects, sampling involves directly drawing from a degenerate distribution centered at $\hat{S}_\Model (e_i \mid \Bfx{i})$ -- \ie just drawing $\hat{S}_\Model (e_i \mid \Bfx{i})$ $R$ times.
The pseudo-code for implementing the \CSDiPOT process with censoring is outlined in Algorithm~\ref{alg:csd_ipot} in Appendix~\ref{appendix:alg}.

Note that the primary computational demand of this method stems from the optional interpolation and extrapolation of the piecewise constant ISD predictions. 
Calculating the conformity scores and estimating their percentiles incur negligible costs in terms of both time and space, once the right-continuous survival distributions are established. We provide computational analysis in Appendix~\ref{appendix:computation}.

\subsection{Theoretical analysis}
\label{sec:method_theory}
Here we discuss the theoretical properties of \texttt{CiPOT}. Unlike \texttt{CSD}~\cite{qi2024conformalized}, which adjusts the ISD curves horizontally (changing the times, for a fixed percentile), our refined version scales the ISD curves vertically. This vertical adjustment leads to several advantageous properties. In particular, we highlight why our method is expected to yield superior performance in terms of marginal and conditional calibration compared to \texttt{CSD} \cite{qi2024conformalized}. 
Table~\ref{tab:alg_comp} summarizes the properties of the two methods.

\begin{table}[ht]
\centering
\caption{Properties of \texttt{CSD} and \CSDiPOT. Note that the calibration guarantees refer to asymptotic calibration guarantees. $^\dagger$See Appendix~\ref{appendix:computation}.}
\label{tab:alg_comp}
\setlength\tabcolsep{3.9pt}
\begin{tabular}{lcccccc}
\toprule
Methods  & \begin{tabular}[c]{@{}c@{}}Marginal\\calibration\\ guarantee\end{tabular}   & \begin{tabular}[c]{@{}c@{}}Conditional\\ calibration\\ guarantee\end{tabular} & Monotonic & \begin{tabular}[c]{@{}c@{}}Harrell's\\discrimination\\guarantee\end{tabular} & \begin{tabular}[c]{@{}c@{}}Antolini's\\discrimination\\guarantee\end{tabular} & \begin{tabular}[c]{@{}c@{}}Space\\ complexity$^\dagger$\end{tabular}\\   \midrule
\texttt{CSD}~\cite{qi2024conformalized}    & $\xmark$  & $\xmark$ & $\xmark$      & $\checkmark$       & $\xmark$      & $O(N\cdot|\mathcal{P}|\cdot R)$              \\
\CSDiPOT   & $\checkmark$   & $\checkmark$   & $\checkmark$        & $\xmark$  & $\checkmark$  & $O(N\cdot R)$                \\
\bottomrule
\end{tabular}
\end{table}

\paragraph{Calibration} 
\CSDiPOT differs from \texttt{CSD} in two major ways: \CSDiPOT (i)~essentially samples the event time from $\hat{S}_\Model ( t \mid t> c_i, \ \Bfx{i} )$ for a censored subject, and (ii)~subsequently converts these times into corresponding survival probability values on the curve.

The first difference contrasts with the \texttt{CSD} method, which samples from a conditional KM distribution, $S_{\text{KM}} (\, t \mid t> c_i\,)$, assuming a homoskedastic survival distribution across subjects (where the conditional KM curves have the same shape and the random disturbance of $e_i$ is independent of the features $\Bfx{i}$). 
However, \CSDiPOT differs by considering the heteroskedastic nature of survival distributions $\hat{S}_\Model (t \mid t> c_i, \ \Bfx{i} )$.
For instance, consider the symptom onset times following exposure to the COVID-19 virus. 
Older adults, who may exhibit more variable immune responses, could experience a broader range of onset times compared to younger adults, whose symptom onset times are generally more consistent~\citep{challenger2022modelling}.
By integrating this feature-dependent variability, \CSDiPOT captures the inherent heteroskedasticity of survival distributions and adjusts the survival estimates accordingly, which helps with conditional calibration. 

Furthermore, by transforming the times into the survival probability values on the predicted ISD curves (the second difference), we mitigate the trouble of inaccurate interpolation and extrapolation of the distribution. 
This approach is particularly useful when the conditional distribution terminates at a relatively high probability, where extrapolating beyond the observed range is problematic due to the lack of data for estimating the tail behavior.
Different extrapolation methods, whether parametric or spline-based, can yield widely varying behaviors in the tails of the distribution, potentially leading to significant inaccuracies in survival estimates.
However, by converting event times into survival percentiles, \CSDiPOT circumvents these issues.
This method capitalizes on the probability integral transform~\cite{angus1994probability}, which ensures that regardless of the specific tail behavior of a survival function, its inverse probability values will follow a uniform distribution. 

The next results state that the output of our method has asymptotic marginal calibration, with necessary assumptions (exchangeability, conditional independent censoring, and continuity). 
We also prove the asymptotic conditional calibrated guarantee for \texttt{CiPOT}. The proofs of these two results are inspired by the standard conformal prediction literature~\citep{romano2019conformalized, izbicki2020flexible}, with adequate modifications to accommodate our method. We refer the reader to Appendix~\ref{appendix:calibration} for the complete proof.


\begin{theorem}[Asymptotic marginal calibration]
\label{theorem:margin_cal}
    If the instances in $\Data$ are exchangeable, and follow the conditional independent censoring assumption, then for a new instance $n+1$, $\forall \ \rho_1 < \rho_2 \in [0, 1]$, 
    \begin{equation*}
        \rho_2 - \rho_1 \quad \leq \quad \prob\left(\Tilde{S}_\mathcal{M} (t_{n+1} \mid \Bfx{n+1}) \in [\rho_1, \rho_2]\right) \quad \leq \quad \rho_2 - \rho_1 + \frac{1}{|\Data^{\textnormal{con}}|+1} .
    \end{equation*}
\end{theorem}




\begin{theorem}[Asymptotic conditional calibration] 
\label{theorem:conditional_cal}
In addition to the assumptions in Theorem~\ref{theorem:margin_cal}, if (i)~the non-processed prediction $\hat{S}_\Model(t\mid \Bfx{i})$ is a consistent survival estimator; (ii)~its inverse function is differentiable; and (iii)~the 1st derivation of the inverse function is bounded by a constant, then the \CSDiPOT process will achieve asymptotic conditional distribution calibration.
\end{theorem}

\paragraph{Monotonicity} Unlike \texttt{CSD}, \CSDiPOT does not face any non-monotonic issues for the post-processed curves as long as the original ISD predictions are monotonic; see proof in Appendix~\ref{appendix:monotonic}.

\begin{theorem}
\label{theorem:monotonic}
    \CSDiPOT process preserves the monotonic decreasing property of the ISD. 
\end{theorem}

\texttt{CSD}, built on the Conformalized Quantile Regression (CQR) framework, struggles with the common issue of non-monotonic quantile curves (refer to Appendix D.2 in~\cite{qi2024conformalized} and our Appendix~\ref{appendix:monotonic}). 
While some methods, like the one proposed by \citet{chernozhukov2010quantile}, address this issue by rearranging quantiles, they can be computationally intensive and risk (slightly) recalibrating and distorting discrimination in the rearranged curves. 
By inherently maintaining monotonicity, \CSDiPOT not only enhances computational efficiency but also avoids these risks. 

\paragraph{Discrimination} 
\citet{qi2024conformalized} demonstrated that \texttt{CSD} theoretically guarantees the preservation of the original model's discrimination performance in terms of Harrell's concordance index (C-index)~\cite{harrell1984regression}. 
However, \CSDiPOT lacks this property; see Appendix~\ref{appendix:discrimination} for details.

As \CSDiPOT vertically scales the ISD curves, it preserves the relative order of survival probabilities at any single time point.
This preservation means that the discrimination power, measured by the area under the receiver operating characteristic (AUROC) at any time, remains intact (Theorem~\ref{theorem:auroc}). 
Furthermore, Antolini's time-dependent C-index ($C^{td}$)~\cite{antolini2005time}, which represents a weighted average AUROC across all time points, is also guaranteed to be maintained by our method (Lemma~\ref{lemma:c_index_td}).
As a comparison, \texttt{CSD} does not have such a guarantee for neither AUROC nor $C^{td}$.

\section{Evaluation metrics}
\label{sec:metrics}
We measure discrimination using Harrell's C-index \cite{harrell1984regression}, rather than Antolini's $C^{td}$~\cite{antolini2005time}, as Lemma~\ref{lemma:c_index_td} already established that $C^{td}$ is not changed by \texttt{CiPOT}. We aim to assess our performance using a measure that represents a relative weakness of our method. 

As to the calibration metrics, the marginal calibration score evaluated on the test set $\Data^{\text{test}}$ is calculated as~\cite{haider2020effective, qi2024conformalized}: \\[-1ex]
\begin{equation}
\label{eq:cal_margin}
    \text{Cal}_{\text{margin}} (\hat{S}; \mathcal{P}) \ =\ 
    \frac{1}{|\mathcal{P}|} \sum_{\rho \in \mathcal{P}} \left( \,
    \prob \left( \hat{S} (e_i \mid \Bfx{i} ) \in [0, \rho]\, , \,  i \in \iset^{\text{test}}  \right)
    -\ \rho  \, \right)^2,
\end{equation}
where $\prob ( \hat{S} (e_i \mid \Bfx{i} ) \in [0, \rho]\, , \,  i \in \iset^{\text{test}}  )$ is calculated by combining~\eqref{eq:cal_uncensored} and~\eqref{eq:cal_censored}; see~\eqref{eq:cal_formal} in Appendix~\ref{appendix:calibration_intro}.
Based on the marginal calibration formulation, a natural way for evaluating the conditional calibration could be: (i)~heuristically define a finite feature space set $\{\mathbb{S}_1, \mathbb{S}_2, \ldots\}$ -- \eg $\mathbb{S}_1$ is the set of divorced elder males, $\mathbb{S}_2$ is females with 2 children, etc.;
and (ii)~calculate the worst calibration score on all the predefined sub-spaces.
This is similar to fairness settings, researchers normally select age, sex, or race as the sensitive attributes to form the feature space.
However, this metric does not scale to higher-dimensional settings because it is challenging to create the feature space set that contains all possible combinations of the features.

Motivated by~\citet{romano2020classification}, we proposed a worst-slab distribution calibration, $\text{Cal}_{\text{ws}}$. 
We start by partition the testing set into a 25\% exploring set $\Data^{\text{explore}}$ and a 75\% exploiting set $\Data^{\text{exploit}}$. The exploring set is then used to find the worst calibrated sub-region in the feature space $\mathbb{R}^{d}$:
\begin{gather*}
    \mathbb{S}_{\boldsymbol{v}, a, b} \ = \ \{ \, \Bfx{i} \in \mathbb{R}^{d}: a \leq \boldsymbol{v}^\intercal \Bfx{i}\leq b \, \} \quad \text{and} \quad \prob\left( \, \Bfx{i} \in \mathbb{S}_{\boldsymbol{v}, a, b}, i \in \iset^{\text{explore}} \, \right) \ \geq \ \kappa, \\
    \text{where} \quad 
     \boldsymbol{v}, a, b = \argmax_{\boldsymbol{v}\in \mathbb{R}^d, a<b\in \mathbb{R}} 
     \frac{1}{|\mathcal{P}|} \sum_{\rho \in \mathcal{P}} \left( 
     \prob \left( \hat{S} (e_i \mid \Bfx{i} ) \in [0, \rho]\, , \,  i \in \iset^{\text{explore}} \, , \, \Bfx{i} \in \mathbb{S}_{\boldsymbol{v}, a, b}) \right)
      - \rho \right)^2.
\end{gather*}
In practice, the parameters $\boldsymbol{v}$, $a$, and $b$ are chosen adversarially by sampling i.i.d. vectors $\boldsymbol{v}$ on the unit sphere in $\mathbb{R}^d$ then finding the $\argmax$ using a grid search on the exploring set. 
$\kappa$ is a predefined threshold to ensure that we only consider slabs that contain at least $\kappa\%$ of the instances (so that we do not encounter a pregnant-man situation). Given this slab, we can calculate the conditional calibration score on the evaluation set for this slab:
\begin{align}
\label{eq:cal_ws}
    \text{Cal}_{\text{ws}}(\hat{S}; \mathcal{P}, \mathbb{S}_{\boldsymbol{v}, a, b})
    \ = \ \frac{1}{|\mathcal{P}|} \sum_{\rho \in \mathcal{P}} \left( \,
    \prob \left( \hat{S} (e_i \mid \Bfx{i} ) \in [0, \rho]\, , \,  i \in \iset^{\text{exploit}} \, , \, \Bfx{i} \in \mathbb{S}_{\boldsymbol{v}, a, b}) \right)
    \ -\ \rho \, \right)^2.
\end{align}
Besides the above metrics, we also evaluate using other commonly used metrics: integrated Brier score (IBS)~\cite{graf1999assessment}, 
and mean absolute error with pseudo-observation (MAE-PO)~\citep{qi2023an}; see Appendix~\ref{appendix:metric}.

\section{Experiments}
\label{sec:exp}
The implementation of \CSDiPOT method, worst-slab distribution calibration score, and the code to reproduce all experiments in this section are available at \url{https://github.com/shi-ang/MakeSurvivalCalibratedAgain}.

\subsection{Experimental setup}

\paragraph{Datasets} 
We use 15 datasets to test the effectiveness of our method.
Table~\ref{tab:data_comp} in Appendix~\ref{appendix:datasets} summarizes the dataset statistics, and Appendix~\ref{appendix:datasets} also contains details of preprocessing steps, KM curves, and histograms of event/censor times. 
Compared with~\cite{qi2024conformalized}, we added datasets with high censoring rates (>60\%) and ones whose KM ends with high probabilities (>50\%).

\vspace{-0.1in}

\paragraph{Baselines} 
We compared 7 survival algorithms: 
\emph{AFT}~\cite{stute1993consistent},
\emph{GB}~\cite{hothorn2006survival},
\emph{DeepSurv}~\cite{katzman2018deepsurv}, 
\emph{N-MTLR}~\cite{fotso2018deep},
\emph{DeepHit}~\cite{lee2018deephit}, 
\emph{CoxTime}~\cite{kvamme2019time},
and \emph{CQRNN}~\cite{pearce2022censored}.
We also include KM as a benchmark (empirical lower bound) for marginal calibration, which is known to achieve perfect marginal calibration~\cite{haider2020effective, qi2024conformalized}.
Appendix~\ref{appendix:baselines} describes the implementation details and hyperparameter settings.

\vspace{-0.1in}

\paragraph{Procedure} 
We divided the data into a training set (90\%) and a testing set (10\%) using a stratified split to balance time $t_i$ and censor indicator $\delta_i$.
We also reserved a balanced 10\% validation subset from the training data for hyperparameter tuning and early stopping. 
This procedure was replicated across 10 random splits for each dataset.


\subsection{Experimental results}
Due to space constraints, the main text presents partial results for datasets with high censoring rates and high KM ending probabilities (\texttt{HFCR}, \texttt{FLCHAIN}, \texttt{Employee}, \texttt{MIMIC-IV}).
Notably, \emph{CQRNN} did not converge on the \texttt{MIMIC-IV} dataset. 
Thus, we conducted a total of 104 method comparisons ($15 \ \text{datasets} \times 7 \ \text{baselines} - 1$). 
Table 2 offers a detailed performance summary of \CSDiPOT versus baselines and \texttt{CSD} method across these comparisons. 
Appendix~\ref{appendix:results} presents the complete results.

\vspace{-0.1in}

\begin{figure}[t]
    \centering
    \vspace{-0.1in}
    \includegraphics[width=\textwidth]{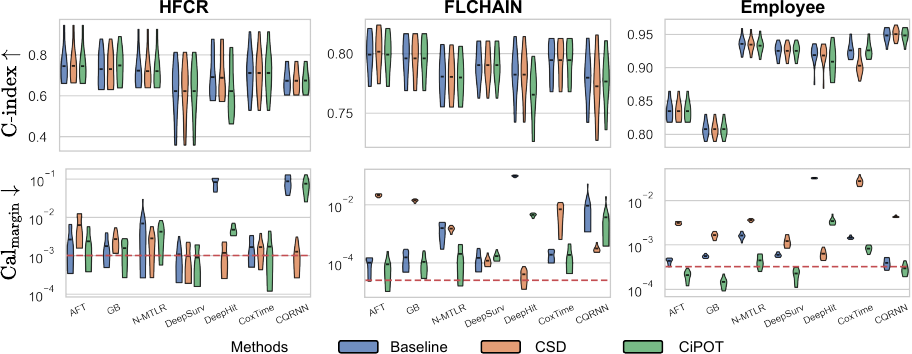}
    \vspace{-0.1in}
    \caption{Violin plots of C-index and $\text{Cal}_{\text{margin}}$ performance of our method (\texttt{CiPOT}) and benchmarks. 
    The shape of each violin plot represents the probability density of the performance scores, with the black bar inside the violin indicating the mean performance. 
    The red dashed lines in the lower panels represent the mean calibration performance for KM, serving as an empirical lower limit.} 
    \label{fig:cindex_dcal_main}
\end{figure}

\paragraph{Discrimination}
The upper panels of Figure~\ref{fig:cindex_dcal_main} indicate minimal differences in the C-index among the methods, with notable exceptions primarily involving \emph{DeepHit}. 
Specifically, \CSDiPOT matched the baseline C-index in 75 instances and outperformed it in 7 out of 104 comparisons. 
This suggests that \CSDiPOT maintains discriminative performance in approximately 79\% of the cases.

\begin{table}[t]
\centering
\caption{Performance summary of \texttt{CiPOT}. Values in parentheses indicate statistically significant differences ($p<0.05$ using a two-sided $t$-test). A tie means the first 3 significant digits are the same. $^\ddagger$The total number of comparisons for $\text{Cal}_{\text{ws}}$ is 69, while it is 104 for the other metrics.}
\label{tab:summary}
\begin{tabular}{ccccccc}
\toprule
                                               &      & C-index & $\text{Cal}_{\text{margin}}$ & $\text{Cal}_{\text{ws}} \, ^\ddagger $ & IBS & MAE-PO \\ \midrule
\multirow{3}{*}{Compare with Baselines}         & Win  & 7 (0)   & \textbf{95 (50)}   & \textbf{64 (29)}   & \textbf{63 (14)}   & \textbf{54 (8)}    \\
                                               & Lose & 22 (0)  & 9 (1)     & 5 (1)     & 23 (0)    & 17 (0)    \\
                                               & Tie  & \textbf{75}      & 0         & 0         & 18        & 33    \\ \midrule
\multirow{3}{*}{Compare with \texttt{CSD}}     & Win  & 11 (1)  & \textbf{68 (37)}   & \textbf{51 (26)}   & \textbf{53 (15)}   & \textbf{39 (8)}    \\
                                               & Lose & 26 (0)  & 36 (20)   & 18 (7)    & 35 (11)   & 39 (4)    \\
                                               & Tie  & \textbf{67}      & 0         & 0         & 16        & 26    \\
\bottomrule
\end{tabular}
\end{table}

\vspace{-0.1in}

\paragraph{Marginal Calibration}
The lower panels of Figure~\ref{fig:cindex_dcal_main} show significant improvements in marginal calibration with \texttt{CiPOT}.
It often achieved near-optimal performance, as marked by the red dashed lines.
Table~\ref{tab:summary} also shows that \CSDiPOT provided better marginal calibration than the baselines in 95 (and significantly in 50) out of 104 comparisons (91\%).

\texttt{CiPOT}'s marginal calibration was better than \texttt{CSD} most of the time (68/104, 65\%). 
The cases where \texttt{CSD} performs better typically involve models like \emph{DeepHit} or \emph{CQRNN}.
This shows that our approach often does not perform as well as \texttt{CSD} when the original model is heavily miscalibrated, which suggests a minor limitation of our method.
Appendix~\ref{appendix:miscalibrated_models} discusses why our method is sub-optimal for these models.

\begin{figure}[t]
    \centering
    \vspace{-0.1in}
    \includegraphics[width=\textwidth]{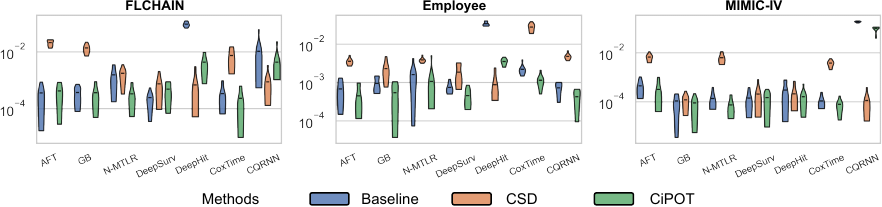}
    \vspace{-0.1in}
    \caption{Violin plots of $\text{Cal}_{\text{ws}}$ performance, where the shape and black bars represent the density and mean. Smaller values represent better performance. Note \emph{CQRNN} did not converge on \texttt{MIMIC-IV}.}
    \label{fig:dcal_ws_main}
\end{figure}

\vspace{-0.1in}

\paragraph{Conditional Calibration} 
For \textit{small} datasets (sample size $< 1000$), in some random split, we can find a worst-slab region $\mathbb{S}_{\boldsymbol{v}, a, b}$ on the exploring set with $\mathbb{P} (\Bfx{i} \in \mathbb{S}_{\boldsymbol{v}, a, b}, \, i\in \iset^{\text{explore}}) \geq 33\% $ but still no subjects in this region in the exploiting set. 
This is probably because we only ensure that the times and censored indicators are balanced during the partition, however, the features can still be unbalanced. 
Therefore, we only evaluated conditional calibration on the 10 larger datasets, resulting in 69 comparisons.
Among them, \CSDiPOT improved conditional calibration in 64 cases (93\%) compared to baselines and in 51 cases (74\%) compared to \texttt{CSD}.

\vspace{-0.1in}

\paragraph{Case Study}
We provide 4 case studies in Figure~\ref{fig:case_study} in Appendix~\ref{appendix:results}, where \texttt{CSD} leads to significant miscalibration within certain subgroups, and \CSDiPOT can effectively generate more conditional calibrated predictions in those groups. These examples show that \texttt{CSD}’s miscalibration is always located at the low-probability regions, which corresponds to our statement (in Section~\ref{sec:method_theory}) that the conditional KM sampling method that \texttt{CSD} used is problematic for the tail of the distribution.

\vspace{-0.1in}

\paragraph{Other Metrics} 
Results in Table~\ref{tab:summary} and Appendix~\ref{appendix:results} 
show that \CSDiPOT also showed improvement in both IBS and MAE-PO, outperforming 63 and 54 out of 104 comparisons, respectively.

\vspace{-0.1in}

\paragraph{Computational Analysis}
Appendix~\ref{appendix:computation} shows the comprehensive results and experimental setup. 
In summary, \CSDiPOT significantly reduces the space consumption and running time.

\vspace{-0.1in}

\paragraph{Ablation Studies}
We conducted two ablation studies to assess (i)~the impact of the repetitions value ($R$) and (ii)~the impact of predefined percentiles ($\mathcal{P}$) on the method; see Appendix~\ref{appendix:ablations}.

\section{Conclusions}
\label{sec:conclusion}
Discrimination and marginal calibration are two fundamental yet distinct elements in survival analysis. 
While marginal calibration is feasible, it overlooks accuracy across different groups distinguished by specific features.
In this paper, we emphasize the importance of conditional calibration for practical applications and propose a principled metric for this purpose. 
By generating conditionally calibrated Individual Survival Distributions (ISDs), we can better communicate the uncertainty in survival analysis models, enhancing their reliability, fairness, and real-world applicability.


We therefore define the Conformalized survival distribution using Individual Survival Probability at Observed Time (\texttt{CiPOT}) -- a post-processing framework that enhances both marginal and conditional calibration without compromising discrimination. It addresses common issues in prior methods, particularly under high censoring rates or when the Kaplan-Meier curve terminates at a high probability. Moreover, this post-processing adjusts the ISDs by adapting the heteroskedasticity of the distribution, leading to asymptotic conditional calibration.
Our extensive empirical tests confirm that \CSDiPOT significantly improves both marginal and conditional performance without diminishing the models' discriminative power.

\begin{ack}
This research received support from the Natural Science and Engineering Research Council of Canada (NSERC), the Canadian Institute for Advanced Research (CIFAR), and the Alberta Machine Intelligence Institute (Amii). The authors extend their gratitude to the anonymous reviewers for their insightful feedback and valuable suggestions.
\end{ack}

\bibliography{main}

\begin{thebibliography}{63}
\providecommand{\natexlab}[1]{#1}
\providecommand{\url}[1]{\texttt{#1}}
\expandafter\ifx\csname urlstyle\endcsname\relax
  \providecommand{\doi}[1]{doi: #1}\else
  \providecommand{\doi}{doi: \begingroup \urlstyle{rm}\Url}\fi

\bibitem[Harrell~Jr et~al.(1984)Harrell~Jr, Lee, Califf, Pryor, and Rosati]{harrell1984regression}
Frank~E Harrell~Jr, Kerry~L Lee, Robert~M Califf, David~B Pryor, and Robert~A Rosati.
\newblock Regression modelling strategies for improved prognostic prediction.
\newblock \emph{Statistics in medicine}, 3\penalty0 (2):\penalty0 143--152, 1984.

\bibitem[Harrell~Jr et~al.(1996)Harrell~Jr, Lee, and Mark]{harrell1996multivariable}
Frank~E Harrell~Jr, Kerry~L Lee, and Daniel~B Mark.
\newblock Multivariable prognostic models: issues in developing models, evaluating assumptions and adequacy, and measuring and reducing errors.
\newblock \emph{Statistics in medicine}, 15\penalty0 (4):\penalty0 361--387, 1996.

\bibitem[Haider et~al.(2020)Haider, Hoehn, Davis, and Greiner]{haider2020effective}
Humza Haider, Bret Hoehn, Sarah Davis, and Russell Greiner.
\newblock Effective ways to build and evaluate individual survival distributions.
\newblock \emph{Journal of Machine Learning Research}, 21\penalty0 (85):\penalty0 1--63, 2020.

\bibitem[Chapfuwa et~al.(2020)Chapfuwa, Tao, Li, Khan, Chandross, Pencina, Carin, and Henao]{chapfuwa2020calibration}
Paidamoyo Chapfuwa, Chenyang Tao, Chunyuan Li, Irfan Khan, Karen~J Chandross, Michael~J Pencina, Lawrence Carin, and Ricardo Henao.
\newblock Calibration and uncertainty in neural time-to-event modeling.
\newblock \emph{IEEE transactions on neural networks and learning systems}, 2020.

\bibitem[Avati et~al.(2020)Avati, Duan, Zhou, Jung, Shah, and Ng]{avati2020countdown}
Anand Avati, Tony Duan, Sharon Zhou, Kenneth Jung, Nigam~H Shah, and Andrew~Y Ng.
\newblock Countdown regression: sharp and calibrated survival predictions.
\newblock In \emph{Uncertainty in Artificial Intelligence}, pages 145--155. PMLR, 2020.

\bibitem[Goldstein et~al.(2020)Goldstein, Han, Puli, Perotte, and Ranganath]{goldstein2020x}
Mark Goldstein, Xintian Han, Aahlad Puli, Adler Perotte, and Rajesh Ranganath.
\newblock X-cal: Explicit calibration for survival analysis.
\newblock \emph{Advances in neural information processing systems}, 33:\penalty0 18296--18307, 2020.

\bibitem[Kamran and Wiens(2021)]{kamran2021estimating}
Fahad Kamran and Jenna Wiens.
\newblock Estimating calibrated individualized survival curves with deep learning.
\newblock \emph{Proceedings of the AAAI Conference on Artificial Intelligence}, 35\penalty0 (1):\penalty0 240--248, May 2021.

\bibitem[Qi et~al.(2024)Qi, Yu, and Greiner]{qi2024conformalized}
Shi-Ang Qi, Yakun Yu, and Russell Greiner.
\newblock Conformalized survival distributions: A generic post-process to increase calibration.
\newblock In \emph{Proceedings of the 41st International Conference on Machine Learning}, volume 235, pages 41303--41339. PMLR, 21--27 Jul 2024.

\bibitem[Verma and Rubin(2018)]{verma2018fairness}
Sahil Verma and Julia Rubin.
\newblock Fairness definitions explained.
\newblock In \emph{Proceedings of the international workshop on software fairness}, pages 1--7, 2018.

\bibitem[Vovk et~al.(2005)Vovk, Gammerman, and Shafer]{vovk2005algorithmic}
Vladimir Vovk, Alexander Gammerman, and Glenn Shafer.
\newblock \emph{Algorithmic learning in a random world}, volume~29.
\newblock Springer, 2005.

\bibitem[Romano et~al.(2019)Romano, Patterson, and Candes]{romano2019conformalized}
Yaniv Romano, Evan Patterson, and Emmanuel Candes.
\newblock Conformalized quantile regression.
\newblock \emph{Advances in neural information processing systems}, 32, 2019.

\bibitem[Cand{\`e}s et~al.(2023)Cand{\`e}s, Lei, and Ren]{candes2023conformalized}
Emmanuel Cand{\`e}s, Lihua Lei, and Zhimei Ren.
\newblock Conformalized survival analysis.
\newblock \emph{Journal of the Royal Statistical Society Series B: Statistical Methodology}, 85\penalty0 (1):\penalty0 24--45, 2023.

\bibitem[Angus(1994)]{angus1994probability}
John~E Angus.
\newblock The probability integral transform and related results.
\newblock \emph{SIAM review}, 36\penalty0 (4):\penalty0 652--654, 1994.

\bibitem[Kaplan and Meier(1958)]{kaplan1958nonparametric}
Edward~L Kaplan and Paul Meier.
\newblock Nonparametric estimation from incomplete observations.
\newblock \emph{Journal of the American statistical association}, 53\penalty0 (282):\penalty0 457--481, 1958.

\bibitem[Visseren et~al.(2021)Visseren, Mach, Smulders, Carballo, Koskinas, B{\"a}ck, Benetos, Biffi, Boavida, Capodanno, et~al.]{visseren20212021}
Frank~LJ Visseren, Fran{\c{c}}ois Mach, Yvo~M Smulders, David Carballo, Konstantinos~C Koskinas, Maria B{\"a}ck, Athanase Benetos, Alessandro Biffi, Jose-Manuel Boavida, Davide Capodanno, et~al.
\newblock {2021 ESC Guidelines on cardiovascular disease prevention in clinical practice: Developed by the Task Force for cardiovascular disease prevention in clinical practice with representatives of the European Society of Cardiology and 12 medical societies With the special contribution of the European Association of Preventive Cardiology (EAPC)}.
\newblock \emph{European heart journal}, 42\penalty0 (34):\penalty0 3227--3337, 2021.

\bibitem[Klein and Moeschberger(2006)]{klein2006survival}
John~P Klein and Melvin~L Moeschberger.
\newblock \emph{Survival analysis: techniques for censored and truncated data}.
\newblock Springer Science \& Business Media, 2006.

\bibitem[Liu et~al.(2021)Liu, Zhou, and Lee]{liu2021ipdfromkm}
Na~Liu, Yanhong Zhou, and J~Jack Lee.
\newblock {IPDfromKM: reconstruct individual patient data from published Kaplan-Meier survival curves}.
\newblock \emph{BMC medical research methodology}, 21\penalty0 (1):\penalty0 111, 2021.

\bibitem[Lei and Wasserman(2014)]{lei2014distribution}
Jing Lei and Larry Wasserman.
\newblock Distribution-free prediction bands for non-parametric regression.
\newblock \emph{Journal of the Royal Statistical Society Series B: Statistical Methodology}, 76\penalty0 (1):\penalty0 71--96, 2014.

\bibitem[Foygel~Barber et~al.(2021)Foygel~Barber, Candes, Ramdas, and Tibshirani]{foygel2021limits}
Rina Foygel~Barber, Emmanuel~J Candes, Aaditya Ramdas, and Ryan~J Tibshirani.
\newblock The limits of distribution-free conditional predictive inference.
\newblock \emph{Information and Inference: A Journal of the IMA}, 10\penalty0 (2):\penalty0 455--482, 2021.

\bibitem[Lei et~al.(2018)Lei, G’Sell, Rinaldo, Tibshirani, and Wasserman]{lei2018distribution}
Jing Lei, Max G’Sell, Alessandro Rinaldo, Ryan~J Tibshirani, and Larry Wasserman.
\newblock Distribution-free predictive inference for regression.
\newblock \emph{Journal of the American Statistical Association}, 113\penalty0 (523):\penalty0 1094--1111, 2018.

\bibitem[Sesia and Cand{\`e}s(2020)]{sesia2020comparison}
Matteo Sesia and Emmanuel~J Cand{\`e}s.
\newblock A comparison of some conformal quantile regression methods.
\newblock \emph{Stat}, 9\penalty0 (1):\penalty0 e261, 2020.

\bibitem[Izbicki et~al.(2020)Izbicki, Shimizu, and Stern]{izbicki2020flexible}
Rafael Izbicki, Gilson Shimizu, and Rafael Stern.
\newblock Flexible distribution-free conditional predictive bands using density estimators.
\newblock In \emph{International Conference on Artificial Intelligence and Statistics}, pages 3068--3077. PMLR, 2020.

\bibitem[Chernozhukov et~al.(2021)Chernozhukov, W{\"u}thrich, and Zhu]{chernozhukov2021distributional}
Victor Chernozhukov, Kaspar W{\"u}thrich, and Yinchu Zhu.
\newblock Distributional conformal prediction.
\newblock \emph{Proceedings of the National Academy of Sciences}, 118\penalty0 (48):\penalty0 e2107794118, 2021.

\bibitem[Izbicki et~al.(2022)Izbicki, Shimizu, and Stern]{izbicki2022cd}
Rafael Izbicki, Gilson Shimizu, and Rafael~B. Stern.
\newblock {CD-split and HPD-split: Efficient Conformal Regions in High Dimensions}.
\newblock \emph{Journal of Machine Learning Research}, 23\penalty0 (87):\penalty0 1--32, 2022.

\bibitem[Papadopoulos et~al.(2002)Papadopoulos, Proedrou, Vovk, and Gammerman]{papadopoulos2002inductive}
Harris Papadopoulos, Kostas Proedrou, Volodya Vovk, and Alex Gammerman.
\newblock Inductive confidence machines for regression.
\newblock In \emph{Machine learning: ECML 2002: 13th European conference on machine learning Helsinki, Finland, August 19--23, 2002 proceedings 13}, pages 345--356. Springer, 2002.

\bibitem[Cox(1972)]{cox1972regression}
David~R Cox.
\newblock Regression models and life-tables.
\newblock \emph{Journal of the Royal Statistical Society: Series B (Methodological)}, 34\penalty0 (2):\penalty0 187--202, 1972.

\bibitem[Kvamme et~al.(2019)Kvamme, Borgan, and Scheel]{kvamme2019time}
Havard Kvamme, {\O}rnulf Borgan, and Ida Scheel.
\newblock Time-to-event prediction with neural networks and cox regression.
\newblock \emph{Journal of Machine Learning Research}, 20:\penalty0 1--30, 2019.

\bibitem[Yu et~al.(2011)Yu, Greiner, Lin, and Baracos]{yu2011learning}
Chun-Nam Yu, Russell Greiner, Hsiu-Chin Lin, and Vickie Baracos.
\newblock Learning patient-specific cancer survival distributions as a sequence of dependent regressors.
\newblock \emph{Advances in Neural Information Processing Systems}, 24:\penalty0 1845--1853, 2011.

\bibitem[Lee et~al.(2018)Lee, Zame, Yoon, and van~der Schaar]{lee2018deephit}
Changhee Lee, William~R Zame, Jinsung Yoon, and Mihaela van~der Schaar.
\newblock Deephit: A deep learning approach to survival analysis with competing risks.
\newblock In \emph{Thirty-second AAAI conference on artificial intelligence}, 2018.

\bibitem[Challenger et~al.(2022)Challenger, Foo, Wu, Yan, Marjaneh, Liew, Thwaites, Okell, and Cunnington]{challenger2022modelling}
Joseph~D Challenger, Cher~Y Foo, Yue Wu, Ada~WC Yan, Mahdi~Moradi Marjaneh, Felicity Liew, Ryan~S Thwaites, Lucy~C Okell, and Aubrey~J Cunnington.
\newblock Modelling upper respiratory viral load dynamics of sars-cov-2.
\newblock \emph{BMC medicine}, 20:\penalty0 1--20, 2022.

\bibitem[Chernozhukov et~al.(2010)Chernozhukov, Fern{\'a}ndez-Val, and Galichon]{chernozhukov2010quantile}
Victor Chernozhukov, Iv{\'a}n Fern{\'a}ndez-Val, and Alfred Galichon.
\newblock Quantile and probability curves without crossing.
\newblock \emph{Econometrica}, 78\penalty0 (3):\penalty0 1093--1125, 2010.

\bibitem[Antolini et~al.(2005)Antolini, Boracchi, and Biganzoli]{antolini2005time}
Laura Antolini, Patrizia Boracchi, and Elia Biganzoli.
\newblock A time-dependent discrimination index for survival data.
\newblock \emph{Statistics in medicine}, 24\penalty0 (24):\penalty0 3927--3944, 2005.

\bibitem[Romano et~al.(2020)Romano, Sesia, and Candes]{romano2020classification}
Yaniv Romano, Matteo Sesia, and Emmanuel Candes.
\newblock Classification with valid and adaptive coverage.
\newblock \emph{Advances in Neural Information Processing Systems}, 33:\penalty0 3581--3591, 2020.

\bibitem[Graf et~al.(1999)Graf, Schmoor, Sauerbrei, and Schumacher]{graf1999assessment}
Erika Graf, Claudia Schmoor, Willi Sauerbrei, and Martin Schumacher.
\newblock Assessment and comparison of prognostic classification schemes for survival data.
\newblock \emph{Statistics in medicine}, 18\penalty0 (17-18):\penalty0 2529--2545, 1999.

\bibitem[Qi et~al.(2023{\natexlab{a}})Qi, Kumar, Farrokh, Sun, Kuan, Ranganath, Henao, and Greiner]{qi2023an}
Shi-Ang Qi, Neeraj Kumar, Mahtab Farrokh, Weijie Sun, Li-Hao Kuan, Rajesh Ranganath, Ricardo Henao, and Russell Greiner.
\newblock An effective meaningful way to evaluate survival models.
\newblock In \emph{Proceedings of the 40th International Conference on Machine Learning}, volume 202, pages 28244--28276. PMLR, 23--29 Jul 2023{\natexlab{a}}.

\bibitem[Stute(1993)]{stute1993consistent}
Winfried Stute.
\newblock Consistent estimation under random censorship when covariables are present.
\newblock \emph{Journal of Multivariate Analysis}, 45\penalty0 (1):\penalty0 89--103, 1993.

\bibitem[Hothorn et~al.(2006)Hothorn, B{\"u}hlmann, Dudoit, Molinaro, and Van Der~Laan]{hothorn2006survival}
Torsten Hothorn, Peter B{\"u}hlmann, Sandrine Dudoit, Annette Molinaro, and Mark~J Van Der~Laan.
\newblock Survival ensembles.
\newblock \emph{Biostatistics}, 7\penalty0 (3):\penalty0 355--373, 2006.

\bibitem[Katzman et~al.(2018)Katzman, Shaham, Cloninger, Bates, Jiang, and Kluger]{katzman2018deepsurv}
Jared~L Katzman, Uri Shaham, Alexander Cloninger, Jonathan Bates, Tingting Jiang, and Yuval Kluger.
\newblock Deepsurv: personalized treatment recommender system using a cox proportional hazards deep neural network.
\newblock \emph{BMC medical research methodology}, 18\penalty0 (1):\penalty0 1--12, 2018.

\bibitem[Fotso(2018)]{fotso2018deep}
Stephane Fotso.
\newblock Deep neural networks for survival analysis based on a multi-task framework.
\newblock \emph{arXiv preprint arXiv:1801.05512}, 2018.

\bibitem[Pearce et~al.(2022)Pearce, Jeong, Zhu, et~al.]{pearce2022censored}
Tim Pearce, Jong-Hyeon Jeong, Jun Zhu, et~al.
\newblock Censored quantile regression neural networks for distribution-free survival analysis.
\newblock In \emph{Advances in Neural Information Processing Systems}, 2022.

\bibitem[Fritsch and Butland(1984)]{fritsch1984method}
Frederick~N Fritsch and Judy Butland.
\newblock A method for constructing local monotone piecewise cubic interpolants.
\newblock \emph{SIAM journal on scientific and statistical computing}, 5\penalty0 (2):\penalty0 300--304, 1984.

\bibitem[Angelopoulos et~al.(2023)Angelopoulos, Bates, et~al.]{angelopoulos2023conformal}
Anastasios~N Angelopoulos, Stephen Bates, et~al.
\newblock Conformal prediction: A gentle introduction.
\newblock \emph{Foundations and Trends{\textregistered} in Machine Learning}, 16\penalty0 (4):\penalty0 494--591, 2023.

\bibitem[Qi et~al.(2022)Qi, Kumar, Xu, Patel, Damaraju, Shen-Tu, and Greiner]{qi2022personalized}
Shi-ang Qi, Neeraj Kumar, Jian-Yi Xu, Jaykumar Patel, Sambasivarao Damaraju, Grace Shen-Tu, and Russell Greiner.
\newblock Personalized breast cancer onset prediction from lifestyle and health history information.
\newblock \emph{Plos one}, 17\penalty0 (12):\penalty0 e0279174, 2022.

\bibitem[Qi et~al.(2023{\natexlab{b}})Qi, Sun, and Greiner]{qi2023survivaleval}
Shi-ang Qi, Weijie Sun, and Russell Greiner.
\newblock {SurvivalEVAL}: A comprehensive open-source python package for evaluating individual survival distributions.
\newblock In \emph{Proceedings of the AAAI Symposium Series}, volume~2, pages 453--457, 2023{\natexlab{b}}.

\bibitem[Chicco and Jurman(2020)]{chicco2020machine}
Davide Chicco and Giuseppe Jurman.
\newblock Machine learning can predict survival of patients with heart failure from serum creatinine and ejection fraction alone.
\newblock \emph{BMC medical informatics and decision making}, 20:\penalty0 1--16, 2020.

\bibitem[mis(2020)]{misc_heart_failure_clinical_records_519}
{Heart Failure Clinical Records}.
\newblock UCI Machine Learning Repository, 2020.
\newblock {DOI}: https://doi.org/10.24432/C5Z89R.

\bibitem[Therneau and Grambsch(2000)]{therneau2000modeling}
Terry Therneau and Patricia Grambsch.
\newblock \emph{Modeling Survival Data: Extending The Cox Model}, volume~48.
\newblock Springer, 01 2000.
\newblock ISBN 978-1-4419-3161-0.
\newblock \doi{10.1007/978-1-4757-3294-8}.

\bibitem[Therneau(2024{\natexlab{a}})]{r-survival-package}
Terry~M Therneau.
\newblock \emph{A Package for Survival Analysis in R}, 2024{\natexlab{a}}.
\newblock URL \url{https://CRAN.R-project.org/package=survival}.
\newblock R package version 3.6-4.

\bibitem[Hosmer et~al.(2008)Hosmer, Lemeshow, and May]{hosmer2008applied}
David~W Hosmer, Stanley Lemeshow, and Susanne May.
\newblock \emph{Applied survival analysis}.
\newblock John Wiley \& Sons, Inc., 2008.

\bibitem[Weinstein et~al.(2013)Weinstein, Collisson, Mills, Shaw, Ozenberger, Ellrott, Shmulevich, Sander, and Stuart]{weinstein2013cancer}
John~N Weinstein, Eric~A Collisson, Gordon~B Mills, Kenna~R Shaw, Brad~A Ozenberger, Kyle Ellrott, Ilya Shmulevich, Chris Sander, and Joshua~M Stuart.
\newblock The cancer genome atlas pan-cancer analysis project.
\newblock \emph{Nature genetics}, 45\penalty0 (10):\penalty0 1113--1120, 2013.

\bibitem[Royston and Altman(2013)]{royston2013external}
Patrick Royston and Douglas~G Altman.
\newblock External validation of a cox prognostic model: principles and methods.
\newblock \emph{BMC medical research methodology}, 13:\penalty0 1--15, 2013.

\bibitem[Fotso et~al.(2019--)]{pysurvival_cite}
Stephane Fotso et~al.
\newblock {PySurvival}: Open source package for survival analysis modeling, 2019--.
\newblock URL \url{https://www.pysurvival.io/}.

\bibitem[Dispenzieri et~al.(2012)Dispenzieri, Katzmann, Kyle, Larson, Therneau, Colby, Clark, Mead, Kumar, Melton, and Rajkumar]{dispenzieri2012use}
Angela Dispenzieri, Jerry~A. Katzmann, Robert~A. Kyle, Dirk~R. Larson, Terry~M. Therneau, Colin~L. Colby, Raynell~J. Clark, Graham~P. Mead, Shaji Kumar, L.~Joseph Melton, and S.~Vincent Rajkumar.
\newblock Use of nonclonal serum immunoglobulin free light chains to predict overall survival in the general population.
\newblock \emph{Mayo Clinic Proceedings}, 87\penalty0 (6):\penalty0 517--523, 2012.
\newblock ISSN 0025-6196.

\bibitem[Knaus et~al.(1995)Knaus, Harrell, Lynn, Goldman, Phillips, Connors, Dawson, Fulkerson, Califf, Desbiens, et~al.]{knaus1995support}
William~A Knaus, Frank~E Harrell, Joanne Lynn, Lee Goldman, Russell~S Phillips, Alfred~F Connors, Neal~V Dawson, William~J Fulkerson, Robert~M Califf, Norman Desbiens, et~al.
\newblock The support prognostic model: Objective estimates of survival for seriously ill hospitalized adults.
\newblock \emph{Annals of internal medicine}, 122\penalty0 (3):\penalty0 191--203, 1995.

\bibitem[Johnson et~al.(2022)Johnson, Bulgarelli, Pollard, Horng, Celi, Anthony, and Mark]{johnson2022mimic}
Alistair Johnson, Lucas Bulgarelli, Tom Pollard, Steven Horng, Leo Celi, Anthony, and Roger Mark.
\newblock {MIMIC-IV} (version 2.0).
\newblock \emph{PhysioNet (2022). https://doi.org/10.13026/7vcr-e114.}, 2022.

\bibitem[P{\"o}lsterl(2020)]{sksurv}
Sebastian P{\"o}lsterl.
\newblock scikit-survival: A library for time-to-event analysis built on top of scikit-learn.
\newblock \emph{Journal of Machine Learning Research}, 21\penalty0 (212):\penalty0 1--6, 2020.
\newblock URL \url{http://jmlr.org/papers/v21/20-729.html}.

\bibitem[Therneau(2024{\natexlab{b}})]{survival-package}
Terry~M Therneau.
\newblock \emph{A Package for Survival Analysis in R}, 2024{\natexlab{b}}.
\newblock URL \url{https://CRAN.R-project.org/package=survival}.
\newblock R package version 3.6-4.

\bibitem[Johnson et~al.(2023)Johnson, Bulgarelli, Shen, Gayles, Shammout, Horng, Pollard, Hao, Moody, Gow, et~al.]{johnson2023mimic}
Alistair~EW Johnson, Lucas Bulgarelli, Lu~Shen, Alvin Gayles, Ayad Shammout, Steven Horng, Tom~J Pollard, Sicheng Hao, Benjamin Moody, Brian Gow, et~al.
\newblock Mimic-iv, a freely accessible electronic health record dataset.
\newblock \emph{Scientific data}, 10\penalty0 (1):\penalty0 1, 2023.

\bibitem[Davidson-Pilon(2024)]{lifelines}
Cameron Davidson-Pilon.
\newblock lifelines, survival analysis in python, January 2024.
\newblock URL \url{https://doi.org/10.5281/zenodo.10456828}.

\bibitem[Cox(1975)]{cox1975partial}
David~R Cox.
\newblock Partial likelihood.
\newblock \emph{Biometrika}, 62\penalty0 (2):\penalty0 269--276, 1975.

\bibitem[Jin(2015)]{jin2015using}
Ping Jin.
\newblock Using survival prediction techniques to learn consumer-specific reservation price distributions.
\newblock Master's thesis, University of Alberta, 2015.

\bibitem[Breslow(1975)]{breslow1975analysis}
Norman~E Breslow.
\newblock Analysis of survival data under the proportional hazards model.
\newblock \emph{International Statistical Review/Revue Internationale de Statistique}, pages 45--57, 1975.

\bibitem[DeGroot and Fienberg(1983)]{degroot1983comparison}
Morris~H DeGroot and Stephen~E Fienberg.
\newblock The comparison and evaluation of forecasters.
\newblock \emph{Journal of the Royal Statistical Society: Series D (The Statistician)}, 32\penalty0 (1-2):\penalty0 12--22, 1983.

\end{thebibliography}
\bibliographystyle{unsrtnat}

\newpage 
\appendix

\section{Calibration in Survival Analysis}
\label{appendix:calibration_intro}

Distribution calibration (or simply ``calibration'') examines the calibration ability across the entire range of ISD predictions~\cite{haider2020effective}.
This appendix provides more details about this metric.

\subsection{Why the survival probability at event times should be uniform?}

The probability integral transform~\citep{angus1994probability} states: 
if the conditional cumulative distribution function (CDF) $F(t \mid \Bfx{i})$ is legitimate and continuous in $t$ for each fixed value of $\Bfx{i}$,
then $F(t \mid \Bfx{i})$ has a standard uniform distribution, $\mathcal{U}_{[0,1]}$. 
Since $S (t \mid \Bfx{i}) = 1 - F(t \mid \Bfx{i})$, then $S (t \mid \Bfx{i}) \sim \mathcal{U}_{[0,1]}$. 

\begin{figure}[ht]
    \centering
    \includegraphics[width=0.75\textwidth]{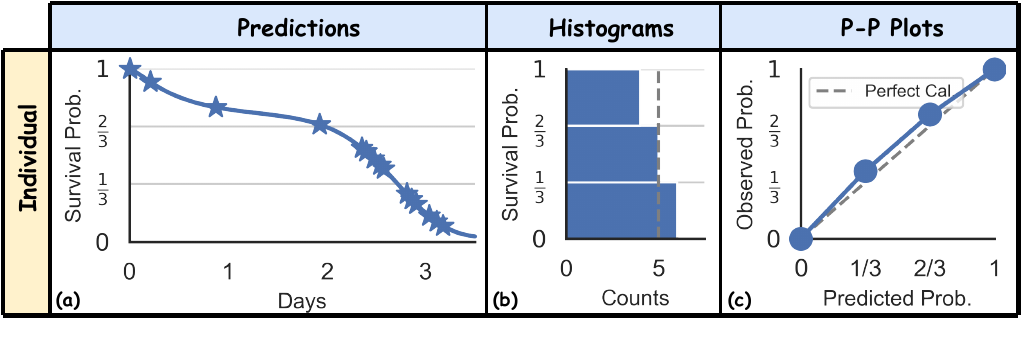}
    \caption{Individual distribution calibration, illustrated using 3 bins separated at $\frac{1}{3}$ and $\frac{2}{3}$. The curve is an oracle's true ISD $S(t\mid \Bfx{i})$. The stars represent 15 realizations of $t\mid \Bfx{i}$. The vertical coordinates of the star represent the time and the horizontal coordinates of the star are the survival probability at observed time $S(e_i^{m} \mid \Bfx{i})$.}
    \label{fig:individual_calibration}
\end{figure}

\subsection{Handling censorship}

Definition~\ref{def:margin_cal_uncensored} defines the marginal calibration for the uncensored dataset. In this section, we expand the definition to any dataset with censored instances. 
Note that~\citet{haider2020effective} proposed the following method, here we just reformulate their methodology to fit the language in this paper, for a better presentation purpose.

Given an uncensored subject, the probability of its survival probability at event time in a probability interval of $[\rho_1, \rho_2]$ is deterministic, as:
\begin{equation*}
    \prob \left( \, \hat{S} (e_i \mid \Bfx{i} ) \in [\rho_1, \rho_2] \,\middle\vert\,  \delta_i  = 1 \, \right) \  = \ \mathbbm{1} \left[\, \hat{S} (e_i \mid \Bfx{i}) \in [\rho_1, \rho_2], \ \delta_i = 1\,\right].
\end{equation*}

For censored subjects, because we do not know the true event time, so there is no way we can know whether the predicted probability is within the interval or not.  
We can ``blur'' the subject uniformly to the probability intervals after the survival probability at censored time $\hat{S} (c_i \mid \Bfx{i})$~\cite{haider2020effective}. 

\begin{align*}
    \prob & \left( \hat{S} (e_i \mid \Bfx{i}) \in [\rho_1, \rho_2] \,\middle\vert\, \delta_i 
    = 0 \right) 
    = \frac{\prob \left( \hat{S} (e_i \mid \Bfx{i}) \in [\rho_1, \rho_2],  \delta_i 
    = 0 \right)}{\prob (\delta_i = 0)} \\ 
    &= \frac{\prob \left( \hat{S} (e_i \mid \Bfx{i}) \in [\rho_1, \rho_2],  \hat{S}(e_i \mid \Bfx{i}) <  \hat{S}(c_i \mid \Bfx{i}) \right)}{\prob (\hat{S}(e_i \mid \Bfx{i}) <  \hat{S}(c_i \mid \Bfx{i}))} \\
    &= \frac{\prob \left( \hat{S} (e_i \mid \Bfx{i}) \in [\rho_1, \rho_2],  \hat{S}(e_i \mid \Bfx{i}) <  \hat{S}(c_i \mid \Bfx{i}) , \hat{S}(c_i \mid \Bfx{i}) \geq \rho_2 \right)}{\prob (\hat{S}(e_i \mid \Bfx{i}) <  \hat{S}(c_i \mid \Bfx{i}))} \\
    &\quad + \frac{\prob \left( \hat{S} (e_i \mid \Bfx{i}) \in [\rho_1, \rho_2],  \hat{S}(e_i \mid \Bfx{i}) <  \hat{S}(c_i \mid \Bfx{i}) , \hat{S}(c_i \mid \Bfx{i}) \in [\rho_1, \rho_2] \right)}{\prob (\hat{S}(e_i \mid \Bfx{i}) <  \hat{S}(c_i \mid \Bfx{i}))} \\
    &\quad + \frac{\prob \left( \hat{S} (e_i \mid \Bfx{i}) \in [\rho_1, \rho_2],  \hat{S}(e_i \mid \Bfx{i}) <  \hat{S}(c_i \mid \Bfx{i}) , \hat{S}(c_i \mid \Bfx{i}) \leq \rho_1 \right)}{\prob (\hat{S}(e_i \mid \Bfx{i}) <  \hat{S}(c_i \mid \Bfx{i}))}
\end{align*}
\begin{align*}
    &= \frac{\prob \left( \rho_1 \leq \hat{S} (e_i \mid \Bfx{i}) \leq \rho_2 \right) \prob \left(\hat{S}(c_i \mid \Bfx{i} ) \geq \rho_2\right)    }{\SurvPred{t_i}{\Bfx{i}}} \\
    &\quad + \frac{\prob \left( \rho_1 \leq \hat{S} (e_i \mid \Bfx{i}) \leq \hat{S}(c_i \mid \Bfx{i}) \right) \prob \left(\hat{S}(c_i \mid \Bfx{i})  \in [\rho_1, \rho_2]\right) }{\SurvPred{t_i}{\Bfx{i}}} +  \frac{\prob (\varnothing)}{\SurvPred{t_i}{\Bfx{i}}} \\
    &= \frac{(\rho_2 - \rho_1)\mathbbm{1}\left[\SurvPred{t_i}{\Bfx{i}} \geq \rho_2\right] + \left(\SurvPred{t_i}{\Bfx{i}} - \rho_1\right)\mathbbm{1}\left[\SurvPred{t_i}{\Bfx{i}} \in [\rho_1, \rho_2]\right]}{\SurvPred{t_i}{\Bfx{i}}},
\end{align*}
where the decomposition of probability in the second to last equality is because of the conditional independent censoring assumption. 
Therefore, for the entire dataset, considering both uncensored and censored subjects, the probability
\begin{align}
\label{eq:cal_formal}
    \prob \left( \hat{S} (e_i \mid \Bfx{i}) \in [\rho_1, \rho_2], i \in \iset \right)  
    &\ =\ \E_{i \in \iset} \left[\delta_i \cdot \prob \left( \, \hat{S} (e_i \mid \Bfx{i} ) \in [\rho_1, \rho_2]  \,\middle\vert\,  \delta_i \, \right)\right]
\end{align}

Therefore, given the above derivation, we can provide a formal definition for marginal distributional calibration for any survival dataset with censorship.
\begin{definition}[Marginal calibration]
For a survival dataset, a model has perfect marginal calibration iff $\forall \ [\rho_1, \rho_2] \subset [0, 1]$,
\begin{equation*}
    \prob \left( \hat{S} (e_i \mid \Bfx{i} ) \in [\rho_1, \rho_2]\, , \, i \in \iset  \right) \ =\  \rho_2 - \rho_1.
\end{equation*}
where the probability $\prob$ is calculated using~\eqref{eq:cal_formal}.
\end{definition}






\section{Algorithm}
\label{appendix:alg}Here we present more details for the algorithm: Conformalized survival distribution using Individual survival Probability at Observed Time (\texttt{CiPOT}).
The pseudo-code is presented in Algorithm~\ref{alg:csd_ipot}.

\begin{algorithm}[h]
   \caption{\CSDiPOT}
   \label{alg:csd_ipot}
\begin{algorithmic}[1]
    \REQUIRE Dataset $\Data$, testing data with feature $\Bfx{n+1} \in \mathbb{R}^d$, survival model $\Model$, predefined percentile levels $\mathcal{P} = \{\rho_1, \rho_2, \ldots\}$, repetition parameter $R$
    \ENSURE Calibrated ISD curve for $\Bfx{n+1}$
   \STATE Randomly partition $\Data$ into a training set $\Data^{\text{train}}$ and a conformal set $\Data^{\text{con}}$
   \STATE Train a survival model $\Model$ using $\Data^{\text{train}}$
   \STATE Make ISD predictions $\{\hat{S}_{\Model}(t \mid \Bfx{i})\}_{i \in \iset^{\text{con}}}$
   \STATE \textit{(Optional)} apply the interpolation and extrapolation to make the ISDs continuous. 
   \STATE Initialize conformity score set: $\Gamma_{\mathcal{M}} = \emptyset$,
   and pseudo-uniform array: $\boldsymbol{u}_{R} = \left[ \, {r}/{R} \, \right]_{r=0}^R$
   \FOR{$i \in \iset^{\text{con}}$}
   \STATE $\gamma_{i, \Model} = \hat{S}_{\Model}(t_i \mid \Bfx{i})$ \hfill $\triangleright$ calculate the iPOT value
   \IF{$\delta_i == 0$}
   \STATE $\Gamma_{\mathcal{M}} = \Gamma_{\mathcal{M}} + \{ u \cdot \gamma_{i, \Model} \mid u \in \boldsymbol{u}_R\}$ \hfill $\triangleright$ sample R times from $\mathcal{U}_{[0, \hat{S}_\Model(t_i \mid \Bfx{i})]}$
   \ELSE
   \STATE $\Gamma_{\mathcal{M}} = \Gamma_{\mathcal{M}} + \{ \underbrace{\gamma_{i, \Model}, \cdots, \gamma_{i, \Model}}_{R \text{ times}}\}$ \hfill $\triangleright$ repeat the iPOT value, $R$ times
   \ENDIF
   \ENDFOR
   \STATE $\forall \ \rho \in \mathcal{P}$, $\Tilde{S}_\Model^{-1}(\rho \mid \Bfx{n+1}) = \hat{S}_\mathcal{M}^{-1} \left(\, \Pct(\rho; \, \Gamma_\Model )\mid \Bfx{n+1} \, \right)$
   \STATE $\Tilde{S}_\Model(t \mid \Bfx{n+1}) \, = \, \inf \{\rho: \Tilde{S}_\Model^{-1}(\rho \mid \Bfx{n+1}) \leq t \}$ \hfill $\triangleright$ transform the inverse ISD into a ISD curve
\end{algorithmic}
\end{algorithm}

Our method is motivated by the split conformal prediction~\cite{papadopoulos2002inductive}. 
The algorithm starts by partitioning the dataset in a training and a conformal set (line 1 in Algorithm~\ref{alg:csd_ipot}).
Previous methods~\cite{romano2019conformalized,candes2023conformalized} recommend mutually exclusive partitioning, \ie the training and conformal set must satisfy $\Data^{\text{train}} \cup \Data^{\text{con}} = \Data$ and $\Data^{\text{train}} \cap \Data^{\text{con}} = \emptyset$. 
However, this can cause one problem: reducing the training set size, resulting in underfitting for the models (especially for deep-learning models) and sacrificing discrimination performance. Instead, we consider the two partitioning policies proposed by~\citet{qi2024conformalized}: (1) using the validation set as the conformal set, and (2) combining the validation and training sets as the conformal set.

Another algorithm detail is the interpolation and extrapolation (line 4 in Algorithm~\ref{alg:csd_ipot}) for the ISDs. 
This optional operation is required for non-parametric survival algorithms (including semi-parametric~\cite{cox1972regression} and discretize time/quantile models~\cite{yu2011learning, lee2018deephit}).
For interpolation, we can use either linear or piecewise cubic Hermit interpolating polynomial (PCHIP)~\cite{fritsch1984method} which both can maintain the monotonic property of the survival curves.
For extrapolation, we extend ISDs using the starting point $[0, 1]$ and ending point $[t_{\text{end}}, \hat{S}_\Model (t_{\text{end}} \mid \Bfx{i})]$.

We also include a side-by-side visual comparison of our method to \texttt{CSD}~\cite{qi2024conformalized} in Figure~\ref{fig:csd_compare}.
Both \texttt{CSD} and \CSDiPOT approaches use the same general framework of conformal prediction, but they differ in how to calculate the conformity score and adjust the positions of the predictions. Note that in Figure~\ref{fig:csd_compare}, the \textbf{horizontal} positions of the circles remain unchanged for \texttt{CSD} while the \textbf{vertical} positions of the circles remain unchanged for \texttt{CiPOT}.
This unique approach for calculating the conformity score, and the downstream benefit for handling censored subjects (thanks to the conformity score design), together provide a theoretical guarantee for the calibration, as elaborated in the next section. 

\begin{figure}[t]
    \centering
    \includegraphics[width=\linewidth]{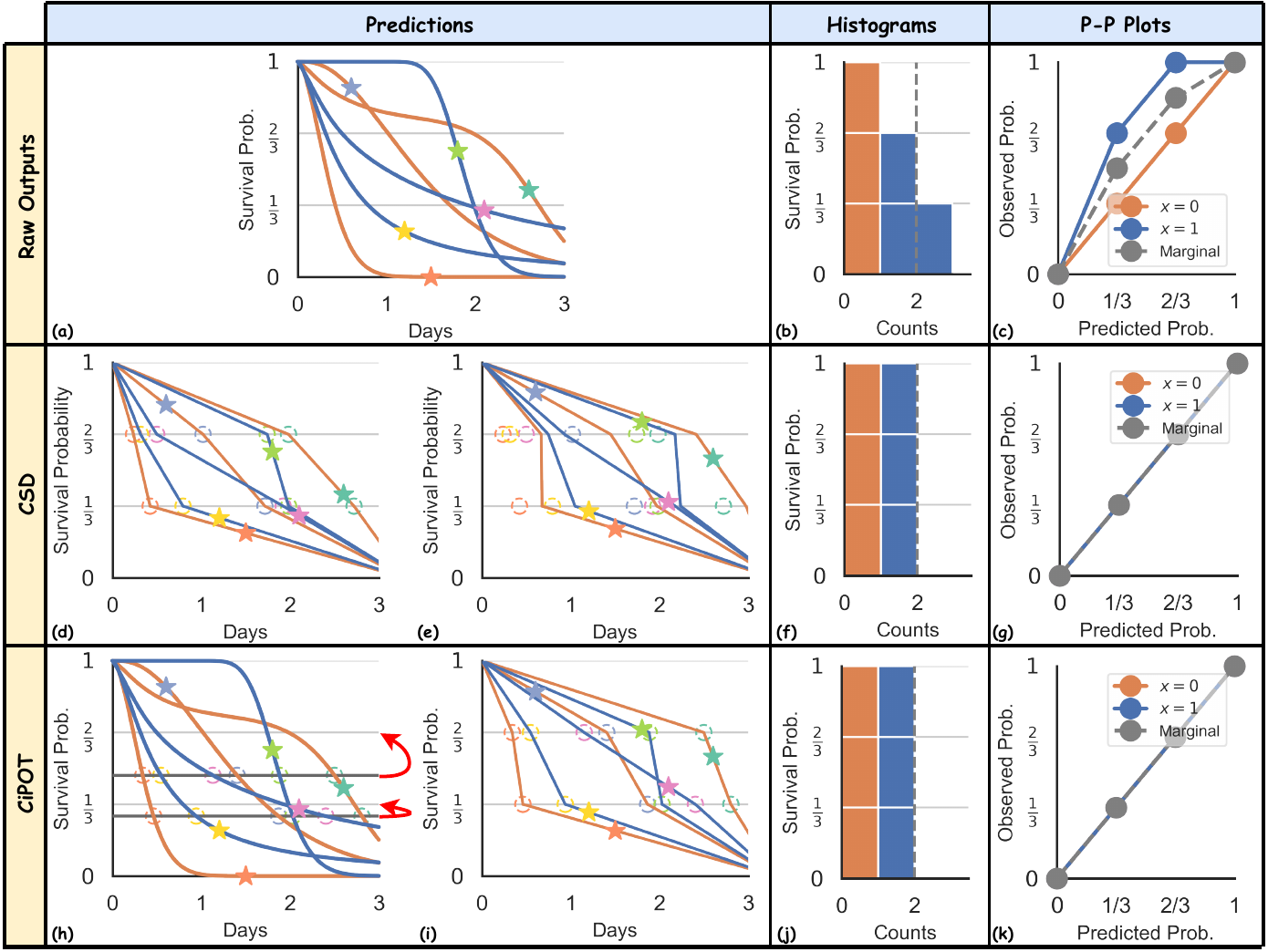}
    \caption{Comparison of \texttt{CSD} and \texttt{CiPOT}. The first row mirrors the original Figure~\ref{fig:iPOT_illust}'s first row, and the third row reflects its second row. \texttt{CSD} steps include: (d) discretizing the ISDs into predicted percentile times (circles), and calculating conformity scores using the horizontal differences between circles and stars (true outcomes); then (e) adjusting the circles horizontally via conformal regression.}
    \label{fig:csd_compare}
\end{figure}

\section{Theoretical Analysis}
\label{appendix:theory}

This appendix offers more details on the theoretical analysis in Section~\ref{sec:method_theory}.

\subsection{More on the marginal and conditional calibration}
\label{appendix:calibration}


This section presents the complete proof for Theorem~\ref{theorem:margin_cal} and Theorem~\ref{theorem:conditional_cal}. 

For the completeness, we start by restating Theorem~\ref{theorem:margin_cal}:
\begin{theorem}[Asymptotic marginal calibration]
    If the instances in $\Data$ are exchangeable, and follow the conditional independent censoring assumption, then for a new dataset $n+1$, $\forall \ \rho_1 < \rho_2 \in [0, 1]$, 
    \begin{equation}
    \label{equation:marginal_calibration_theorem}
        \rho_2 - \rho_1 \quad \leq \quad \prob\left(\Tilde{S}_\mathcal{M} (t_{n+1} \mid \Bfx{n+1}) \in [\rho_1, \rho_2]\right) \quad \leq \quad \rho_2 - \rho_1 + \frac{1}{|\Data^{\textnormal{con}}|+1} .
    \end{equation}
\end{theorem}
\begin{proof}
    This proof is inspired by the proof in Theorem D.1 and D.2 from~\citet{angelopoulos2023conformal}, the main difference is that our conformity scores are different, and we also have censoring subjects.  

    Given the target percentile levels $\rho_1$ and $\rho_2$, 
    As we remember in the \CSDiPOT procedure, we essentially vertically shift the prediction at $\Pct(\rho; \Gamma)$ to $\rho$ (see~\eqref{eq:adjustment}).
    Therefore, the condition in~\eqref{equation:marginal_calibration_theorem}
    \begin{align*}
        \Tilde{S}_\mathcal{M} (t_{n+1} \mid \Bfx{n+1})\ \in\ [\rho_1, \rho_2],
    \end{align*}
    can be transformed into this equivalent condition
    \begin{align*}
        \hat{S}_\mathcal{M} (t_{n+1} \mid \Bfx{n+1})\ \in\  [\Pct (\rho_1 ; \Gamma_\Model), \Pct (\rho_2 ; \Gamma_\Model)] .
    \end{align*}

    As we recall, the conformity score set consists of the iPOT score for each subject in the conformal set. Let us consider the easy case of an uncensored dataset, then 
    \begin{align*}
        \Gamma_\Model = \{ \hat{S}_\Model (e_i \mid \Bfx{i}) \}_{i \in \iset^{\text{con}}}
    \end{align*}
    Without loss of generality, we assume that the conformity scores in $\Gamma_\Model$ are sorted. 
    This assumption is purely technical as it aims for simpler math expression in this proof, \ie $\hat{S}_\Model (e_1 \mid \Bfx{1}) < \hat{S}_\Model (e_2 \mid \Bfx{2})< \cdots < \hat{S}_\Model (e_{|\Data|^{\text{con}}} \mid \Bfx{|\Data^{\text{con}}|})$.
    Therefore, by the exchangeability assumption (the order between subject $n+1$ and subjects $i \in \iset^{\text{con}}$ do not matter), the iPOT score of subject $n+1$ is equally likely to fall into any of the $|\Data^{\text{con}}|+1$ intervals between the $[0, 1]$, separated by the $|\Data^{\text{con}}|$ conformity scores. 
    Therefore,
    \begin{align*}
        \prob \left(\hat{S}_\mathcal{M} (t_{n+1} \mid \Bfx{n+1}) \,  \in \, [\Pct (\rho_1 ; \Gamma_\Model), \Pct (\rho_2 ; \Gamma_\Model)] \right) = \frac{\lceil (\rho_2 - \rho_1)(|\Data^{\text{con}}|+1)\rceil}{(|\Data^{\text{con}}|+1)}
    \end{align*}
    This value is always higher than $\rho_2 - \rho_1$, and if the conformity scores in $\Gamma_\Model$ do not have any tie\footnote{This is a technical assumption in conformal prediction. This assumption is easy to solve in practice because users can always add a vanishing amount of random noise to the scores to avoid ties.

}, we can also see that the above equation is less than $\rho_2 - \rho_1 + \frac{1}{|\Data^{\text{con}}|+1}$.

    Now let's consider the censored case. Based on the probability integral transform, for a censored subject $j$, the probability that its iPOT value falls into the percentile interval, before knowing its censored time, is (from here we omit the subscript $\Model$ for simple expression) \begin{align*}
        \prob \left(\SurvPred{e_j}{\Bfx{j}} \leq \rho \right) = \rho, \quad
        \text{and} \quad \prob \left(\rho_1 \leq \SurvPred{e_j}{\Bfx{j}} \leq \rho_2 \right) = \rho_2 - \rho_1
    \end{align*}
    Now given its censored time $c_j$, we can calculate the conditional probability
    \begin{align*}
        \prob &\left(\rho_1 \leq \SurvPred{e_j}{\Bfx{j}} \leq \rho_2  \,\middle\vert\, e_j>c_i\right) 
        = \prob \left(\rho_1 \leq \SurvPred{e_j}{\Bfx{j}} \leq \rho_2  \,\middle\vert\, \SurvPred{e_j}{\Bfx{j}} \leq \SurvPred{c_j}{\Bfx{j}}\right) \\
        &= \frac{\prob\left(\rho_1 \leq \SurvPred{e_j}{\Bfx{j}} \leq \rho_2  , \SurvPred{e_j}{\Bfx{j}} \leq \SurvPred{c_j}{\Bfx{j}} \right)}{\prob\left(\SurvPred{e_j}{\Bfx{j}} \leq \SurvPred{c_j}{\Bfx{j}}\right)} \\
        &= \frac{\prob\left(\rho_1 \leq \SurvPred{e_j}{\Bfx{j}} \leq \rho_2  , \SurvPred{e_j}{\Bfx{j}} \leq \SurvPred{c_j}{\Bfx{j}},  \SurvPred{c_j}{\Bfx{j}} < \rho_1 \right)}{\SurvPred{c_j}{\Bfx{j}}} \\
        & \quad + \frac{\prob\left(\rho_1 \leq \SurvPred{e_j}{\Bfx{j}} \leq \rho_2  , \SurvPred{e_j}{\Bfx{j}} \leq \SurvPred{c_j}{\Bfx{j}},  \rho_1 \leq \SurvPred{c_j}{\Bfx{j}} \leq \rho_2 \right)}{\SurvPred{c_j}{\Bfx{j}}} \\
        & \quad + \frac{\prob\left(\rho_1 \leq \SurvPred{e_j}{\Bfx{j}} \leq \rho_2  , \SurvPred{e_j}{\Bfx{j}} \leq \SurvPred{c_j}{\Bfx{j}} ,  \SurvPred{c_j}{\Bfx{j}} > \rho_2 \right)}{\SurvPred{c_j}{\Bfx{j}}} \\
        &= \frac{0}{\SurvPred{c_j}{\Bfx{j}}} + \frac{\prob\left(\rho_1 \leq\SurvPred{e_j}{\Bfx{j}} \leq \SurvPred{c_j}{\Bfx{j}},  \rho_1 \leq \SurvPred{c_j}{\Bfx{j}} \leq \rho_2 \right)}{\SurvPred{c_j}{\Bfx{j}}} \\
        & \quad + \frac{\prob\left(\rho_1 \leq \SurvPred{e_j}{\Bfx{j}} \leq \rho_2  ,  \SurvPred{c_j}{\Bfx{j}} > \rho_2 \right)}{\SurvPred{c_j}{\Bfx{j}}} \\
        &= \frac{\left(\SurvPred{c_j}{\Bfx{j}} - \rho_1\right) \mathbbm{1}\left[\rho_1 \leq \SurvPred{c_j}{\Bfx{j}} \leq \rho_2 \right] + (\rho_2 - \rho_1) \mathbbm{1}\left[\SurvPred{c_j}{\Bfx{j}} > \rho_2 \right]}{\SurvPred{c_j}{\Bfx{j}}} ,
    \end{align*}
    where the probability decomposition in the last equality is because the conditional independent censoring assumption, \ie $e_j \ \bot \ c_j \mid \Bfx{j}$.
    This above derivation means the $\SurvPred{e_j}{\Bfx{j}}$ follows the uniform distribution $\mathcal{U}_{[0, \SurvPred{c_i}{\Bfx{j}}]}$.
    Therefore, if we do one sampling for each censored subject using $\mathcal{U}_{[0, \SurvPred{c_j}{\Bfx{j}}]}$, the above proof asymptotically converges to the upper and lower bounds for uncensored subjects, for any survival dataset with censorship.
\end{proof}

For the conditional calibration, we start by formally restating Theorem~\ref{theorem:conditional_cal}. 

\begin{theorem}[Asymptotic conditional calibration] 
With the conditional independent censoring assumption, if we have these additional three assumptions:
\begin{enumerate}[label=(\roman*)]
    \item the non-processed prediction $\hat{S}(t\mid \Bfx{i})$ is also a consistent survival estimator with:
    \begin{align}
    \label{eq:consistent}
        \prob \left( \E \left[  \sup_{t} \left(\hat{S} (t\mid \Bfx{i}) - S (t\mid \Bfx{i}) \right)^2 \,\middle\vert\, \hat{S}  \right]   \geq \eta_n \right) \leq \sigma_n , \ \text{s.t.} \quad \eta_n = o(1), \sigma_n =o(1),
    \end{align}
    \item the inverse ISD estimation $\hat{S}^{-1}(t\mid \Bfx{i})$ is differentiable,
    \item there exist some $M$ such that $\inf_\rho \frac{d \ \hat{S}^{-1}(\rho \mid \Bfx{n+1})}{d\ \rho} \geq M^{-1}$,
\end{enumerate}
then the \CSDiPOT process can asymptotically achieve the conditional distribution calibration:
\begin{align*}
    \Tilde{S} (t \mid \Bfx{n+1})  = {S} (t \mid \Bfx{n+1}) + o_p(1).
\end{align*}
\end{theorem}

\begin{proof}
    In order to prove this theorem, it is enough to show that (i)~$\Tilde{S}(t \mid \Bfx{n+1}) = \hat{S}(t \mid \Bfx{n+1}) + o_p(1)$, and 
    (ii)~$\hat{S}(t \mid \Bfx{n+1}) = {S}(t \mid \Bfx{n+1}) + o_p(1)$.

    The second equality is obvious to see under the consistent survival estimator assumption~\eqref{eq:consistent}.

    Now let's focus on the first equality.   
    We borrow the idea from Lemma 5.2 and Lemma 5.3 in~\citet{izbicki2020flexible}. It states under the consistent survival estimator assumption, we can have
    $\hat{S}^{-1} (\Pct(\rho; \Gamma) \mid \Bfx{n+1}) = \hat{S}^{-1}(\rho \mid \Bfx{n+1})+ o_p(1)$.

    
    The only difference between the above claim and the original claim in~\citet{izbicki2020flexible}, is that they use the CDF while we use the survival function, \ie the complement of the CDF.
    
    According to~\eqref{eq:adjustment}, the above claim can be translate to 
    $\Tilde{S}^{-1} (\rho \mid \Bfx{n+1}) = \hat{S}^{-1}(\rho \mid \Bfx{n+1})+ o_p(1)$.

    Then if $|\Tilde{S}^{-1} (\rho \mid \Bfx{n+1}) - \hat{S}^{-1}(\rho \mid \Bfx{n+1})| = o_p(1)$, if the $\hat{S}$ is differentiable, we can have 
    \begin{align*}
        |\Tilde{S} (t \mid \Bfx{n+1}) - \hat{S}(t \mid \Bfx{n+1})| = o_p(1)\left( \inf_\rho \frac{d \ \hat{S}^{-1}(\rho \mid \Bfx{i})}{d\ \rho}\right)^{-1}.
    \end{align*}
    Therefore, if there exist some $M$ such that $\inf_\rho \frac{d \ \hat{S}^{-1}(\rho \mid \Bfx{n+1})}{d\ \rho} \geq M^{-1}$. We can get the first equality proved. 


    For the censored subjects, we can use the same reasoning and steps in Theorem~\ref{theorem:margin_cal} under the conditional independent censoring assumption. This will finish the proof.
\end{proof}

We are aware that linear interpolation and extrapolation will make the survival curves nondifferentiable.
Therefore, we recommend using PCHIP interpolation~\cite{fritsch1984method} for \texttt{CiPOT}.
And extrapolation normally does not need to apply for the survival prediction algorithm because the iPOT value $\hat{S}_\Model (t_i \mid \Bfx{i})$ can be obtained within the curve range for those methods.
However, for quantile-based algorithms, sometimes we need extrapolation. And we recognize this as a future direction to improve.

\subsection{More on the monotonicity}
\label{appendix:monotonic}
\texttt{CSD} is a conformalized quantile regression based~\cite{romano2019conformalized} method. 
It first discretized the curves into a quantile curve, and adjusted the quantile curve at every discretized level~\cite{qi2024conformalized}.
Both the estimated quantile curves ($\hat{q} (\rho)$) and the adjustment terms ($\text{adj} (\rho)$) are monotonically increasing with respect to the quantile levels.
However, the adjusted quantile curve -- calculated as the original quantile curve minus the adjustment -- is no longer monotonic. 
For example, if $\hat{q} (50\%) = 5$ and $\hat{q} (60\%) = 6$, with corresponding adjustments of $\text{adj} (50\%) = 2$ and $\text{adj} (60\%) = 4$, the post-CSD quantile curve will be $5-2$ at 50\% and $6-4$ at 60\%, demonstrating non-monotonicity.
For a detailed description of their algorithm, readers are referred to~\citet{qi2024conformalized}.

However, \CSDiPOT has this nice property.
Here we restate the Theorem~\ref{theorem:auroc}
\begin{theorem}
    \CSDiPOT process preserves the monotonic decreasing property of the ISD, s.t., 
    \begin{equation}
        \forall i \in \iset, \quad \forall \ a \leq b \in \mathbb{R}_{+}:  \quad \Tilde{S}_{\Model} (a \mid \Bfx{i}) \geq \Tilde{S}_{\Model} (b \mid \Bfx{i}).
    \end{equation} 
\end{theorem}

\begin{proof}
    The proof of this theorem is straightforward. The essence of the proof lies in the monotonic nature of all operations within \texttt{CiPOT}.

    First of all, the percentile operation $\Pct (\rho ; \Gamma_\Model)$ is a monotonic function, \ie for all $\rho_1 <\rho_2$, $\Pct (\rho_1 ; \Gamma_\Model) < \Pct (\rho_2 ; \Gamma_\Model)$.

    Second, because the non-post-processed ISD curves $\hat{S}(t\mid \Bfx{i})$ are monotonic. Therefore, the inverse survival function $\hat{S}^{-1}(\rho \mid \Bfx{i})$ is also monotonic. 
    Therefore, after the adjustment step as detailed in~\eqref{eq:adjustment}, ~\eqref{eq:adjustment}, for all $\rho_1 <\rho_2$, it follows that $\Tilde{S}_{\Model}^{-1} (\rho_1 \mid \Bfx{n+1}) < \Tilde{S}_{\Model}^{-1} (\rho_2 \mid \Bfx{n+1})$.

    Lastly, by converting the inverse survival function $\Tilde{S}_{\Model}^{-1} (\rho \mid \Bfx{n+1})$ back to survival function $\Tilde{S}_{\Model} (t \mid \Bfx{n+1})$, the monotonicity is preserved.

    These steps collectively affirm the theorem's proof through the intrinsic monotonicity of the operations involved in \texttt{CiPOT}.
\end{proof}

\subsection{More on the discrimination performance}
\label{appendix:discrimination}

This section explores the discrimination performance of \CSDiPOT in the context of survival analysis. 
Discrimination performance, which is crucial for evaluating the effectiveness of survival models, is typically assessed using three key metrics:
\begin{itemize}
    \item Harrell's concordance index (C-index)
    \item Area under the receiver operating characteristic curve (AUROC)
    \item Antolini's time-dependent C-index
\end{itemize}
We will analyze \texttt{CiPOT}'s performance across these metrics and compare it to the performance of \texttt{CSD}. The comparative analysis aims to highlight any improvements and trade-offs introduced by the \CSDiPOT methodology.

\paragraph{Harrell's C-index}
C-index is calculated as the proportion of all comparable
subject pairs whose predicted and outcome orders are concordant, defined as
\begin{equation}
\label{eq:c-index}
\begin{aligned}
    \text{C-index} (\{\hat{\eta}_i\}_{i\in \iset^{\text{test}}})
    &= \frac{\sum_{i,j \in \iset^{\text{test}}}  \delta_i  \cdot \mathbbm{1}[t_i < t_j] \cdot \mathbbm{1}[
    \hat{\eta}_i > \hat{\eta}_j] }{\sum_{i,j  \in \iset^{\text{test}}} \delta_i \cdot \mathbbm{1}[t_i < t_j] } ,
\end{aligned}
\end{equation}
where $\hat{\eta}_i$ denotes the model's predicted risk score of subject $i$, which can be defined as
the negative of predicted mean/median survival time ($\E_t [\hat{S} (t \mid \Bfx{i})]$ or $\hat{S}^{-1} (0.5 \mid \Bfx{i})$). 

\texttt{CSD} has been demonstrated to preserve the discrimination performance of baseline survival models, as established in Theorem 3.1 by~\citet{qi2024conformalized}. 
In contrast, \CSDiPOT does not retain this property.
To illustrate this, we present a counterexample that explains why Harrell's C-index may not be maintained when using \texttt{CiPOT}.

\begin{figure}[ht]
    \centering
    \includegraphics[width=\textwidth]{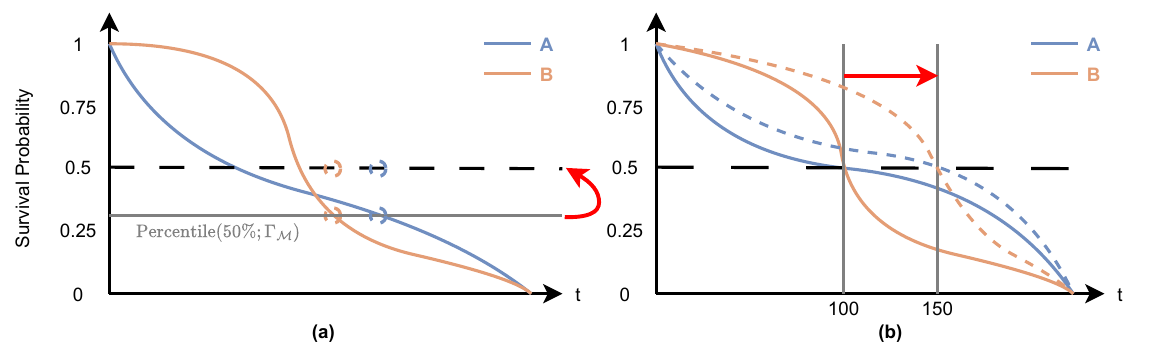}
    \caption{Counter examples of (a) Harrell's C-index performance is not preserved by \texttt{CiPOT}; and (b) AUROC performance is not preserved by \texttt{CSD}.}
    \label{fig:cindex_illus}
\end{figure}

As shown in Figure~\ref{fig:cindex_illus}(a), two ISD curves cross at a certain percentile level $25\%<\rho^*<50\%$.
Initially, the order of median survival times (where the curves cross 50\%) for these curves indicates that patient A precedes patient B.
However, after applying the adjustment as defined in~\eqref{eq:adjustment} -- which involves vertically shifting the prediction from empirical percentile level $\Pct(50\%;\Gamma_\Model)$ to desired percentile level $50\%$. 
The order of the post-processed median times for the two curves (indicated by the hollow circles) is patient B ahead of patient A.
That means this adjustment leads to a reversal in the order of risk scores, thereby compromising the C-index.

\paragraph{AUROC}
The area under the receiver operating characteristic curve (AUROC) is a widely recognized metric for evaluating discrimination in binary predictions.
Harrell's C-index can be viewed as a special case of the AUROC~\cite{harrell1996multivariable}, if we use the negative of survival probability at a specified time $t^*$ -- $\hat{S} (t^*\mid \Bfx{i})$ as the risk score.

The primary distinction lies in the definition of comparable pairs. In Harrell's C-index, comparable pairs are those for which the event order is unequivocally determined. 
Conversely, for the AUROC evaluation at time $t^*$, comparable pairs are defined as one subject experiencing an event before $t^*$ and another experiencing it after $t^*$. 
This implies that for AUROC at $t^*$, a pair of uncensored subjects both having event times before (or both after) $t^*$, is not considered comparable, whereas for the C-index, such a pair is indeed considered comparable.

The AUROC can be calculated using:
\begin{equation}
\label{eq:auroc}
\begin{aligned}
    \text{AUROC} (\hat{S}, t^*)
    &= \frac{\sum_{i,j \in \iset^{\text{test}}} \delta_i \cdot \mathbbm{1}[t_i \leq t^*] \cdot \mathbbm{1}[t_j > t^*] \cdot \mathbbm{1}[
    \hat{S} (t^* \mid \Bfx{i}) < \hat{S} (t^* \mid \Bfx{j})]}{\sum_{i,j  \in \iset^{\text{test}}}  \delta_i \cdot \mathbbm{1}[t_i \leq t^*] \cdot \mathbbm{1}[t_j > t^*]} ,
\end{aligned}
\end{equation}
From this equation, because the part of  $\delta_i \cdot \mathbbm{1}[t_i \geq t^*] \cdot \mathbbm{1}[t_j < t^*]$ is independent of the prediction (it only relates to the dataset labels). As long as the post-processing does not change the order of the survival probabilities, we can maintain the same AUROC score.

\begin{theorem}
\label{theorem:auroc}
    Applying the $\CSDiPOT$ adjustment to the ISD prediction does not affect the relative order of the survival probabilities at any single time, therefore does not affect the AUROC score of the model. Formally, $\forall \ i, j \in \iset$ and $\forall \ t^* \in \mathbb{R}_+$, given
    \begin{align*}
        \hat{S}_\Model (t^* \mid \Bfx{i}) < \hat{S}_\Model (t^* \mid \Bfx{j}),
    \end{align*}
    we must have
    \begin{align*}
        \Tilde{S}_\Model (t^* \mid \Bfx{i}) < \Tilde{S}_\Model (t^* \mid \Bfx{j}).
    \end{align*}
    Here $\Tilde{S}_\Model$ is calculated using~\ref{eq:reverse_ISD} in Section~\ref{sec:method_step2}.
\end{theorem}

\begin{proof}
    The intuition is that if we scale the ISD curves vertically. Then the vertical order of the ISD curves at every time point should not be changed.
    
    Formally, we first represent $\Tilde{S}_\Model (t^* \mid \Bfx{i})$ by $\Tilde{\rho}_i^*$, and represent $\Tilde{S}_\Model (t^* \mid \Bfx{j})$ by $\Tilde{\rho}_j^*$. 
    Then, by applying~\eqref{eq:reverse_ISD} and then~\eqref{eq:adjustment}, we can have 
    \begin{align*}
        \Tilde{S}_\Model^{-1} (\Tilde{\rho}_i^* \mid \Bfx{i}) &= \hat{S}_\Model^{-1} (\Pct (\Tilde{\rho}_i^* ; \Gamma_\Model) \mid \Bfx{i}) \\
        \Tilde{S}_\Model^{-1} (\Tilde{\rho}_j^* \mid \Bfx{j}) &= \hat{S}_\Model^{-1} (\Pct (\Tilde{\rho}_j^* ; \Gamma_\Model) \mid \Bfx{j})
    \end{align*}
    where $\Pct (\Tilde{\rho}_i^* ; \Gamma_\Model)$ is the original predicted probability at $t^*$, \ie $\hat{S}_\Model (t^* \mid \Bfx{i}) =  \Pct (\Tilde{\rho}_i^* ; \Gamma_\Model)$.
    
    Because the $\Pct$ operation and inverse functions are monotonic (Theorem~\ref{theorem:monotonic}), therefore, this theorem holds.
\end{proof}

\texttt{CSD} adjusts the survival curves horizontally (\eg along the time axis). Hence, while the horizontal order of median/mean survival times does not change – as proved in the Theorem 3.1 from~\citet{qi2024conformalized} – the vertical order, represented by survival probabilities, might not be preserved by \texttt{CSD}. 

Let's use a counter-example to illustrate our point.
Figure~\ref{fig:cindex_illus}(b) shows two ISD predictions $\SurvPred{t}{\Bfx{A}}$ and $\SurvPred{t}{\Bfx{B}}$ for subjects A and B. 
Suppose the two ISD curves both have the median survival time at $t=100$, and the two curves only cross once.
Without loss of generality, we assume $\SurvPred{t^*}{\Bfx{A}} < \SurvPred{t^*}{\Bfx{B}}$ holds for all $t^*<100$ and $\SurvPred{t^*}{\Bfx{A}} > \SurvPred{t^*}{\Bfx{B}}$ holds for all $t^*>100$. 
Now, suppose that \texttt{CSD} modified the median survival time from $t=100$ to $t=150$ for both of the predictions. 
Then the order between these two predictions at any time in the range of $t^*\in [100, 150]$ is changed from $\SurvPred{t^*}{\Bfx{A}} > \SurvPred{t^*}{\Bfx{B}}$ to $\SurvPred{t^*}{\Bfx{A}} < \SurvPred{t^*}{\Bfx{B}}$.

It is worth mentioning that in Figure~\ref{fig:iPOT_illust}, a blue curve is partially at the top in (a), intersecting an orange curve around 1.7 days, while the orange curve is consistently at the top in (e). 
This might raise concerns that the pre- and post-adjustment curves do not maintain the same probability ordering at every time point, suggesting a potential violation.
In fact, this discrepancy arises from the discretization step used in our process, which did not capture the curve crossing at 1.7 days due to the limited number of percentile levels (2 levels at $\frac{1}{3}$ and $\frac{2}{3}$) used for simplicity in this visualization.
The post-discretization positioning of the orange curve above the blue curve in Figure~\ref{fig:iPOT_illust}(e) does not imply that the post-processing step alters the relative ordering of subjects. Instead, it reflects the limitations of using only fewer percentile levels. Note that other crossings, such as those at approximately 1.5 and 2.0 days, are captured. In practice, we typically employ more percentile levels (\eg 9, 19, 39, or 49 as in Ablation Study \#2 -- see Appendix~\ref{appendix:ablations}), which allows for a more precise capture of all curve crossings, thereby preserving the relative ordering.

\paragraph{Antolini's C-index}
Time-dependent C-index, $C^{td}$, is a modified version of Harrell's C-index~\cite{antolini2005time}. 
Instead of estimating the discrimination over the point predictions, it estimates the discrimination over the entire curve.
\begin{align*}
    C^{td} (\hat{S})
    &= \frac{\sum_{i,j \in \iset^{\text{test}}}  \delta_i  \cdot \mathbbm{1}[t_i < t_j] \cdot \mathbbm{1}[
    \hat{S} (t_i \mid \Bfx{i}) < \hat{S} (t_i \mid \Bfx{j})] }{\sum_{i,j  \in \iset^{\text{test}}} \delta_i \cdot \mathbbm{1}[t_i < t_j] } .
\end{align*}
Compared with~\eqref{eq:c-index}, the only two differences are: the risk score for the earlier subject $n_i$ is represented by the iPOT value $\hat{S} (t_i \mid \Bfx{i})$, 
and the risk score for the later subject $n_j$ is represented by the predicted survival probability for $j$ at $t_i$, $\hat{S} (t_i \mid \Bfx{j})$.

$C^{td}$ can also be represented by the weighted average of AUROC over all time points (Equation 7 and proof in Appendix 1 from~\citet{antolini2005time}). Suppose given a series of time grid $\{t_0, \ldots, t_k, \ldots, t_K\}$, 
\begin{equation}
\label{eq:c_index_td}
\begin{aligned}
    C^{td} (\hat{S})
    &= \frac{\sum_{k=0}^K \text{AUROC} (\hat{S}, t_k) \cdot \omega(t_k)}{\sum_{k=0}^K \omega(t_k)},
\end{aligned}
\end{equation}
where
\begin{align}
\label{eq:c_td_weight}
    \omega(t_k) = \frac{\sum_{i,j  \in \iset^{\text{test}}}  \delta_i \cdot \mathbbm{1}[t_i \leq t^*] \cdot \mathbbm{1}[t_j > t^*]}{\sum_{i,j  \in \iset^{\text{test}}}  \mathbbm{1}[t_i < t_j]}
\end{align}
The physical meaning of $\omega(t_k)$ measures the proportion of comparable pairs at $t_k$ over all possible pairs. Therefore, we can have the following important property of our \texttt{CiPOT}.

\begin{lemma}
\label{lemma:c_index_td}
    Applying the $\CSDiPOT$ adjustment to the ISD prediction does not affect the time-dependent C-index of the model.
\end{lemma}
\begin{proof}
    Because the $\omega(t_k)$ in~\ref{eq:c_td_weight} is independent of the prediction, and also given Theorem~\ref{theorem:auroc}, $\text{AUROC}(\hat{S}, t_k) = \text{AUROC}(\Tilde{S}, t_k)$, we can easily have $C^{td}(\hat{S}) = C^{td}(\Tilde{S})$ from the formula in~\ref{eq:c_index_td}.
\end{proof}

As in the previous discussion, \texttt{CSD} does not preserve the vertical order represented by survival probabilities. Therefore, it is natural to see that  \texttt{CSD} also does not preserve the $C^{td}$ performance.

\subsection{More on the significantly miscalibrated models}
\label{appendix:miscalibrated_models}

\begin{figure}
    \centering
    \includegraphics[width=\textwidth]{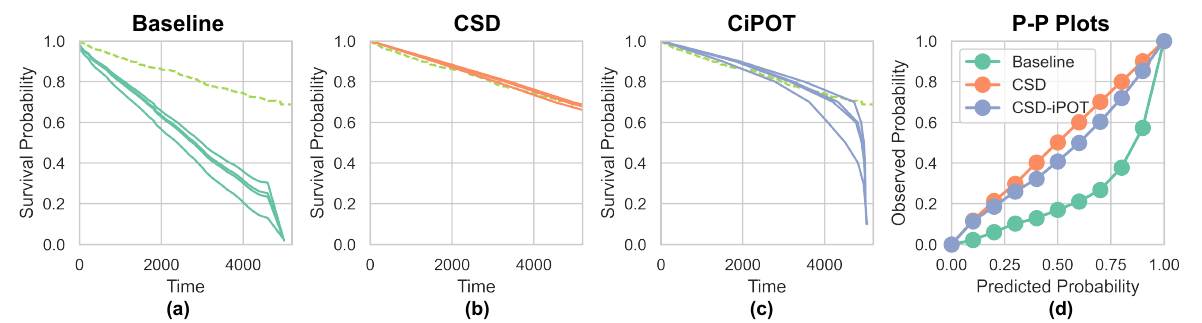}
    \caption{An real example using \textit{DeepHit} as the baseline, on the \texttt{FLCHAIN} dataset. The predicted curves in the panels are for the same 4 subjects in the test set. The dashed green line represents the KM curve on the test set. (a) Non post-processed baseline. (b) \texttt{CSD} method on \textit{DeepHit}. (c) \CSDiPOT method on \textit{DeepHit}. (d) P-P plots comparison of the three methods. }
    \label{fig:deephit_example}
\end{figure}

Compared to \texttt{CSD}, our \CSDiPOT method exhibits weaker performance on the \emph{DeepHit} baselines. This section explores the three potential reasons for this disparity.

First of all, \emph{DeepHit} tends to struggle with calibration for datasets that have high KM ending probability~\citep{qi2022personalized} and has poor calibration compared to other baselines for most datasets~\cite{kamran2021estimating, qi2024conformalized}.
This is because the \emph{DeepHit} formulation assumes that, by the end of the predefined $t_{\text{max}}$, every individual must already have had the event.
Hence, this formulation incorrectly estimates the true underlying survival distribution (often overestimates the risks) for individuals who might survive beyond $t_{\text{max}}$. 

Furthermore, apart from the standard likelihood loss, \emph{DeepHit} also contains a ranking loss term that changes the undifferentiated indicator function in the C-index calculation in~\eqref{eq:c-index} with an exponential decay function. 
This modification potentially enhances the model's discrimination power but compromises its calibration.

Lastly, Figure~\ref{fig:deephit_example}(a) shows an example prediction using \texttt{DeepHit} on the \texttt{FLCHAIN} (72.48\% censoring rate with KM curve ends at 68.16\%).
The solids curves represent the ISD prediction from \emph{DeepHit} for 4 randomly selected subjects in the test set.
And the dashed green curve represents the KM curve for the entire test set.
It is evident that \emph{DeepHit} tends to overestimate the subjects' risk scores (or underestimate the survival probabilities), see Figure~\ref{fig:deephit_example}(d).
Specifically, at the last time point ($t=5215$), KM predicts that most of the instances ($68.16\%$) should survive beyond this time point.
However, the ISD predictions from \emph{DeepHit} show everyone must die by this last point ($\hat{S}_{\text{DeepHit}} (5215 \mid \Bfx{i}) = 0$ for all $\Bfx{i}$ -- see Figure~\ref{fig:deephit_example}(a)). 
This clearly violates the unbounded range assumption proposed in Section~\ref{sec:method_step1}, which assumes $\hat{S}_\Model (t \mid \Bfx{i}) >0$ for all $t\geq0$.
This violation is the main reason why \CSDiPOT exhibits weaker performance on the \texttt{DeepHit} baseline.

\texttt{CSD} can effectively solve this overestimate issue (Figure~\ref{fig:deephit_example}(b)), as it shift the curves horizontally, \ie no upper limit for right-hand side for shifting. 
\texttt{CiPOT}, on the other hand, scale the curves vertically.
In such a case, the scaling must be performed within the percentile $[0, 1]$.
Furthermore, \CSDiPOT does not have any intervention for the starting and ending probability ($\rho=1$ and $\rho=0$) of the curves.
So no matter how the post-process changes the percentile in the middle of the curves, the starting and ending points should not be changed, just like the curves in Figure~\ref{fig:deephit_example}(c), whereas the earlier parts of the curve are similar as \texttt{CSD}'s, the last parts gradually drop to 0. 


Consequently, while \CSDiPOT significantly improves upon \emph{DeepHit}, as shown in Figure \ref{fig:deephit_example}(d), it still underperforms compared to \texttt{CSD} when dealing with models that are notably miscalibrated like \textit{DeepHit}.

\section{Evaluation metrics}
\label{appendix:metric}

We use Harrell's C-index~\cite{harrell1984regression} for evaluating discrimination performance.
The formula is presented in~\eqref{eq:c-index}.
Because we are dealing with a model that may not have proportional hazard assumption, therefore, as recommended~\cite{harrell1996multivariable}, we use the negative value of the predicted median survival time as the risk score, \ie $\hat{\eta}_i = \hat{S}^{-1} (0.5 \mid \Bfx{i})$.

The calculation of marginal calibration for a censored dataset is presented in Appendix~\ref{appendix:calibration_intro} and calculated using~\eqref{eq:cal_formal} and~\eqref{eq:cal_margin}.

Conditional calibration, $\text{Cal}_{\text{ws}}$, is estimated using~\eqref{eq:cal_ws}.
Here we present more details on the implementation. 
First of all, the evaluation involves further partitioning the testing set into exploring and exploiting datasets. 
Note that this partition does not need to be stratified (wrt to time $t_i$ and event indicator $\delta_i$). 
Furthermore, for the vectors $\boldsymbol{v}$, we sampling $M$ i.i.d. vectors on the unit sphere in $\mathbb{R}^d$. 
Generally, we want to select a high value for $M$ to enable all possible exploration.
For \textit{small} or \textit{medium} datasets, we use $M=1000$.
However, due to the computational complexity, for \textit{large} datasets, we gradually decrease the value of $M\in[100, 1000]$ to get an acceptable evaluating time.
We set $\kappa = 33\%$ in~\eqref{eq:cal_ws} for finding the $\mathbb{S}_{\boldsymbol{v}, a, b}$, that means we want to find a worst-slab that contains a least $33\%$ of the subjects in the testing set.

Integrated Brier score (IBS) measures the accuracy of the predicted probabilities over all times. IBS for survival prediction is typically defined as the integral of Brier scores (BS) over time points:
\begin{equation*}
\begin{aligned}
    &\text{IBS} (\hat{S}; t_{\text{max}})
    = \frac{1}{t_{\text{max}}}  \cdot \int_0^{t_{\text{max}}} \text{BS}(t) \ dt, \\
    &= \frac{1}{|\iset^{\text{test}}|} \sum_{i \in \iset^{\text{test}}} \frac{1}{t_{\text{max}}}  \cdot \int_0^{t_{\text{max}}} \left(
       \frac{\delta_i \cdot \mathbbm{1} [t_i \leq t] \cdot \Surv{t}{\Bfx{i}}^2 }{G(t_i)} 
     + \frac{\mathbbm{1} [t_i > t] \cdot (1 - \Surv{t}{\Bfx{i}} )^2 }{G(t)}  \right) \ dt,
\end{aligned}
\end{equation*}
where $G(t)$ is the non-censoring probability at time $t$. It is estimated with KM on the censoring distribution (flip the event indicator of data), and its reciprocal $\frac{1}{G(t)}$ is referred to as the inverse probability censoring weights (IPCW). 
$t_{\text{max}}$ is defined as the maximum event time of the combined training and validation datasets. 

Mean absolute error calculates the time-to-event precision, \ie the average error of predicted times and true times. Here we use MAE-pseudo observation (MAE-PO)~\cite{qi2023an} for handling censorship in the calculation.
\begin{gather*}
    \text{MAE}_{\text{PO}} (\{\hat{t}_i\}_{i\in \iset})
    = \frac{1}{\sum_{i \in \iset^{\text{test}}} \ \omega_i} \ \sum_{i \in \iset^{\text{test}}} \ \omega_i \times \left| (1 - \delta_i) \cdot e_{\text{PO}}(t_i, \iset^{\text{test}}) + \delta_i \cdot t_i - \hat{t}_i \right| \ , \\
    \text{where} \quad e_{\text{PO}}(t_i, \iset^{\text{test}}) = 
    \begin{cases}
    N \times \E_t \left[ S_{\text{KM}(\iset^{\text{test}})} (t) \right] - (N-1) \times \E_t \left[ S_{\text{KM}(\iset^{\text{test}} - i)} (t) \right] & \text{if} \quad \delta_i = 0,\\
    t_i & \text{otherwise}, 
    \end{cases}\\
    \text{and } \quad \omega_i = 
    \begin{cases}
    1 - \delta_i \cdot S_{\text{KM}(\iset^{\text{test}})} (t_i) & \text{if} \quad \delta_i = 0,\\
    1 & \text{otherwise}.
    \end{cases} 
\end{gather*}
Here $S_{\text{KM}(\iset^{\text{test}})} (t)$ represents the population level KM curve estimated on the entire testing set $\iset^{\text{test}}$, and $S_{\text{KM}(\iset^{\text{test}} - i)} (t)$ represent the KM curves estimated on all the test subjects but exclude subject $i$.

C-index, $\text{Cal}_{\text{margin}}$ (also called D-cal), ISB, and MAE-PO are implemented in the \texttt{SurvivalEVAL} package~\cite{qi2023survivaleval}. 
For $\text{Cal}_{\text{ws}}$, please see our Python code for implementation.

\section{Experimental Details}
\label{appendix:exp}

\subsection{Datasets}
\label{appendix:datasets}

We provide a brief overview of the datasets used in our experiments.

In this study, we evaluate the effectiveness of \CSDiPOT across 15 datasets. 
Table~\ref{tab:data_comp} summarizes the data statistics.
Compared to the datasets used in~\cite{qi2024conformalized}, we have added \texttt{HFCR}, \texttt{WHAS}, \texttt{PdM}, \texttt{Churn}, \texttt{FLCHAIN}, \texttt{Employee}, and \texttt{MIMIC-IV}.
Specifically, we use the original \texttt{GBSG} dataset, as opposed to the modified version by~\citet{katzman2018deepsurv} used in~\cite{qi2024conformalized}, which has a higher censoring rate and more features. 
For the rest of the datasets, we employ the same preprocessing methods as~\citet{qi2024conformalized} -- see Appendix E of their paper for details about these datasets. 
Below, we describe these newly added datasets:
\begin{table}[ht]
\centering
\caption{Key statistics of the datasets. We categorize datasets into \emph{small}, \emph{medium}, and \emph{large}, based on the number of instances, using thresholds of 1,000 and 10,000 instances. The bolded number represents datasets with a high percentage of censorship ($\geq 60\%$) or its KM estimation ends at a high probability ($\geq 50\%$). Numbers in parentheses indicate the number of features after one-hot encoding. }
\label{tab:data_comp}
\begin{tabular}{lrcrcc}
\toprule
Dataset  & \#Sample    & Censor Rate & Max $t$  & \#Feature & KM End Prob. \\   \midrule
\texttt{HFCR}~\cite{chicco2020machine, misc_heart_failure_clinical_records_519}    & 299         &  \textbf{67.89\%}  & 285      &  11   & \textbf{57.57\%}             \\
\texttt{PBC}~\cite{therneau2000modeling, r-survival-package}      & 418         & \textbf{61.48\%}  & 4,795     &  17  & 35.34\%                 \\
\texttt{WHAS}~\cite{hosmer2008applied}    & 500         & 57.00\%  & 2,358     &  14     & 0\%                \\
\texttt{GBM}~\cite{weinstein2013cancer, haider2020effective}      & 595         & 17.23\%  & 3,881      &  8 (10)   & 0\%                \\
\texttt{GBSG}~\cite{royston2013external, r-survival-package}     & 686       & 56.41\%  & 2,659      & 8   & 34.28\%               \\
\hdashline
\texttt{PdM}~\cite{pysurvival_cite}      & 1,000       & \textbf{60.30\%}  & 93       & 5 (8)  & 0\%                    \\
\texttt{Churn}~\cite{pysurvival_cite}    & 1,958       & 52.40\%  & 12       & 12 (19)     & 24.36\%                  \\
\texttt{NACD}~\cite{haider2020effective}     & 2,396       & 36.44\%  & 84.30      & 48    & 12.46\%                    \\
\texttt{FLCHAIN}~\cite{dispenzieri2012use, r-survival-package}  & 7,871       & \textbf{72.48\%}  & 5,215      & 8 (23)   & \textbf{68.16\%}                  \\
\texttt{SUPPORT}~\cite{knaus1995support}  & 9,105       & 31.89\%  & 2,029      & 26 (31)   & 24.09\%                 \\
\hdashline
\texttt{Employee}~\cite{pysurvival_cite} & 11,991      & \textbf{83.40\%}  & 10         & 8 (10)   & \textbf{50.82\%} \\
\texttt{MIMIC-IV}~\cite{johnson2022mimic, qi2023an}    & 38,520    & \textbf{66.65\%}  & 4404       & 93  & 0\%       \\
\texttt{SEER-brain}~\cite{qi2024conformalized} & 73,703    & 40.12\%  & 227        & 10    & 26.58\%         \\
\texttt{SEER-liver}~\cite{qi2024conformalized} & 82,841    & 37.57\%  & 227        & 14    & 18.01\%         \\
\texttt{SEER-stomach}~\cite{qi2024conformalized} & 100,360 & 43.40\%  & 227        & 14    & 28.23\%         \\
\bottomrule
\end{tabular}
\end{table}

Heart Failure Clinical Record dataset (\texttt{HFCR})~\cite{chicco2020machine} contains medical records of 299 patients with heart failure, aiming to predict mortality from left ventricular systolic dysfunction. This dataset can be downloaded from UCI Machine Learning Repository~\cite{misc_heart_failure_clinical_records_519}.

Worcester Heart Attack Study dataset (\texttt{WHAS})~\cite{hosmer2008applied} contains 500 patients with acute myocardial infarction, focusing on the time to death post-hospital admission.
The data was already post-processed and can be downloaded from the \texttt{scikit-survival} package~\cite{sksurv}. 

Predictive Maintenance (\texttt{PdM}) contains information on 1000 equipment failures. 
The goal is to predict the time to equipment failure and therefore help alert the maintenance team to prevent that failure.
It includes 5 features that describe the pressure, moisture, temperature, team information (the team who is running this equipment), and equipment manufacturer.
We apply one-hot encoding on the team information and equipment manufacturer features. 
The dataset can be downloaded from the \texttt{PySurvival} package~\cite{pysurvival_cite}.

The customer churn prediction dataset (\texttt{Churn})  focuses on predicting customer attrition.
We apply one-hot encoding on the US region feature and exclude subjects who are censored at time 0.
The dataset can be downloaded from the \texttt{PySurvival} package~\cite{pysurvival_cite}.

Serum Free Light Chain dataset (\texttt{FLCHAIN}) is a stratified random sample containing half of the subjects from a study on the relationship between serum free light chain (FLC) and mortality~\cite{dispenzieri2012use}. 
This dataset is available in R’s \texttt{survival} package~\cite{survival-package}.
Upon downloading, we apply a few preprocessing steps.
First, we remove the three subjects with events at time zero.
We impute missing values for the ``creatinine'' feature using the median of this feature. 
Additionally, we eliminate the chapter feature (a disease description for the cause of death by chapter headings of the ICD code) because this feature is only available for deceased (uncensored) subjects -- hence, knowing this feature will be equivalent to leaking the event indicator label to the model.

\texttt{Employee} dataset contains employee activity information that can used to predict when an employee will quit.
The dataset can be downloaded from the \texttt{PySurvival} package~\cite{pysurvival_cite}.
It contains duplicate entries; after dropping these duplicates, the number of subjects in the dataset is reduced from 14,999 to 11,991.
We also apply one-hot encoding to the department information.

\texttt{MIMIC-IV} database~\cite{johnson2023mimic} provides critical care data information for patients within the hospital.
We focus on a cohort of all-cause mortality data curated by~\cite{qi2023an}, featuring patients who survived at least 24 hours post-ICU admission.
The event of interest, death, is derived from hospital records (during hospital stay) or state records (after discharge).
The features are laboratory measurements within the first 24 hours after ICU admission.

German Breast Cancer Study Group (\texttt{GBSG})~\cite{royston2013external} contains 686 patients with node-positive breast cancer, complete with prognostic variables.
This dataset is available in R’s \texttt{survival} package~\cite{survival-package}.
While the original \texttt{GBSG} offers a higher rate of censoring and more features, \cite{qi2024conformalized} utilized a modified version of \texttt{GBSG} from~\cite{katzman2018deepsurv}, merged with uncensored portions of the \texttt{Rotterdam} dataset, resulting in fewer features, a lower censor rate, and a larger sample size.


\begin{figure}[ht]
    \centering
    \includegraphics[width=\textwidth]{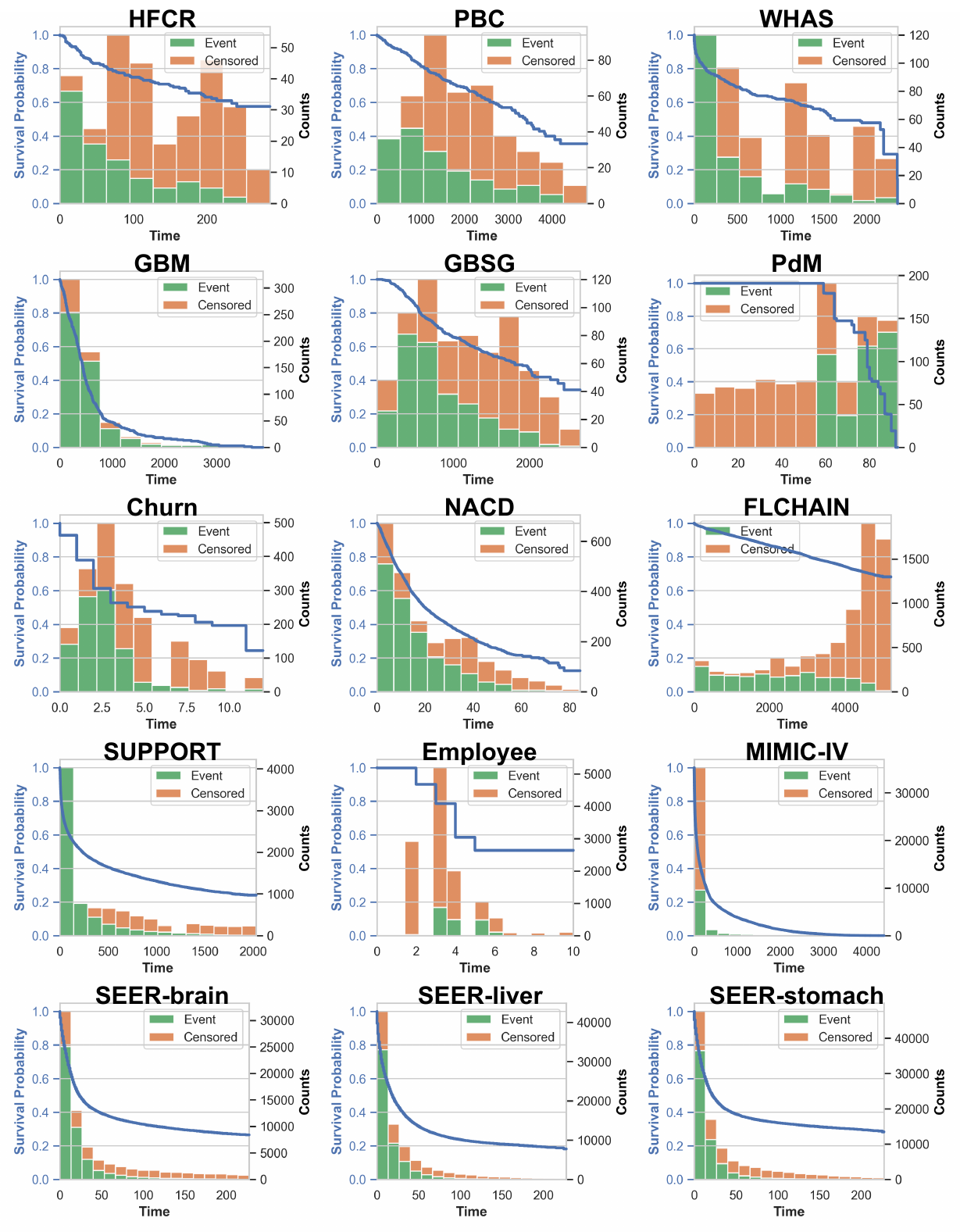}
    \caption{Kaplan Meier curves and event/censored histograms for all 15 datasets.}
    \label{fig:datasets}
\end{figure}

Figure~\ref{fig:datasets} shows the Kaplan-Meier (KM) estimation (blue curves) for all 15 datasets, alongside the event and censored histograms (represented by green and orange bars, respectively) in a stacked manner, where the number of bins is determined by the Sturges formula: $\lceil \log (|\Data|) + 1 \rceil$.

\subsection{Baselines}
\label{appendix:baselines}

In this section, we detail the implementation of the seven baseline models used in our experiments, consistent with~\citet{qi2024conformalized}.
\begin{itemize}
    \item Accelerate Failure Time (\emph{AFT})~\cite{stute1993consistent} with Weibull distribution is a linear parametric model that uses a small $l_2$ penalty on parameters during optimization. It is implemented in \texttt{lifelines} packages~\cite{lifelines}.
    \item Gradient Boosting Cox model (\emph{GB})~\cite{hothorn2006survival} is an ensemble method that employs 100 boosting stages with a partial likelihood loss~\cite{cox1975partial} for optimization and 100\% subsampling for fitting each base learner. The model is implemented in \texttt{scikit-survival} packages~\cite{sksurv}.
    \item Neural Multi-Task Logistic Regression (\emph{N-MTLR})~\cite{fotso2018deep} is a discrete-time model which is an NN-extension of the linear multi-task logistic regression model (MTLR)~\cite{yu2011learning}. The number of discrete times is determined by the square root of the number of uncensored patients. We use quantiles to divide those uncensored instances evenly into each time interval, as suggested in~\cite{jin2015using, haider2020effective}. We utilize the \emph{N-MTLR} code provided in~\citet{qi2024conformalized}.
    \item \emph{DeepSurv}~\cite{katzman2018deepsurv} is a NN-extension of the Cox proportional hazard model (CoxPH)~\cite{cox1972regression}. To make ISD prediction, we use the Breslow method~\cite{breslow1975analysis} to estimate the population-level baseline hazard function. We utilize the \emph{DeepSurv} code provided in~\citet{qi2024conformalized}.
    \item \emph{DeepHit}~\cite{lee2018deephit} is also a discrete-time model where the number and locations of discrete times are determined in the same way as the \emph{N-MTLR} model (the square root of numbers of uncensored patients, and quantiles). The model is implemented in \texttt{pycox} packages~\cite{kvamme2019time}. 
    \item \emph{CoxTime}~\cite{kvamme2019time} is a non-proportional neural network extension of the CoxPH. The model is implemented in \texttt{pycox} packages~\cite{kvamme2019time}. 
    \item Censored Quantile Regression Neural Network (\emph{CQRNN})~\cite{pearce2022censored} is a quantile regression-based method. We add the bootstrap-rearranging post-processing~\cite{chernozhukov2010quantile} to correct non-monotonic predictions. We use the \emph{CQRNN} code provided in~\citet{qi2024conformalized}.
\end{itemize}

\subsection{Hyperparameter settings for the main experiments}
\label{appendix:hyperparameter}
\paragraph{Full hyperparameter details for NN-Based survival baselines}

In the experiments, all neural network-based methods (including \emph{N-MTLR}, \emph{DeepSurv}, \emph{DeepHit}, \emph{CoxTime}, and \emph{CQRNN}) used the same architecture and optimization procedure.


\begin{itemize}
    \item Training maximum epoch: 10000
    \item Early stop patients: 50
    \item Optimizer: Adam
    \item Batch size: 256
    \item Learning rate: 1e-3
    \item Learning rate scheduler: CosineAnnealingLR
    \item Learning rate minimum: 1e-6
    \item Weight decay: 0.1
    \item NN architecture: [64, 64]
    \item Activation function: ReLU
    \item Dropout rate: 0.4
\end{itemize}

\paragraph{Full hyperparameter details for \texttt{CSD} and \CSDiPOT}

\begin{itemize}
    \item Interpolation: \{Linear, PCHIP\}
    \item Extrapolation: Linear
    \item Monotonic method:  \{Ceiling, Flooring, Booststraping\}
    \item Number percentile: \{9, 19, 39, 49\}
    \item Conformal set: \{Validation set, Training set $+$ Validation set\}
    \item Repetition parameter: \{3, 5, 10, 100, 1000\}
\end{itemize}

\begin{figure}[ht]
    \centering
    \includegraphics[width=\textwidth]{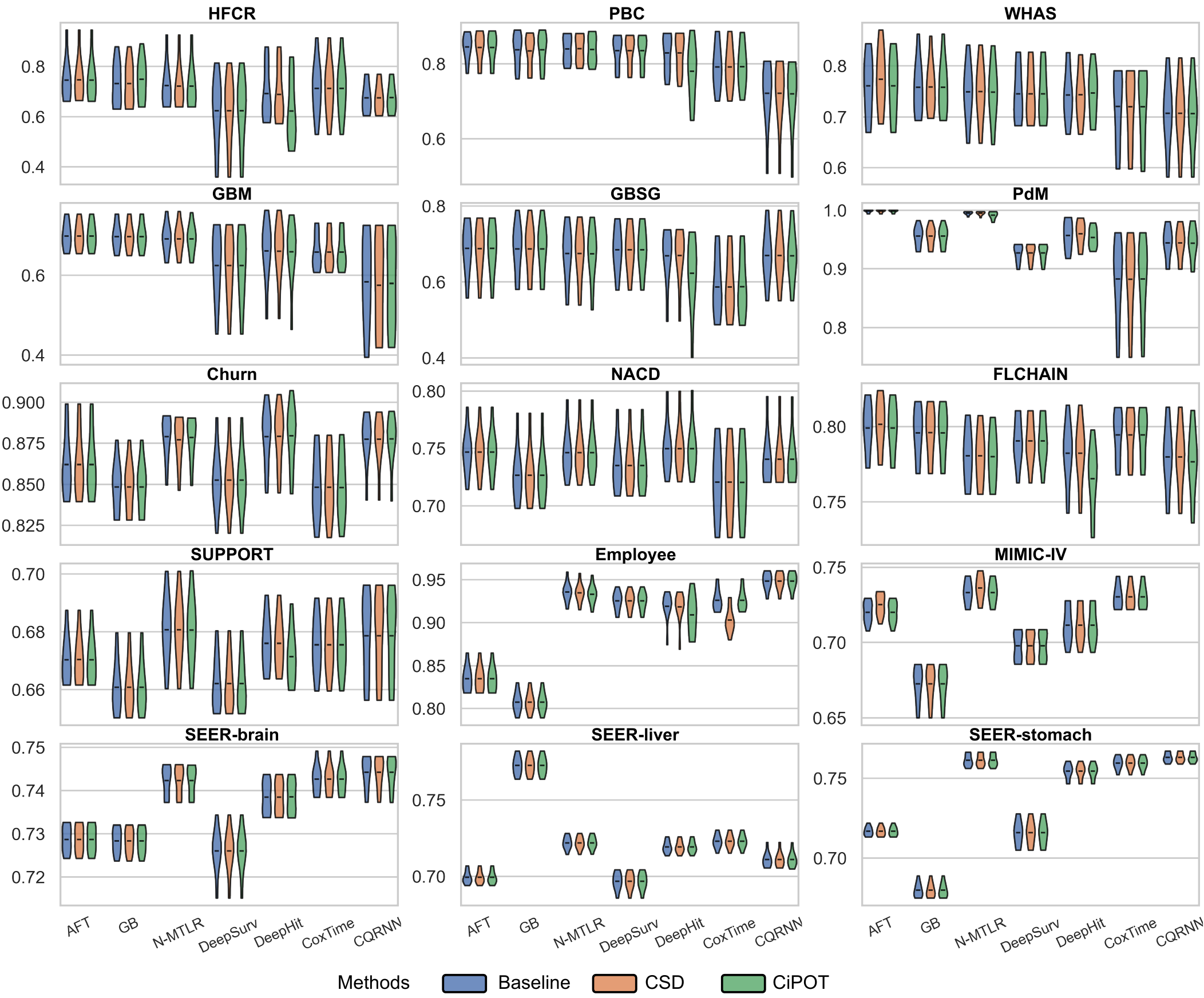}
    \caption{Violin plots of C-index performance of our method (\texttt{CiPOT}) and benchmarks. A higher value indicates better performance. The shape of each violin plot represents the probability density of the performance scores, with the black bar inside the violin indicating the mean performance.}
    \label{fig:cindex_all}
\end{figure}

\begin{figure}[ht]
    \centering
    \includegraphics[width=\textwidth]{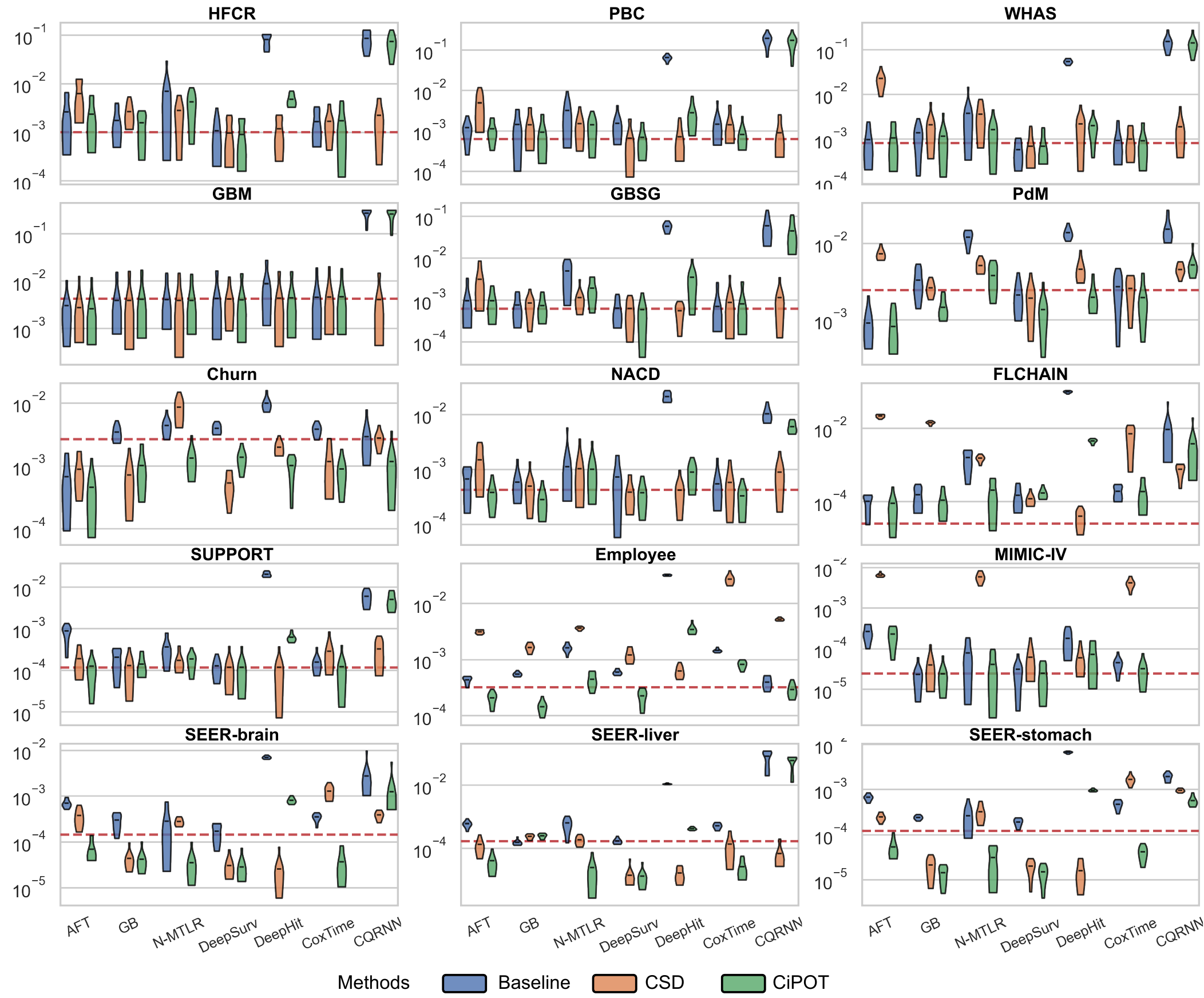}
    \caption{Violin plots of $\text{Cal}_{\text{margin}}$ performance of our method (\texttt{CiPOT}) and benchmarks. A lower value indicates superior performance. The shape of each violin plot represents the probability density of the performance scores, with the black bar inside the violin indicating the mean performance.
    The red lines represent the mean calibration performance for KM, serving as an empirical lower limit.}
    \label{fig:cal_all}
\end{figure}

\begin{figure}[ht]
    \centering
    \includegraphics[width=\textwidth]{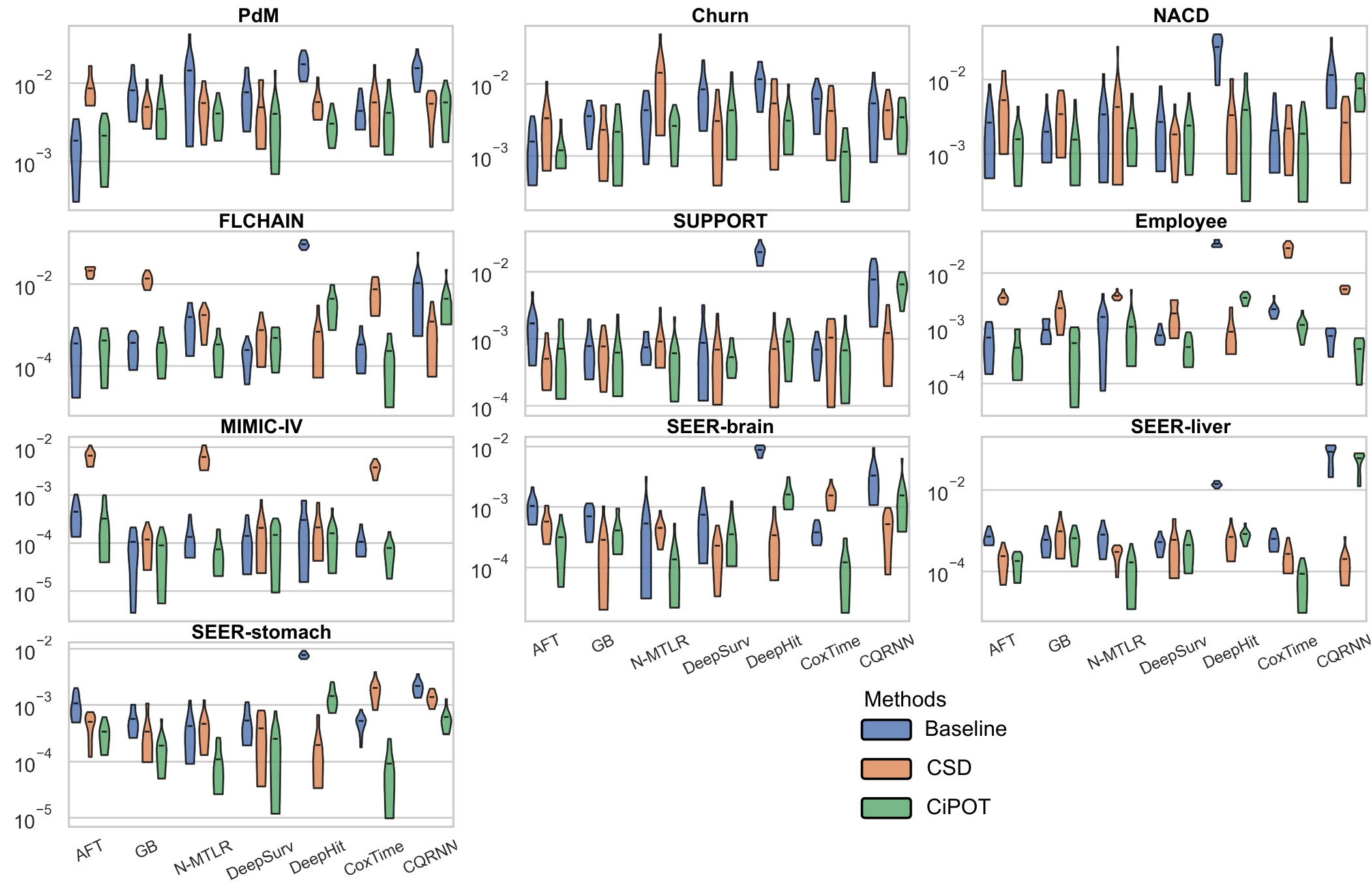}
    \caption{Violin plots of $\text{Cal}_{\text{ws}}$ performance of our method (\texttt{CiPOT}) and benchmarks. A lower value indicates superior performance. The shape of each violin plot represents the probability density of the performance scores, with the black bar inside the violin indicating the mean performance.}
    \label{fig:cal_ws_all}
\end{figure}

\begin{figure}[ht]
    \centering
    \includegraphics[width=\linewidth]{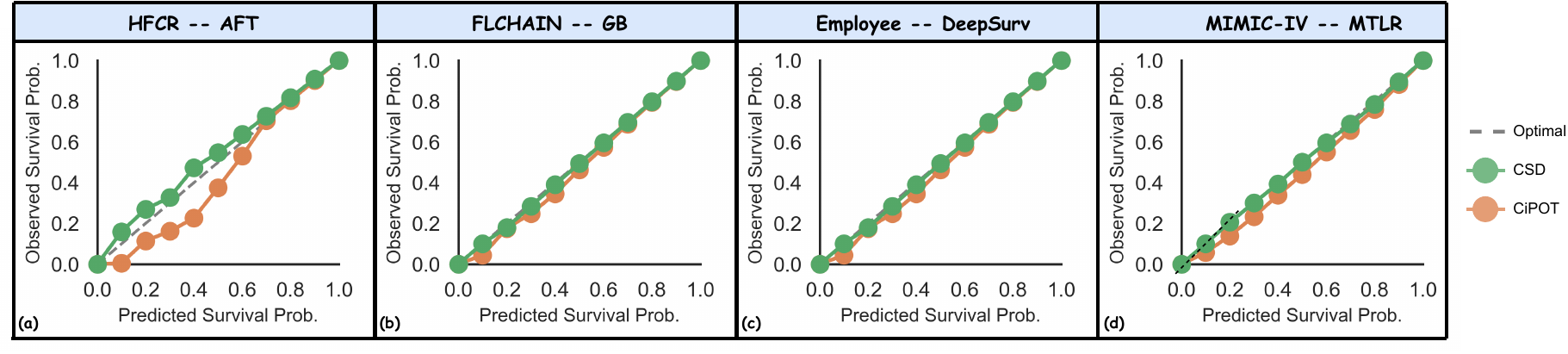}
    \caption{Case studies of the conditional calibration between \texttt{CSD} and \texttt{CiPOT}. (a) For the \textbf{elder age} subgroup on \texttt{HFCR}, with AFT as the baseline; (b) For \textbf{women} subgroup on \texttt{FLCHAIN}, with GB as the baseline; (c) For the \textbf{high salary} subgroup on \texttt{Employee}, with DeepSurv as the baseline; (d) For the \textbf{non-white-racial} subgroup on \texttt{MIMIC-IV}, with MTLR as the baseline. All four cases show that \CSDiPOT is close to the ideal, while \texttt{CSD} is not.}
    \label{fig:case_study}
\end{figure}

\begin{figure}[ht]
    \centering
    \includegraphics[width=\textwidth]{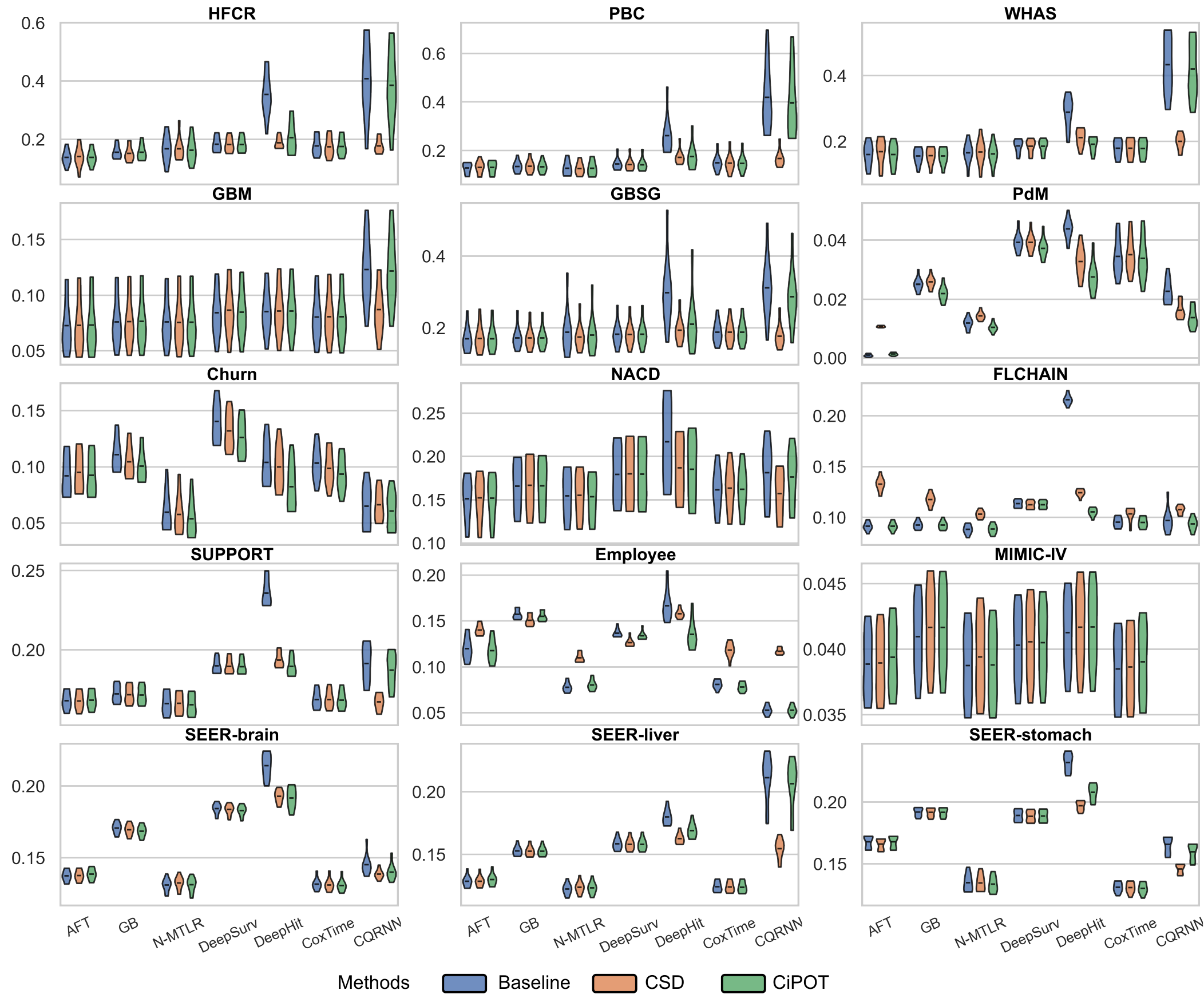}
    \caption{Violin plots of IBS performance of our method (\texttt{CiPOT}) and benchmarks. A lower value indicates superior performance. The shape of each violin plot represents the probability density of the performance scores, with the black bar inside the violin indicating the mean performance.}
    \label{fig:ibs}
\end{figure}

\begin{figure}[ht]
    \centering
    \includegraphics[width=\textwidth]{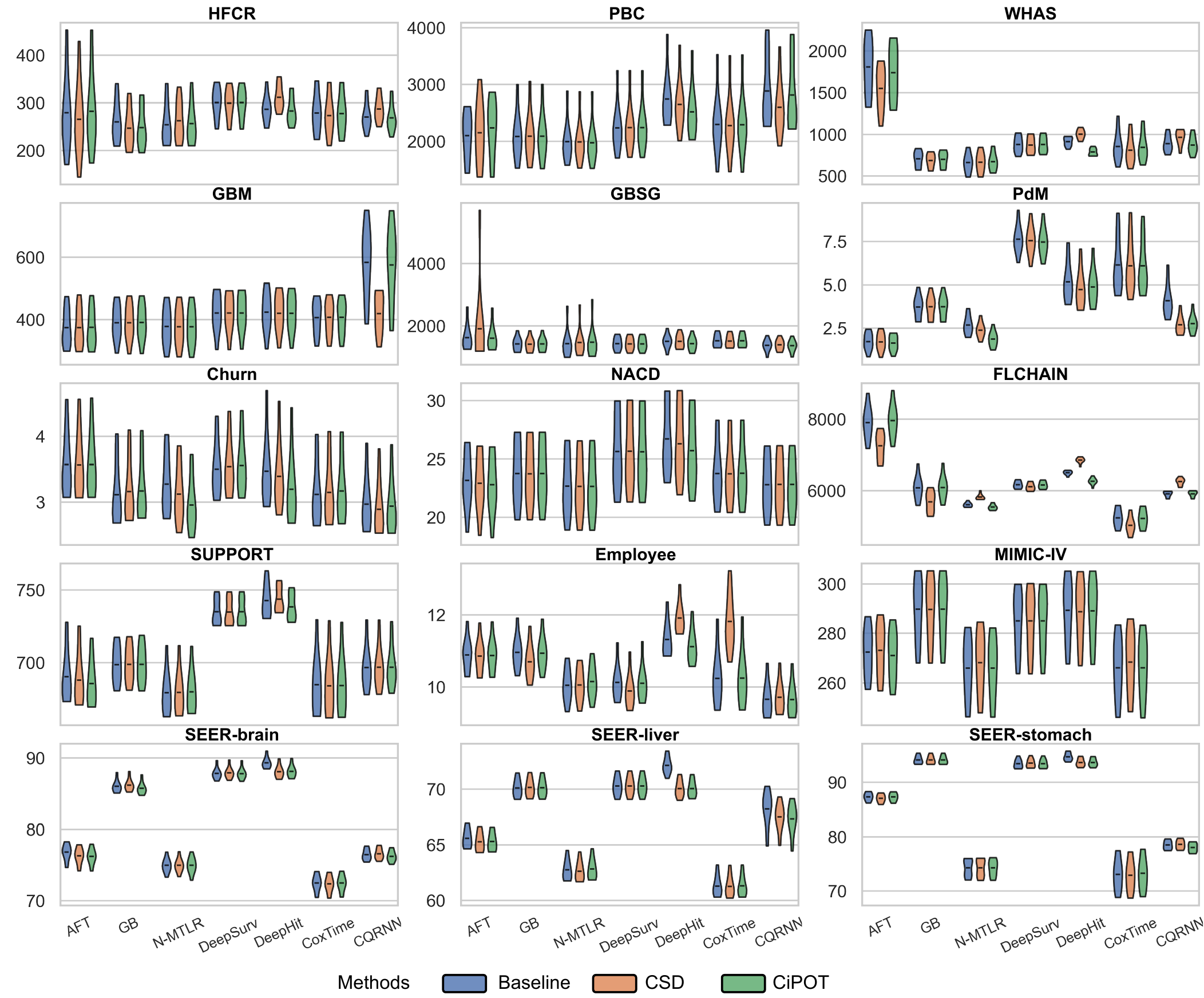}
    \caption{Violin plots of MAE-PO performance of our method (\texttt{CiPOT}) and benchmarks. A lower value indicates superior performance. The shape of each violin plot represents the probability density of the performance scores, with the black bar inside the violin indicating the mean performance.}
    \label{fig:mae}
\end{figure}

\subsection{Main results}
\label{appendix:results}

In this section, we present the comprehensive results from our primary experiment, which focuses on evaluating the performance of \CSDiPOT compared to the original non-post-processed baselines and \texttt{CSD}.

Note that for the \texttt{MIMIC-IV} datasets, \emph{CQRNN} fails to converge with any hyperparameter setting possibly due to the extremely skewed distribution. As illustrated in Figure~\ref{fig:datasets}, 80\% of event and censoring times happen within the first bin, and the distribution exhibits long tails extending to the 17th bin.

\paragraph{Discrimination}
In Figure~\ref{fig:cindex_all}, we demonstrate the discrimination performance of \CSDiPOT compared to benchmark methods, as measured by C-index.
The panels are ordered by dataset size, from smallest to largest.
In each panel, the performance of
the non-post-processed baselines is shown with blue bars, \texttt{CSD} with orange bars, and \CSDiPOT with green bars.
We can see from the figure that the three methods exhibit basically the same performance across all datasets.
Indeed, as we can see from the summary in Table~\ref{tab:summary}, \CSDiPOT ties with the baselines in 75 out of 104 times.
In the remaining 29 times they do not tie, none of them are significantly different from each other.
For the 22 times \CSDiPOT underperforms, most of them are wrt to \emph{DeepHit} or \emph{CQRNN} baselines (\eg \emph{DeepHit} for \texttt{HFCR}, \texttt{PBC}, \texttt{GBSG}, etc.).
And \CSDiPOT can even outperform 7 times (\eg \emph{GB} in \texttt{HFCR}).

\paragraph{Marginal calibration}
In Figure~\ref{fig:cal_all}, we present the marginal calibration performance of \CSDiPOT versus the benchmark models.
The arrangement and color schemes of the panels are consistent with those used previously. 
For marginal calibration evaluation, we add a ``dummy'' model -- Kaplan-Meier (KM) curve (depicted by red dashed lines in each panel) -- to serve as the empirical lower limit. 
We calculate each KM curve using the training set and apply it identically to all test samples.
It is called a ``dummy'' because it lacks the ability to discriminate between individuals. 
However, it asymptotically achieves perfect marginal calibration (see Appendix B in~\cite{qi2024conformalized}).

Our results in Figure~\ref{fig:cal_all} indicate a significant improvement in calibration performance with \CSDiPOT over both baselines and \texttt{CSD}.
Overall, our method outperforms the baselines in 95 out of 104 times (Table~\ref{tab:summary}).
In the remaining 9 times where it does not outperform the baseline, only once is the difference statistically significant, and the marginal calibration score in this single case (GB for \texttt{SEER-liver}) is still close to the empirical lower bound (red dashed line).

For datasets characterized by higher censoring rates or high KM ending probabilities (\texttt{HFCR}, \texttt{PBC}, \texttt{PdM}, \texttt{FLCHAIN}, \texttt{Employee}, \texttt{MIMIC-IV}), our method shows superior performance. 
For those datasets, \texttt{CSD} tends to produce non-calibrated predictions compared with baselines.
This outcome likely stems from the inaccuracies of the KM-sampling method under conditions of high censor rates or ending probabilities, as discussed in Section~\ref{sec:problems}. 
Nonetheless, our approach still improves the marginal calibration scores for the baselines under these circumstances.

\paragraph{Conditional calibration}
Figure~\ref{fig:cal_ws_all} showcases the conditional calibration performance of \CSDiPOT versus the benchmarks, evaluated using $\text{Cal}_{\text{ws}}$.

Compared with marginal calibration, we cannot use any method to establish the empirical upper bound.
One might think that once we establish the worst slab, we can calculate the KM curve on this slab and then use it as the empirical upper bound for the conditional calibration.
However, identifying a universal worst slab across all the models is impractical.
That means different models exhibit varying worst-slabs. 
\eg \emph{N-MTLR} can have the worst calibration for overweighted males,
while \emph{DeepSurv} might perform relatively good calibration for this group but poorly for disabled cardiovascular patients.

The results in Figure~\ref{fig:cal_all} indicate a significant improvement in conditional calibration performance using our method over both baselines and \texttt{CSD}.
Overall, our method improves the conditional calibration performance of baselines 64 out of 69 times, with significant improvements 29 times out of 64 (Table~\ref{tab:summary}).

Our method outperforms \texttt{CSD} in 51 out of 69 cases.
Most instances where our method underperforms are relative to the \emph{DeepHit} and \emph{CQRNN} baselines (the reasons are explained in Appendix~\ref{appendix:miscalibrated_models}). 
To evaluate the practical benefits of \CSDiPOT over \texttt{CSD}, we present four case studies in Figure~\ref{fig:case_study}. 
The figure showcases 4 concrete examples where \texttt{CSD} (orange) leads to significant miscalibration within certain subgroups (\ie elderly patients, women, high-salary, and non-white-racial), but \CSDiPOT (green) can effectively generate more conditional calibrated predictions which are closer to the optimal line. 
Moreover, all four examples illustrate that the miscalibration of \texttt{CSD} consistently occurs in low-probability regions, corroborating our assertion in Section~\ref{sec:method_theory} that the conditional Kaplan-Meier sampling method employed by \texttt{CSD} is problematic if the tail of the distribution is unknown.

\paragraph{IBS}

Figure~\ref{fig:ibs} illustrates the IBS performance of \CSDiPOT versus the benchmarks.
According to~\citet{degroot1983comparison}, BS can be decomposed into a calibration part and a discrimination part. 
This implies that IBS, an integrated version of BS, assesses aspects of calibration.
Our results in Figure~\ref{fig:ibs} and Table~\ref{tab:summary} show that our method improves the IBS score in most of the cases (63 wins, 18 ties, and 23 losses).

\paragraph{MAE-PO}

Our results in Figure~\ref{fig:mae} and Table~\ref{tab:summary} show that our method improves the MAE-PO in general (54 wins, 33 ties, and 17 losses).

\subsection{Computational analysis}
\label{appendix:computation}

\paragraph{Space complexity}

Although most of the computational cost arises from ISD interpolation and extrapolation, most of the memory cost of our method stems from storing the conformity scores into an array and performing the $\Pct(\cdot)$ operation. 

Let's reuse the symbol $N=|\Data^{\text{con}}|$ as the number of subjects in the conformal set, $\mathcal{P} = \{\rho_1, \rho_2, \ldots\}$ is the predefined
discretized percentiles, so that $|\mathcal{P}|$ be the number of predefined percentiles. 
And $R$ is the repetition parameter for KM-sampling.

\texttt{CSD}~\cite{qi2024conformalized} process first discretized the ISD curves into percentile times (PCTs) using $\mathcal{P}$. 
Then it calculates a conformity score for each individual at every percentile level $\rho$ (\ie each individual will contribute for $|\mathcal{P}|$ conformity scores).
Lastly, the ``KM-sampling'' process involves repeating each individual by $R$ times.
Therefore, the memory complexity of \texttt{CSD} is $O(N\cdot|\mathcal{P}|\cdot R)$.
In contrast, our \CSDiPOT method only calculates one iPOT score (as the conformity score) for each duplicated individual.  After sampling, the total memory complexity of \CSDiPOT is $O(N\cdot R)$.

Let consider an example of using \texttt{SEER-stomach} dataset in our main experiment in Appendix~\ref{appendix:results}, with a repetition parameter $R=1000$, and number of predefined percentiles $|\mathcal{P}| = 19$\footnote{As recommend by~\citet{qi2023an}, the performance reaches the optimal when $|\mathcal{P}| = 9$ or $19$. $|\mathcal{P}|=19$ means that the predefined percentile levels are $\mathcal{P} = \{5\%, 10\%, \ldots, 95\%\}$. }. 
\texttt{CSD}'s conformity score matrix requires $N \times |\mathcal{P}| \times R = (100360 * 0.9) \times 19 \times 1000 \times 8 \ \text{bytes} \approx 13.73 \ \text{Gb}$. 
Our \CSDiPOT method needs only $N \times R \approx 0.72$ Gb to store the conformity score.
If we change to an even larger dataset or increase $|\mathcal{P}|$, \texttt{CSD} may become infeasible.

Other memory costs, 
\eg storing features and ISD predictions, incur negligible memory costs.
This is because the number of feature $d$, and the length of the ISDs are much smaller than the repeat parameters $R$.

\begin{figure}
    \centering
    \includegraphics[width=\textwidth]{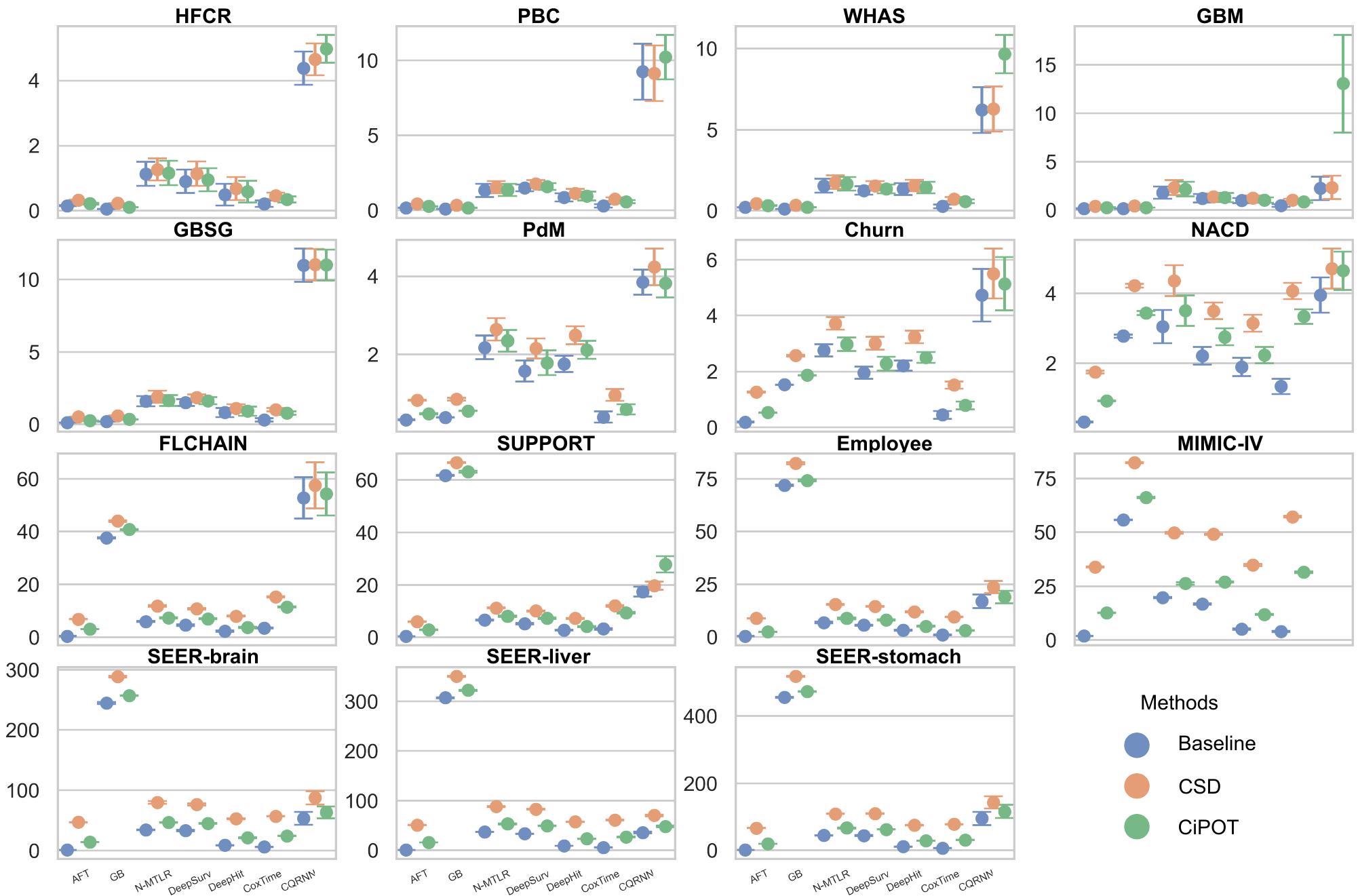}
    \caption{Training time comparisons (mean with 95\% confidence interval). }
    \label{fig:train_time}
\end{figure}

\begin{figure}
    \centering
    \includegraphics[width=\textwidth]{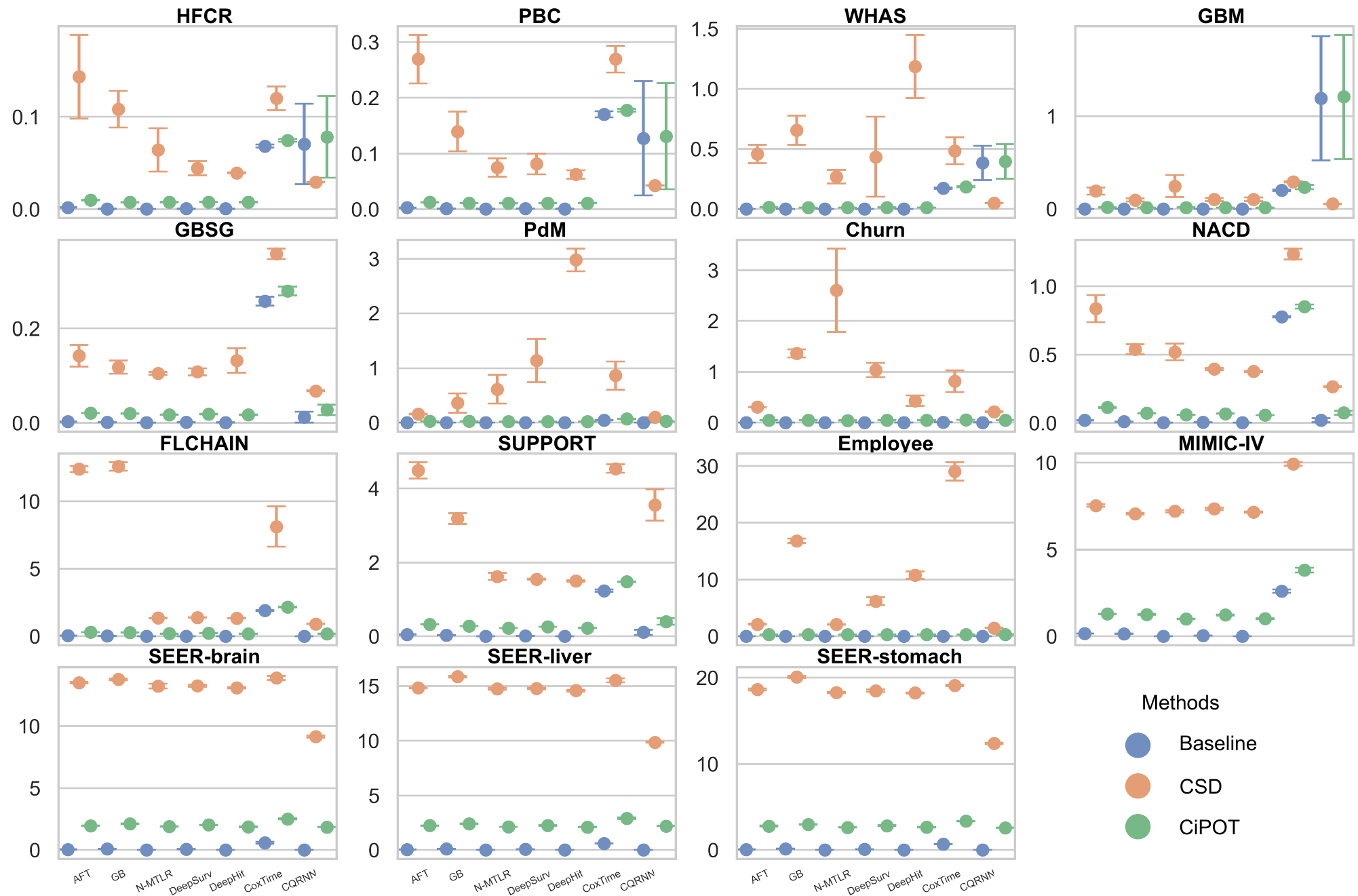}
    \caption{Inference time comparisons (mean with 95\% confidence interval).}
    \label{fig:infer_time}
\end{figure}

\paragraph{Time complexity}
The primary sources of time complexity in \CSDiPOT are two-fold:
\begin{itemize}
    \item ISD interpolation and extrapolation (line 4 in Algorithm~\ref{alg:csd_ipot});
    \item An optional monotonic step for the \emph{CQRNN} model (see discussion below).
\end{itemize}
Note that other time complexity, \eg running the $\Pct$ operation or adjusting the ISD curves using lines 14-15 in Algorithm~\ref{alg:csd_ipot}, incur negligible time cost.

Here we analyze two kinds of time complexity: training complexity and inference complexity.
Figure~\ref{fig:train_time} and Figure~\ref{fig:infer_time} empirically compare the training time and inference time of the \CSDiPOT method with those of non-post-processed baselines and \texttt{CSD} across 10 random splits. Both \texttt{CSD} and \CSDiPOT use the following hyperparameters to enable a fair comparison:

\begin{itemize}
    \item Interpolation: PCHIP
    \item Extrapolation: Linear
    \item Monotonic method:  Bootstrapping
    \item Number of percentile: 9
    \item Conformal set: Training set $+$ Validation set
    \item Repetition parameter: 1000
\end{itemize}

Each point in Figure~\ref{fig:train_time} represents an average training time of the method, where for the non-post-processed baselines it is purely the training time of the survival model, while for \texttt{CSD} and \CSDiPOT, it is the training time of the baselines plus the running time for the conformal training/learning.
In Figure~\ref{fig:train_time}, we see that \CSDiPOT can significantly reduce the additional time imposed by the \texttt{CSD} step on all survival analysis models (\emph{AFT}, \emph{GB}, \emph{N-MTLR}, \emph{DeepSurv}, \emph{DeepHit}, and \emph{CoxTime}).
The only exceptional is the quantile-based method \emph{CQRNN}, where for the 4 small datasets (\texttt{HFCR}, \texttt{PBC}, \texttt{WHAS}, and \texttt{GBM}) \CSDiPOT actually increase the training time.
This is because, for those datasets, the quantile curves predicted by \emph{CQRNN} are not monotonic.
To address this, we attempted three methods before we directly apply \CSDiPOT, including ceiling, flooring, and bootstrap rearranging~\citet{chernozhukov2010quantile} -- with the bootstrap method proving most effective, albeit at a significant computational cost.

Similarly, each point in Figure~\ref{fig:infer_time} represents an average inference time of the method, where for the non-post-processed baselines it is purely the inference time of the ISD predictions from the survival model, while for \texttt{CSD} and \CSDiPOT, it is the inference time of the ISDs plus the post-processing time.
We observe the inference time follows the same trend as the training time.
The extra cost for the 4 small datasets (\texttt{HFCR}, \texttt{PBC}, \texttt{WHAS}, and \texttt{GBM}) is still due to the non-monotonic issues.
In those cases, the predicted curves by \emph{CQRNN} become monotonic after applying \texttt{CSD}. However, surprisingly, after applying \CSDiPOT, the curve remains non-monotonic.
That is why the inference time for baselines and \CSDiPOT is higher than \texttt{CSD} on those four datasets.

\subsection{Ablation Studies}
\label{appendix:ablations}

\begin{figure}[ht]
    \centering
    \includegraphics[width=\textwidth]{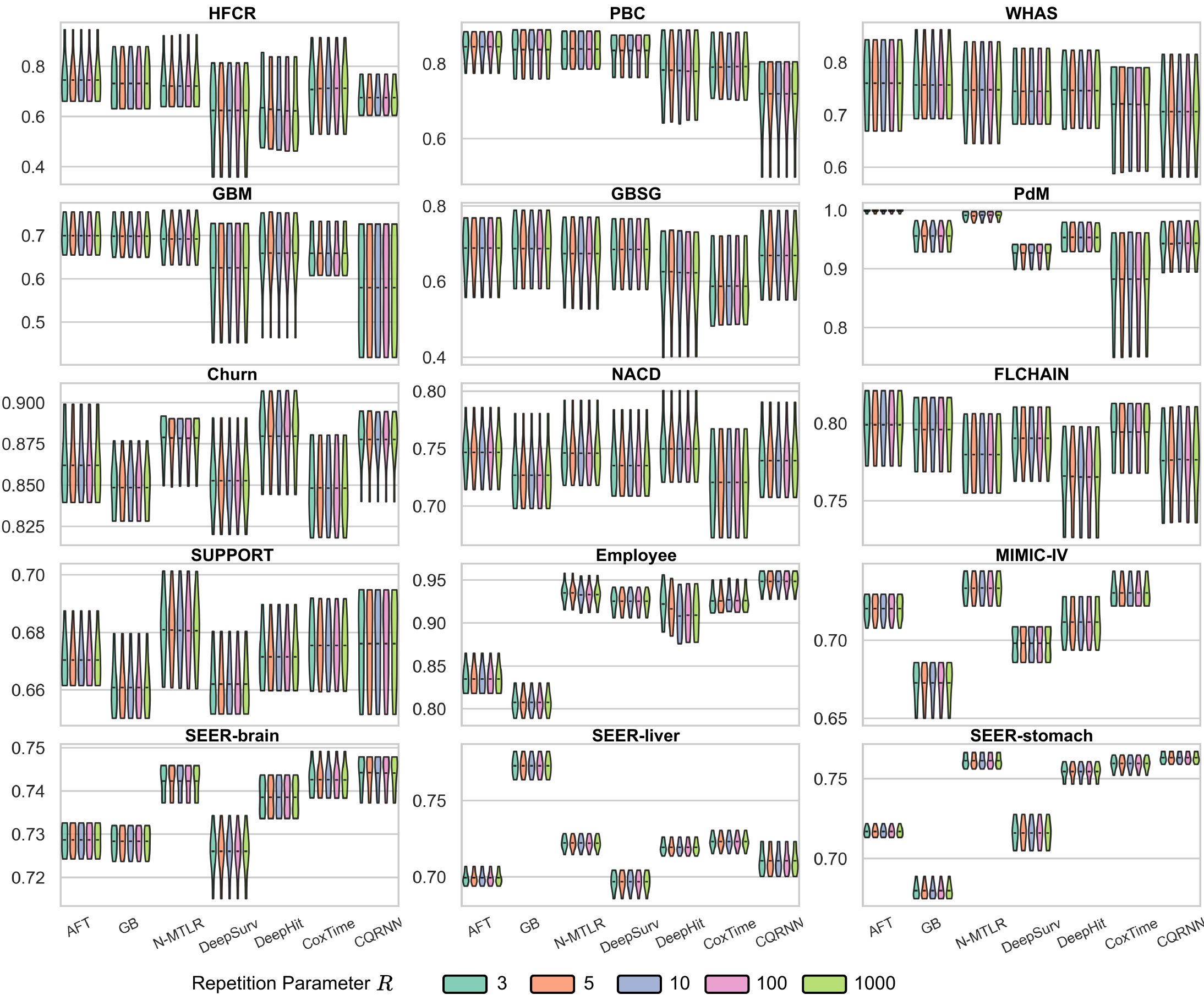}
    \caption{Violin plots of C-index performance for \textbf{Ablation Study \#1: impact of repetition parameter}. 
    A higher value indicates superior performance.}
    \label{fig:cindex_n_sampling}
\end{figure}

\begin{figure}[ht]
    \centering
    \includegraphics[width=\textwidth]{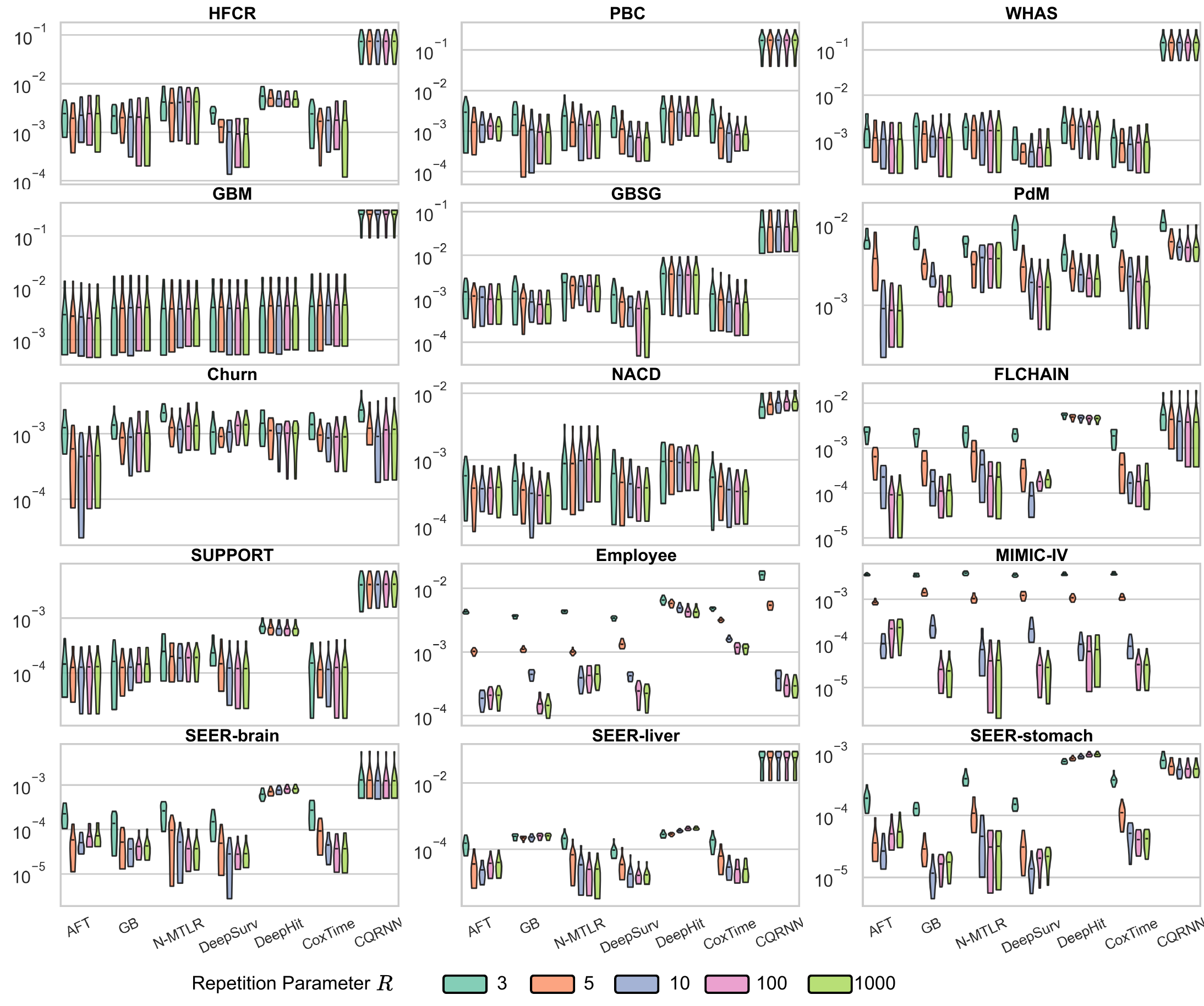}
    \caption{
    Violin plots of $\text{Cal}_{\text{margin}}$ performance for \textbf{Ablation Study \#1: impact of repetition parameter}. 
    A lower value indicates superior performance.}
    \label{fig:dcal_n_sampling}
\end{figure}

\begin{figure}[ht]
    \centering
    \includegraphics[width=\textwidth]{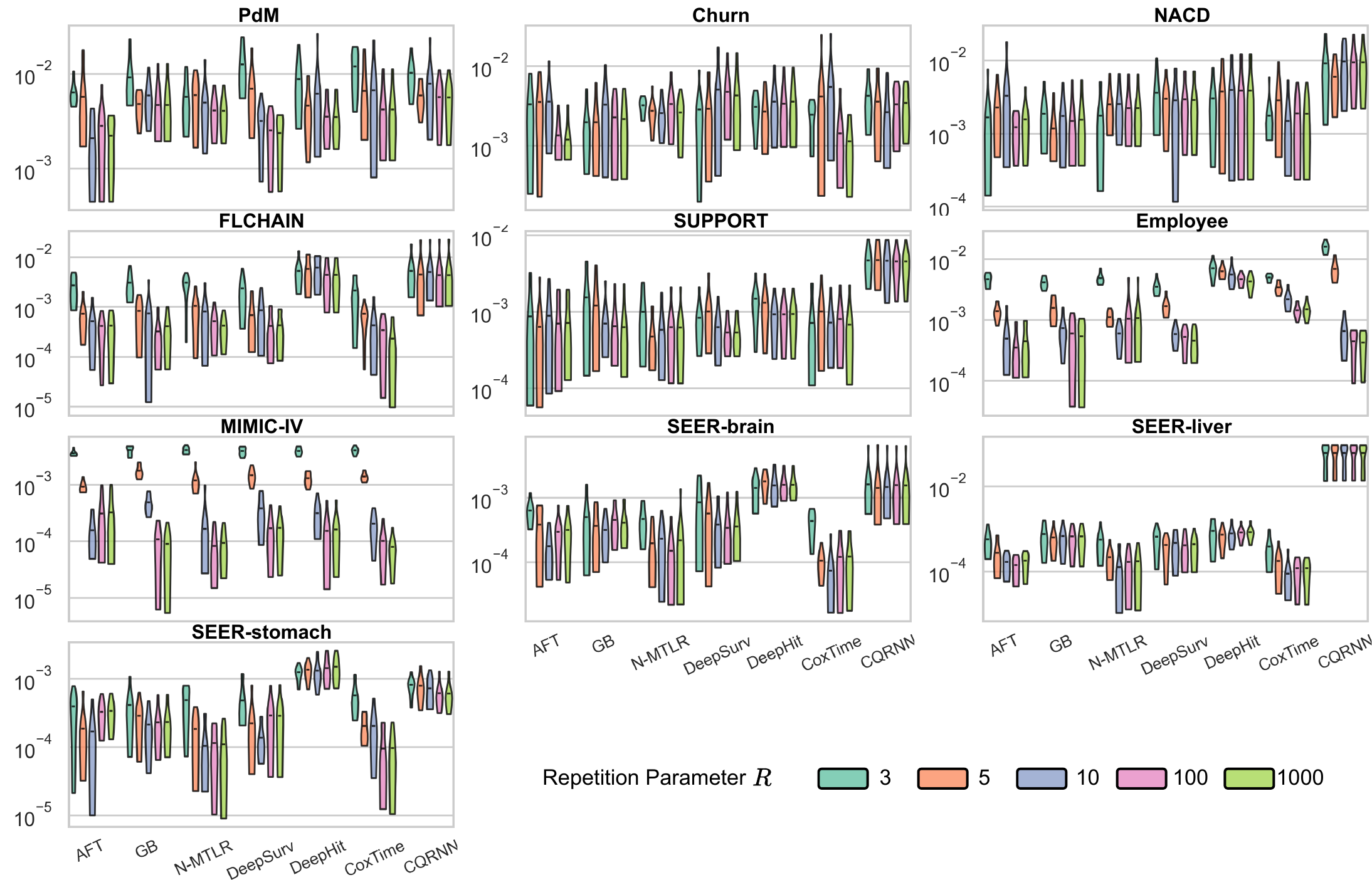}
    \caption{
    Violin plots of $\text{Cal}_{\text{ws}}$ performance for \textbf{Ablation Study \#1: impact of repetition parameter}. 
    A lower value indicates superior performance.}
    \label{fig:dcal_ws_n_sampling}
\end{figure}

\paragraph{Ablation Study \#1: impact of repetition parameter $R$}


As we proved in Theorem~\ref{theorem:margin_cal}, for a censored subject $j$, if we sample its iPOT value using $\mathcal{U}_{0, \hat{S}_\Model (c_i \mid \Bfx{i})}$, we can asymptotically achieve the exact marginal calibration.
However, empirically, due to limited sample sizes, we find that only making one sampling for each censored subject will not achieve a good calibration performance.
Instead, we propose the method of repetition sampling, \ie sampling $R$ times from $\mathcal{U}_{0, \hat{S}_\Model (c_i \mid \Bfx{i})}$.

This ablation study tries to find how this repetition parameter $R$ affects the performance, in terms of both discrimination and calibration.
We gradually increase $R$, from 3 to 1000, and assume the performance should converge at a certain level.
This ablation study uses the following hyperparameters to enable a fair comparison:
\begin{itemize}
    \item Interpolation: PCHIP
    \item Extrapolation: Linear
    \item Monotonic method:  Bootstrapping
    \item Number of percentile: 9
    \item Conformal set: Training set $+$ Validation set
    \item Repetition parameter: {3, 5, 10, 100, 1000}
\end{itemize}

Figure~\ref{fig:cindex_n_sampling}, Figure~\ref{fig:dcal_n_sampling}, and Figure~\ref{fig:dcal_ws_n_sampling} present the C-index, marginal calibration, and conditional calibration performances, respectively.

\textbf{TL;DR} The repetition parameter value has barely any impact on the C-index, and increasing $R$ can benefit the marginal and conditional calibration, with convergence observed around $R=100$.

In Figure~\ref{fig:cindex_n_sampling}, we see that the C-index for \CSDiPOT using 5 different repetition numbers has no visible differences, for almost all baselines and all datasets.
The only exception is there are slight differences for \textit{DeepHit} baseline on \texttt{HFCR} and \texttt{Employee} datasets, where higher $R$ will slightly decrease the C-index performance insignificantly.


In Figure~\ref{fig:dcal_n_sampling}, we can clearly see the repetition parameter has a great impact on the marginal calibration.
For most datasets, a higher repetition will significantly improve (decrease) the marginal calibration score. 
This trend is more clear for high censoring rate datasets (\texttt{FLCHAIN}, \texttt{Employee}, \texttt{MIMIC-IV},  etc),
while for low censoring rate datasets (\texttt{GBM}, \texttt{NACD}, \texttt{SUPPORT}, which have censoring rates less than 40\%), the trends still exist but the differences are not very significant.
The marginal calibration performance usually converges when $R \geq 100$ (the pink and green violins are almost the same for all datasets and all baselines).
As a higher repetition parameter will increase the memory usage and computational complexity (Appendix~\ref{appendix:computation}), we suggest using $R=100$.

In Figure~\ref{fig:dcal_ws_n_sampling}, the trends in conditional calibration are similar to those in marginal calibration. 
Most of the differences are more significant for high censoring rate datasets and less significant for low censoring rate datasets.
And the conditional calibration performance also converges at $R \geq 100$ (the pink and green violins are almost the same for all datasets and all baselines).

\paragraph{Ablation Study \#2: impact of predefined percentiles $\mathcal{P}$}

\begin{figure}
    \centering
    \includegraphics[width=\textwidth]{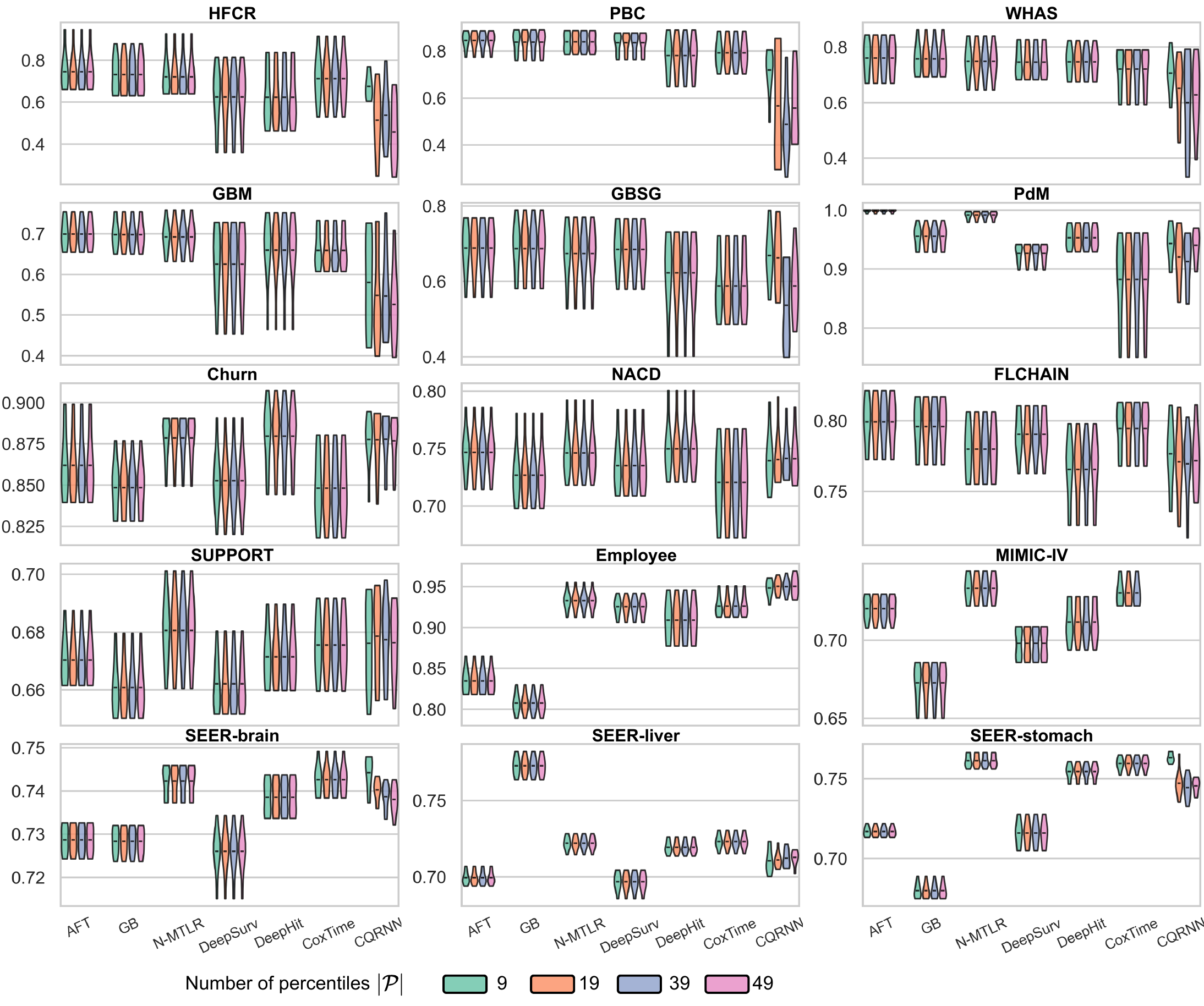}
    \caption{Violin plots of C-index performance for \textbf{Ablation Study \#2: impact of predefined percentiles}. 
    A higher value indicates superior performance. }
    \label{fig:cindex_n_percentile}
\end{figure}

\begin{figure}
    \centering
    \includegraphics[width=\textwidth]{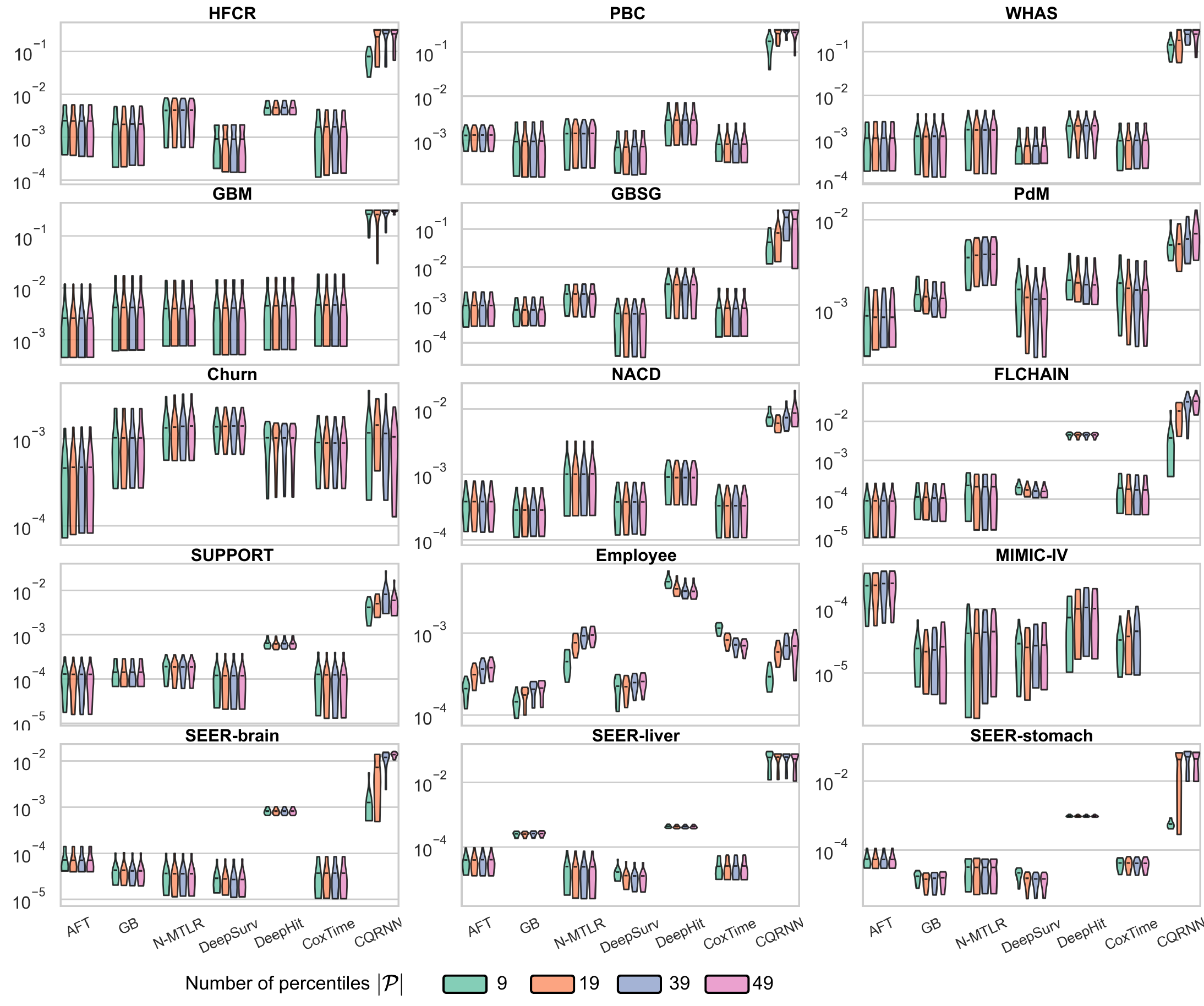}
    \caption{Violin plots of $\text{Cal}_{\text{margin}}$ performance for \textbf{Ablation Study \#2: impact of predefined percentiles}. 
    A lower value indicates superior performance. }
    \label{fig:dcal_n_percentile}
\end{figure}

\begin{figure}
    \centering
    \includegraphics[width=\textwidth]{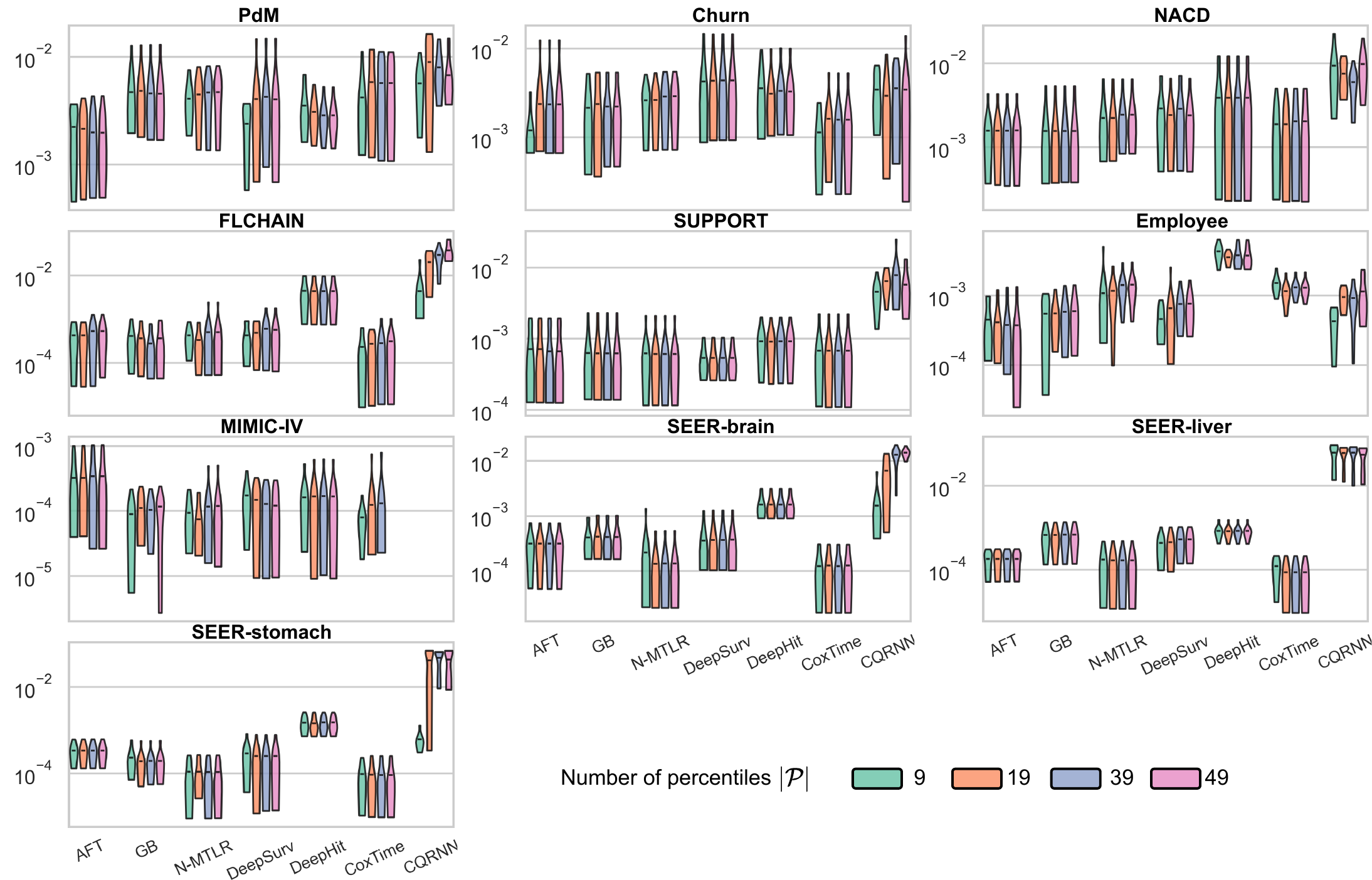}
    \caption{Violin plots of $\text{Cal}_{\text{ws}}$ performance for \textbf{Ablation Study \#2: impact of predefined percentiles}. 
    A lower value indicates superior performance.}
    \label{fig:dcal_ws_n_percentile}
\end{figure}

Different choices of $\mathcal{P}$ may lead to slightly different survival distributions, all of which allow us to obtain provable distribution calibration, as discussed next. 
Theoretically, discretizing a continuous curve into a series of discrete points may result in some loss of information. 
However, this can be mitigated by using a sufficiently fine grid for percentile discretization.
Therefore, we anticipate that if we select $\mathcal{P}$ that contains finer grid percentiles, the performance (both calibration and discrimination) will be better.

Previous studies have commonly employed 10 equal-width probability intervals for calculating distribution calibration~\cite{haider2020effective, goldstein2020x, qi2024conformalized}, making an intuitive starting choice for 9 percentile levels at $\mathcal{P}= \{10\%, 20\%, \ldots, 90\%\}$\footnote{Note that $0\%$ and $100\%$ are excluded because $\Tilde{S}_\Model^{-1}(1 \mid \Bfx{i}) = 0$ and $\Tilde{S}_\Model^{-1}(0 \mid \Bfx{i}) = t_{\text{max}} <\infty$ are fixed endpoints.}.

We compare this standard setting with 19 percentile levels ($\mathcal{P}= \{5\%, 10\%, \ldots, 95\%\}$), 39 percentile levels ($\mathcal{P}= \{2.5\%, 5\%, \ldots, 97.5\%\}$), and 49 percentile levels ($\mathcal{P}= \{2\%, 4\%, \ldots, 98\%\}$).
All the calibration evaluations (for both marginal and conditional) are performed on 10 equal-width intervals to maintain comparability, following recommendations by~\cite{haider2020effective}.
This ablation study uses the following hyperparameters to enable a fair comparison:

\begin{itemize}
    \item Interpolation: PCHIP
    \item Extrapolation: Linear
    \item Monotonic method:  Bootstrapping
    \item Number percentile: {9, 19, 39, 49}
    \item Conformal set: Training set $+$ Validation set
    \item Repetition parameter: 1000
\end{itemize}

Figure~\ref{fig:cindex_n_percentile}, Figure~\ref{fig:dcal_n_percentile}, and Figure~\ref{fig:dcal_ws_n_percentile} present the C-index, marginal calibration, and conditional calibration performances, respectively.

\textbf{TL;DR} The number of percentiles has no impact on the C-index, and has a slight impact on marginal and conditional calibration.

In Figure~\ref{fig:cindex_n_percentile}, we see that the C-index for the 4 percentile numbers has no visible differences, for all baselines except \emph{CQRNN}.
It is worth mentioning that the difference for \emph{CQRNN} is due to its quantile regression nature, requiring more trainable parameters as percentiles increase.

In Figure~\ref{fig:dcal_n_percentile}, we can see the number of percentiles has some impact on the marginal calibration.
For example, the results for \texttt{Employee} and \texttt{MIMIC-IV} show some perturbation as we increase the number of percentiles.
However, trends are inconsistent across models, \eg fewer percentile levels are preferable by \emph{AFT}, \emph{GB}, \emph{N-MTLR}, \emph{DeepSurv} for \texttt{Employee}, while more percentile levels are preferable by \emph{DeepHit} and \emph{CoxTime}.
Also, most differences in Figure~\ref{fig:dcal_n_percentile} are insignificant. 

In Figure~\ref{fig:dcal_ws_n_percentile}, the trends in conditional calibration are similar to those in marginal calibration. 
Most of the differences are insignificant and the trends are inconsistent.

These results suggest the choice of predefined percentiles $\mathcal{P}$ has minimal impact on the performance.
In practical applications, the reader can choose the best-performing $\mathcal{P}$ at their preference.
However, it is worth noticing that a higher percentile will result in more computational cost. 
Therefore, we suggest that it is generally enough to choose $\mathcal{P}= \{10\%, 20\%, \ldots, 90\%\}$.

\end{document}